\documentclass{article}

\usepackage{microtype}
\usepackage{graphicx}
\usepackage{subfigure}
\usepackage{booktabs,makecell} 
\usepackage{algorithm,algorithmic}
\usepackage{amssymb, multirow, paralist, color,amsmath,amsthm}
\usepackage{textcomp}
\usepackage{siunitx}
\usepackage{mathtools}

\usepackage{hyperref}

\newcommand{\Norm}[1]{\left\|#1\right\|}

\usepackage[accepted]{icml2023}
\def \S {\mathbf{S}}

\def \X {\mathcal{X}}
\def \O {\mathcal{O}}

\def \R {\mathbb{R}}

\def \w {\mathbf{w}}
\def \v {\mathbf{v}}

\def \x {\mathbf{x}}

\def \x {\mathbf{x}}

\def \1 {\mathbf{1}}

\def \I {\mathbb{I}}

\def \B {\mathcalB}

\newcommand{\inner}[2]{\left\langle #1,\, #2 \right\rangle}

\def \E {\mathbb{E}}
\def \x {\mathbf{x}}

\def \D {\mathcal{D}}

\def \w {\mathbf{w}}

\def \R {\mathbb{R}}
\def \S {\mathcal{S}}

\def \v {\mathbf{v}}

\def \I {\mathbb{I}}

\def \B {\mathcal{B}}

\def \X {\mathcal{X}}

\def\EX{{\mathbb{E}}}

\usepackage{soul}
\usepackage{enumitem}

\def\bw{\mathbf{w}}

\newtheorem{thm}{Theorem}

\newtheorem{lemma}{Lemma}
\newtheorem{cor}{Corollary}

\newtheorem{definition}[cor]{Definition}
\newtheorem{ass}{Assumption}
\usepackage[colorinlistoftodos,bordercolor=green,backgroundcolor=white,linecolor=green,textsize=scriptsize]{todonotes}

\icmltitlerunning{Provable Multi-instance Deep AUC Maximization with Stochastic Pooling}

\begin{document}

\twocolumn[
\icmltitle{Provable Multi-instance Deep AUC Maximization with Stochastic Pooling}



\icmlsetsymbol{equal}{*}

\begin{icmlauthorlist}
\icmlauthor{Dixian Zhu}{iowa}
\icmlauthor{Bokun Wang}{tamu}
\icmlauthor{Zhi Chen}{iowa-ee}
\icmlauthor{Yaxing Wang}{peking}
\icmlauthor{Milan Sonka}{iowa-ee}
\icmlauthor{Xiaodong Wu}{iowa-ee}
\icmlauthor{Tianbao Yang}{tamu}
\end{icmlauthorlist}

\icmlaffiliation{iowa}{Department of Computer Science, University of Iowa, IA, USA}
\icmlaffiliation{tamu}{Department of Computer Science and Engineering, Texas A\&M University, College Station, TX, USA}
\icmlaffiliation{iowa-ee}{Department of Electrical and Computer Engineering, University of Iowa, IA, USA}
\icmlaffiliation{peking}{Beijing Institute of Ophthalmology, Beijing Tongren Hospital, Capital Medical University, Beijing, China}

\icmlcorrespondingauthor{Dixian Zhu, Tianbao Yang}{dixian-zhu@uiowa.edu, tianbao-yang@tamu.edu}

\icmlkeywords{Machine Learning, ICML}

\vskip 0.3in
]



\printAffiliationsAndNotice{} 

\begin{abstract}
 This paper considers a novel application of deep AUC maximization (DAM) for multi-instance learning (MIL), in which a single class label is assigned to a bag of instances (e.g., multiple 2D slices of a CT scan for a patient). We address a neglected yet non-negligible computational challenge of MIL in the context of DAM, i.e., bag size is too large to be loaded into {GPU} memory for backpropagation, which is required by the standard pooling methods of MIL.  To tackle this challenge, we propose variance-reduced stochastic pooling methods in the spirit of stochastic optimization by formulating the loss function over the pooled prediction as a multi-level compositional function. By synthesizing  techniques from stochastic compositional optimization and non-convex min-max optimization, we propose a unified and provable muli-instance DAM (MIDAM) algorithm with stochastic smoothed-max pooling or stochastic attention-based pooling, which only samples a few instances for each bag to compute a stochastic gradient estimator and to update the model parameter. We establish a similar convergence rate of the proposed MIDAM algorithm as the state-of-the-art DAM algorithms. Our extensive experiments on conventional MIL datasets and medical datasets demonstrate the superiority  of our MIDAM algorithm. {The method is open-sourced at~\url{https://libauc.org/}}.
\end{abstract}
\setlength{\textfloatsep}{2pt}
\setlength\abovedisplayskip{2pt}
\setlength\belowdisplayskip{2pt}
\section{Introduction}
Deep AUC maximization (DAM) has recently achieved great success for many AI applications due to its capability of handling imbalanced data~\cite{10.1145/3554729}. For example, it earned first place at Stanford CheXpert competition~\cite{irvin2019chexpert}, and state-of-the-art performance on other datasets~\cite{yuan2021large,wang2021advanced}. However,  a novel application of DAM for multi-instance learning (MIL) has not been studied in the literature.  

MIL refers to a setting where multiple instances are observed for an object of interest and only one label is given to describe that object. Many real-life applications can be formulated as MIL.  For example, the medical imaging data for diagnosing a patient usually consists of a series of 2D high-resolution images (e.g., CT scan), and only  a single label (containing a tumor or not) is assigned to the patient~\cite{7812612}.   MIL has a long history in machine learning and various methods have been proposed for traditional learning with tabular data~\cite{Babenko2008MultipleIL,DBLP:journals/corr/CarbonneauCGG16} and deep learning (DL) with unstructured data~\cite{Oquab15,8099499,pmlr-v80-ilse18a}.    The fundamental theorem of symmetric functions~\cite{NIPS2017_f22e4747,DBLP:journals/corr/QiSMG16}, inspires a general three-step approach for classifying a bag of instances: (i) a transformation of individual instances, (ii) a pooling of transformed instances using a symmetric (permutation-invariant) function, (iii) a transformation of pooled representation. 
A key in the implementation of the three steps is the symmetric function that takes the transformations of all instances as input and produces an output, which is also known as the {\bf pooling operation}. In the literature, various pooling strategies have been explored, e.g., max pooling, average pooling, and smoothed-max (i.e., log-exp-sum) pooling of predictions~\cite{lirias1654230},  attention-based pooling of feature representations~\cite{pmlr-v80-ilse18a}. 

However, to the best of our knowledge, none of the existing works have tackled the computational challenge of MIL in the context of DL when a bag is large due to the existence of multiple instances in the bag.  The limitation of computing resources (e.g., the memory size of GPU) might prevent loading all instances of a bag  at once, creating a severe computational bottleneck for training. For example, an MRI scan of the brain may produce up to hundreds of 2D slices of high resolution~\citep{calabrese2022university}. It is hard to process all slices of a patient at each iteration for DL. Even if the size of an image can be reduced to fit into the memory, the convergence performance will be compromised due to a small batch size (i.e., few patients can be processed due to many slices per patient). A naive approach to deal with this challenge is to use mini-batch stochastic pooling, i.e., only sampling a few instances from a bag for computing the pooled prediction and conducting the update. However, this naive approach does not ensure optimization of the objective that is defined using the pooling of all instances for each bag due to the error of mini-batch stochastic pooling.  


We tackle this challenge of multi-instance DAM in a spirit of stochastic optimization by (i) formulating  the pooled prediction as a compositional function whose inner functions  are expected functions over instances of that bag, and (ii) proposing efficient and provable stochastic algorithms for solving the {\it non-convex min-max optimization with a multi-level compositional objective function}. A key feature of the proposed algorithms is replacing the deterministic pooling over all instances of a bag by a {\it variance-reduced stochastic pooling (VRSP)}, whose computation only requires sampling a few instances from the bag. To ensure the optimization of the original objective, the VRSP is constructed following the principle of stochastic compositional optimization such that the variance of stochastic pooling estimators is reduced in the long term. In particular, the inner functions of the pooled prediction  are tracked and estimated by  moving average estimators separately for each bag. Based on VRSP, stochastic gradient estimators are computed for updating the model parameter, which can be efficiently implemented by backpropagation. 

Our contributions are summarized in the following: 
\begin{itemize}[noitemsep,topsep=0pt,parsep=0pt,partopsep=0pt]
\item We propose variance-reduced stochastic pooling estimators for both smoothed-max pooling and attention-based pooling. Building on these stochastic pooling estimators, we develop unified efficient algorithms of multi-instance DAM (MIDAM) based on a min-max objective for the two pooling operations. 
\item We develop novel convergence analysis of the proposed MIDAM algorithms by (i) proving the averaged error of variance-reduced stochastic pooling estimators over all iterations will converge to zero, and (ii) establishing a convergence rate showing our algorithms can successfully find an $\epsilon$-stationary solution of the min-max objective of DAM.  
\item We conduct extensive experiments of proposed MIDAM algorithms on conventional MIL benchmark datasets and emerging medical imaging datasets with high-resolution medical images,  demonstrating the better performance of our algorithms.
\end{itemize}

\section{Related Works}\vspace*{-0.05in}
In this section, we introduce previous works on AUC maximization, multi-instance learning, and medical image classification, and then discuss how they are related to our work.

{\bf Deep AUC maximization (DAM).}  Maximizing the area under the receiver operating characteristic curve (AUC), as an effective method for dealing with imbalanced datasets, has been vigorously studied for the last two decades~\cite{10.1145/3554729}.  Earlier studies focus on learning traditional models, e.g., SVM, decision tree~\citep{cortes2003auc,joachims2005support,ferri2002learning}. 
Inspired by the Wilcoxon-Man-Whitney statistic, a variety of pairwise losses and optimization algorithms have been studied for AUC optimization~\citep{gao2013one,zhao2011online,kotlowski2011bipartite,gao2015consistency,calders2007efficient,charoenphakdee2019symmetric}. 
Inspired by the min-max objective corresponding to the pairwise square loss function~\cite{ying2016stochastic}, stochastic algorithms have been developed for DAM~\citep{liu2019stochastic,yuan2021large}. In this work, we propose efficient and scalable methods for DAM under the multi-instance learning (MIL) scenario with real big-data applications.

{\bf Multi-instance learning.} Multi-instance learning (MIL) has been extensively studied and adopted for real applications since decades ago~\cite{ramon2000multi,andrews2002support,Oquab15,kraus2016classifying}. Usually, a simple MIL pooling strategy, that is, max-pooling over a data bag has been widely utilized. This idea has been incorporated with support vector machine (SVM) and neural networks~\cite{andrews2002support, Oquab15,wang2018multi}. Other pooling strategies have also been proposed, e.g., mean, smoothed-max (aka. log-sum-exponential), generalized mean, noisy-or, noisy-and~\cite{wang2018multi,ramon2000multi,keeler1990integrated,kraus2016classifying}. 
Recently, attention-based pooling was proposed for deep MIL~\cite{pmlr-v80-ilse18a}. 
It is worth noting that almost all the pooling strategies (except max-pooling) require loading all the data from a bag to do the computation, specifically backpropagation. However, there is still no existing method that considers mitigating the computational issue when the data size is too large even for a single data bag. 

{\bf Medical image classification.} In MRI/CT scans,  multiple slices of images are acquired at different locations  of the patient?s body, which not only improves the diagnostic capabilities but also lowers doses of radiation. Hence, a patient can be represented by a series of 2D slices. A traditional approach is to concatenate these 2D slices into a 3D image and then learn a 3D convolutional neural network (CNN)~\cite{s20185097}. However, this approach suffers from several drawbacks. 
First, it demands more computational and memory resources as processing high-resolution 3D images is more costly than processing 2D images. As a consequence, the mini-batch size  for back-propagation in training is compromised or the resolution is reduced, which can harm the learning capability. Third, it is more difficult to interpret the prediction of a DL model based on 3D images as radiologists still use 2D slices to make diagnostic decision~\citep{10.1117/1.JMI.7.5.051203}. To avoid these issues, we will investigate MIL and make it practical for medical image classification.

\vspace*{-0.05in}
\section{Preliminaries}\vspace*{-0.05in}
{\bf Notations.} Let $\X_i=\{\x_i^1, \ldots, \x_i^{n_i}\}$ denote a bag of data instances (e.g., 2D image slices of an MRI/CT scan). Let $\D=\{(\X_i, y_i), i=1,\ldots, n\}$ denote the set of labeled data, where  $y_i\in\{0, 1\}$ denotes the label associated with the bag $i$. Let $\D_+\subset\D$ only contain $D_+$ positive bags with $y_i=1$ and $\D_-\subset\D$ only contain $D_-$ negative bags with $y_i=0$. Without loss  of generality,  let $\w\in\R^{d}$ denote all weights to be learned, which includes the weights of the feature encoder network, the weights of the instance-level classifier, and the parameters in the attention-based pooling.   Let $e(\w_e; \x)\in\R^{d_o}$ be the instance-level representation encoded by a neural network $\w_e$, $\phi(\w; \x)\in [0,1]$ be the instance-level prediction score (after some activation function), and $h(\w; \X_i)\in [0,1]$ be the pooled prediction score of the bag $i$ over all its instances. Besides, $\sigma(\cdot)$ denotes the sigmoid activation.

{\bf Multi-instance Learning (MIL).} We work under the standard MIL assumption that (i) an instance can be associated with a label and (ii) a bag is labeled positive if at least one of its instances has a positive label, and negative if all of its instances have negative labels~\cite{DIETTERICH199731}. The assumption  implies that a MIL model must be permutation-invariant for the prediction function $h(\X)$. To achieve permutation invariant property, fundamental theorems of symmetric functions have been developed~\cite{NIPS2017_f22e4747,DBLP:journals/corr/QiSMG16}. In particular, \citet{NIPS2017_f22e4747} show that a scoring function for a set of instances $\X$, $h(\X)\in\R$, is a symmetric function  if and only if it can be decomposed as $h(\X) = g(\sum_{\x\in\X}\psi(\x))$, where $g$ and $\psi$ are suitable transformations. \citet{DBLP:journals/corr/QiSMG16} prove that for any $\epsilon>0$, a Hausdorff continuous symmetric function $h(\X)\in\R$ can be arbitrarily approximated by a function in the form $g(\max_{\x\in\X}\psi(\x))$, where max is the element-wise vector maximum operator and $\psi$ and $g$ are continuous functions. These theories provide support for several widely used pooling operators used for MIL.  

{\bf Max and smoothed-max pooling of predictions.}  The simplest approach is to take the maximum of predictions of all instances in the bag, i.e., $h(\w; \X) = \max_{\x\in\X}\phi(\w; \x)$. However, the max operation is non-smooth, which usually causes difficulty in optimization. In practice, a smoothed-max (aka. log-sum-exp) pooling operator is used instead: 
\begin{align}\label{eqn:softpool}
h(\w; \X) = \tau \log\left(\frac{1}{|\X|}\sum_{\x\in\X}\exp(\phi(\w; \x)/\tau)\right),
\end{align} where $\tau>0$ is a hyperparameter and $\phi(\w;\x)$ is the prediction score for instance $\x$.

{\bf Mean pooling of predictions.}  The mean pooling operator just takes the average of predictions of individual instances, i.e.,  $h(\w; \X) = \frac{1}{|\X|}\sum_{\x\in\X}\phi(\w; \x)$. Indeed, smoothed-max pooling interpolates between the max pooling (with $\tau=0$) and the mean pooling (with $\tau=\infty$). 

{\bf Attention-based Pooling.} Attention-based pooling was recently introduced for deep MIL~\cite{pmlr-v80-ilse18a}, which aggregates the feature representations using attention, i.e., 
\begin{align}\label{eqn:attpoolf}
E(\w; \X) = \sum_{\x\in\X}\frac{\exp(g(\w; \x))}{\sum_{\x'\in\X}\exp(g(\w; \x'))}e(\w_e; \x)
\end{align}
where 
$g(\w;\x)$ is a parametric function, e.g., $g(\w; \x)=\w_a^{\top}\text{tanh}(V e(\w_e; \x))$, where $V\in\R^{m\times d_o}$ and $\w_a\in\R^m$. Based on the aggregated feature representation, the bag level prediction can be computed by 
\begin{align}\label{eqn:attpool}
h(\w; \X) & = \sigma(\w_c^{\top}E(\w; \X))\\
&= \sigma\left(\sum_{\x\in\X}\frac{\exp(g(\w; \x))\delta(\w;\x)}{\sum_{\x'\in\X}\exp(g(\w; \x'))}\right), \nonumber 
\end{align}
where $\delta(\w;\x) = \w_c^{\top}e(\w_e; \x)$. 
In this paper, we will focus on smoothed-max pooling and attention-based pooling due to their generality and the challenge of handling them. 


\noindent{\bf Deep AUC Maximization (DAM).} AUC score can be interpreted as the probability of a positive sample   ranking higher than a negative sample~\cite{Hanley}, i.e., $\text{AUC}(h) = \EX_{\X, \X'}\bigl[\I(h(\w; \X) - h(\w; \X')\ge 0) \big | y=1, y'=0\bigr]$. 
In practice, one often replaces the indicator function in the above definition of AUC by a {\em convex surrogate loss} $\ell: \R \to \R^+$ which satisfies $\mathbb I(h(\w; \X') - h(\w;  \X)>  0) \le \ell(h(\w;  \X') - h(\w; \X))$~\cite{10.1145/1102351.1102399,herschtal2004optimising,Xinhua,pmlr-v28-kar13,Wang,DBLP:conf/icml/ZhaoHJY11,ying2016stochastic,fastAUC18,natole2018stochastic}.  
Hence,  empirical AUC maximization  can be formulated as
$\min_{\bw\in \R^d}  \hat\EX\bigl[\ell(h(\w; \X') - h(\w; \X))|y=1, y'=0 \bigr]$, where $\hat\EX$ is the empirical average over data in the training set $\D$. 

Since optimizing the pairwise formulation is not suitable in some learning scenarios (e.g., online learning, federated learning)~\cite{ying2016stochastic,guo2020communication}, recent works of DAM have followed the line of min-max optimization~\cite{yuan2021large,liu2019stochastic}. Denote $c$ as a margin parameter and $\hat \EX_{i\in\D}$ as the empirical average over $i\in\D$. The objective is: 

	\begin{align}\label{eq:MI_DAM}
	&\min_{\w\in\R^d,(a,b)\in\R^2}\max_{\alpha\in\Omega}F\left(\w,a,b,\alpha\right):=\\\nonumber
 &\underbrace{\hat\EX_{i\in\D_+}\left[(h(\w; \X_i) - a)^2 \right]}_{F_1(\w, a)} + \underbrace{\hat\EX_{i\in\D_-}\left[(h(\w; \X_i) - b)^2 \right]}_{F_2(\w, b)} \\\nonumber
	& + \underbrace{\alpha (c+ \hat\EX_{i\in\D_-}h(\w; \X_i) - \hat\EX_{i\in\D_+}h(\w; \X_i)) - \frac{\alpha^2}{2}}_{F_3(\w, \alpha)},
	\end{align}
	where the first term is  the variance of prediction scores of positive data, the second term is  the variance of prediction scores of negative data.  The maximization over $\alpha\in\Omega$ yields a term that aims to push the mean score of positive data to be far away from the mean score of negative data. When $\Omega=\R$, the above min-max objective was shown to be equivalent to the pairwise square loss formulation~\cite{ying2016stochastic}, and when $\Omega=\R^+$, the above objective is the min-max margin objective proposed in~\cite{yuan2021large}. It is notable that we use conditional expectation given positive or negative labels instead of joint expectation over $(\X_i, y_i)$ as in~\cite{ying2016stochastic,yuan2021large,liu2019stochastic}. The reason is that we consider the batch learning setting and it was found in~\cite{zhubenchmark} sampling positive and negative data separately at each iteration is helpful for improving the performance.  

\vspace*{-0.05in}
\section{Multi-instance DAM} 
\vspace*{-0.05in}
Although efficient stochastic algorithms have been developed for DAM, a unique challenge exists in multi-instance DAM due to the computing of the pooled prediction $h(\w; \X)$. For example, in smoothed-max pooling  computing $h(\w; \X_i) = \tau \log(\frac{1}{|\X_i|}\sum_{\x\in\X_i}\exp(\phi(\w; \x)/\tau))$ requires processing all instances in the bag $\X_i$ to calculate their prediction scores $\phi(\w; \x), \forall\x\in\X_i$. Hence, one may need to load all instances of a bag into the GPU memory for forward propagation and backpropagation. This is prohibited if the size of each bag (i.e., the total sizes of all instances in each bag) is large. 

A naive approach to address this challenge is to replace the pooling over all instances with mini-batch pooling over randomly sampled instances of a bag. The mini-batch smoothed-max pooling can be computed as $h(\w; \B_i) = \tau \log(\frac{1}{|\B_i|}\sum_{\x\in\B_i}\exp(\phi(\w; \x)/\tau))$, where $\B_i\subset\X_i$ only contains a few sampled instances from the bag of all instances. However, this approach does not work since $h(\w; \B_i)$ is not an unbiased estimator, i.e., $\EX_{\B_i}h(\w; \B_i) \neq  h(\w; \X_i)$. As a result, the mini-batch pooled prediction would incur a large estimation error that depends on the number of sampled instances, i.e., $\EX_{\B_i}[(h(\w; \B_i) - h(\w; \X_i))^2]\leq O(\frac{1}{|\B_i|})$, which would lead to non-negligible optimization error~\cite{hu2020biased}. 
 
We propose a solid approach to deal with this challenge. Below, we first describe the high-level idea. Then, we present more details of variance-reduced stochastic  pooling estimators and the corresponding stochastic gradient estimators of the min-max objective. Finally, we present a unified algorithm for using both stochastic pooling methods.  

We regard the pooled prediction as two-level compositional functions $h(\w;\X_i)=f_2(f_1(\w; \X_i))$, where $f_2$ is a simple function that will be exhibited shortly for the two pooling operations,  and $f_1(\w; \X_i)=\EX_{\x\sim \X_i}[f_1(\w; \x)]$ involves average over the set of instances $\x\in\X_i$.  As a result, we cast the terms of objective into three-level compositional functions $f(f_2(f_1(\w)))$, where $f$ is a stochastic function. In particular, the first term in the min-max objective can be cast as  $\frac{1}{|\D_+|}\sum_{i\in\D_+}f(f_2(f_1(\w; \X_i)), a)$, where $f(\cdot, a)= (\cdot - a)^2$. The second term can be cast as $\frac{1}{|\D_-|}\sum_{i\in\D_-}f(f_2(f_1(\w; \X_i)), b)$.  As a result, the three terms of the objective can be written as
  	\begin{align*}
	F_1(\w, a)& =\frac{1}{|\D_+|}\sum_{y_i=1}f(f_2(f_1(\w; \X_i)), a)\nonumber\\
 F_2(\w, b)  & =\frac{1}{|\D_-|}\sum_{y_i=0}f(f_2(f_1(\w; \X_i)), b) \\\nonumber
	F_3(\w, \alpha)&= \alpha \bigg(c+ \frac{1}{|\D_+|}\sum_{y_i=1}f_2(f_1(\w; \X_i)) \\
 &\hspace*{0.5in}- \frac{1}{|\D_-|}\sum_{y_i=0}f_2(f_1(\w; \X_i)\bigg)- \frac{\alpha^2}{2} \nonumber.
	\end{align*}
To optimize the above objective, we need to compute a stochastic gradient estimator. 
Let us consider the gradient of the first term in terms of $\w$, i.e., 
\begin{align}\label{eqn:F1g}
&\nabla_\w F_1(\w, a) = \frac{1}{|\D_+|}\cdot\\
&\sum_{y_i=1} \nabla f_1(\w; \X_i)\nabla f_2(f_1(\w; \X_i))\nabla_1 f(f_2(f_1(\w; \X_i)), a),\nonumber
\end{align}
where $\nabla_1$ denotes the partial gradient in terms of the first argument. The key challenge lies in computing the innermost function $f_1(\w; \X_i)$ and its gradient $\nabla f_1(\w; \X_i)$. Due to that the functional value $f_1(\w; \X_i)$ is inside non-linear functions $f_2, f$, one needs to compute an estimator of $f_1(\w; \X_i)$ to ensure the convergence for solving the min-max problem. To this end, we will follow stochastic compositional optimization techniques to track and estimate $f_1(\w; \X_i)$ for each bag $\X_i$ separately such that their variance is reduced in the long term~\cite{wang2022finite}. 
\vspace*{-0.05in}
\subsection{Variance-reduced Stochastic Pooling (VRSP) Estimators and Stochastic Gradient Estimators}\vspace*{-0.05in}
 We write the smoothed-max pooling in~(\ref{eqn:softpool}) as  $h(\w; \X_i) = f_2(f_1(\w; \X_i))$, where $f_1, f_2$ are defined as:
 
\begin{align*}
&f_1(\w; \X_i)=\frac{1}{|\X_i|}\sum_{\x_i^j \in \X_i} \exp(\phi(\w; \x_i^j)/\tau),\\
&f_2(s_i) =\tau\log(s_i).
\end{align*}
We express the attention-based pooling in \eqref{eqn:attpool} as $h(\w;\X_i) = f_2(f_1(\w;\X_i))$, where $f_1$, $f_2$ are defined as: 
\begin{align*}
& f_1 (\w;\X_i) = \begin{bmatrix}
\frac{1}{|\X_i|}\sum_{\x_i^j \in \X_i} \exp(g(\w;\x_i^j)) \w_c^\top e(\w_e;\x_i^j) \\
\frac{1}{|\X_i|}\sum_{\x_i^j \in \X_i} \exp(g(\w;\x_i^j))
\end{bmatrix},\\
&f_2(s_i) = \sigma\left(\frac{s_{i1}}{s_{i2}}\right).
\end{align*}
One difference between the two pooling operators is that the inner function $f_1$ for attention-based pooling is a vector-valued function with two components.  For both pooling operators, the costs lie at the calculation of $f_1(\w; \X_i)$.  To estimate $f_1(\w; \X_i)$, we maintain a dynamic estimator denoted by $s_i$.  At the $t$-th iteration, we sample some positive bags $\S_+^t\subset\D_+$ and some negative bags $\S_-^t\subset\D_-$. For those sampled bags $i\in \S_+^t\cup \S_-^t$,  we update $s_i^t$ by:
\begin{align}\label{eqn:ui} 
s_i^t = (1-\gamma_0) s_i^{t-1} + \gamma_0 f_1(\w^t;\B_i^t), i\in\S_+^t\cup \S_-^t,
\end{align}
where $\B_i^t\subset\X_i$ refers to a mini-batch of instances sampled from $\X_i$ and $\gamma_0\in[0,1]$ is a hyperparameter.  For smoothed-max pooling, $s_i^t$ is computed by 
\begin{align}\label{eqn:sm}
s_i^t = (1-\gamma_0) s_i^{t-1} + \frac{\gamma_0 }{|\B^t_i|}\sum_{\x_i^j \in \B^t_i} \exp(\phi(\w^t; \x_i^j)/\tau), 
\end{align}
and for attention-based pooling, $s_i^t$ is computed by  
\begin{align}\label{eqn:ab}
s^t_i &=  (1-\gamma_0)s^{t-1}_i  + \gamma_0\cdot\\
 &\begin{bmatrix}
	\frac{1}{|\B^t_i|}\sum_{\x_i^j \in \B^t_i} \exp(g(\w^t;\x_i^j)) \delta(\w^t;\x_i^j) \\
    \frac{1}{|\B^t_i|}\sum_{\x_i^j \in \B^t_i} \exp(g(\w^t;\x_i^j))
\end{bmatrix}. \nonumber
	\end{align}
With $s_i^t$, we refer to $f_2(s_i^t)$ as the variance-reduced stochastic pooling (VRSP) estimator.  We will prove in the next section that the moving average estimators $s_i^t$ will ensure the averaged estimation error $\frac{1}{T}\sum_{t=0}^{T-1}\|s_i^t - f_1(\w^t; \X_i)\|^2$ for all bags across all iterations will converge to zero as $T\rightarrow\infty$ by properly updating the model parameter and setting the hyper-parameters. As a result, the following lemma will guarantee that the stochastic pooling estimator $f_2(s_i^t)$ will have a diminishing error in the long term. 
\vspace*{-0.05in}\begin{lemma}
If $f_2$ is continuously differentiable on a compact domain and there exists $c>0$ such that $f_2$ is $c$-Lipschitz continuous on that domain, then 
$(f_2(s_i^t) - f_2(f_1(\w^t;\X_i))^2 \leq c^2\|s_i^t - f_1(\w^t; \X_i)\|^2$ for $s_i^t, f_1(\w^t;\X_i) \in \mathrm{dom} f_2$.
\end{lemma}
\vspace*{-0.1in}
Building on the VRSP estimators, a stochastic gradient estimator of the objective can be easily computed. In particular, The gradient of $f(f_2(f_1(\w; \X_i)))$ in terms of $\w^t$ can be estimated by $\nabla f_1(\w^t;\B_i^t)\nabla f_2(s_i^{t-1})\nabla_1 f(f_2(s_i^{t-1}), a^t)$, and a stochastic gradient estimator of $f_2(f_1(\w^t; \X_i))$ can be computed by $\nabla f_1(\w^t;\B_i^t)\nabla f_2(s_i^{t-1})$.  As a result,  the stochastic gradient estimators in terms of $\w, a, b, \alpha$ of the three terms $F_1(\w, a)$, $F_2(\w, b)$ and $F_3(\w, \alpha)$ of the objective are computed as following, respectively: \begin{align*}
&G^t_{1,\w} = \hat\E_{i\in\S_+^t}\nabla f_1(\w^t; \B_{i}^t) \nabla f_2(s^{t-1}_i)\nabla_1 f( f_2(s^{t-1}_i), a^t), \\
&G^t_{2,\w} = \hat\E_{i\in\S_-^t}\nabla f_1(\w^t; \B_{i}^t) \nabla f_2(s^{t-1}_i)\nabla_1 f( f_2(s^{t-1}_i), b^t),\\
&G^t_{3,\w} = \alpha^t \cdot\left(\hat\E_{i\in\S_-^t}\nabla f_1(\w^t; \B_{i}^t) \nabla f_2(s^{t-1}_i)\right.\\
 & \hspace*{1.0in} \left.- \hat\E_{i\in\S_+^t}\nabla f_1(\w^t; \B_{i}^t) \nabla f_2(s^{t-1}_i)\right),\\
&G^t_{1,a}  = \hat\E_{i\in\S_+^t} \nabla_2 f( f_2(s^{t-1}_i)  , a^t),\\
 &G^t_{2, b} =\hat\E_{i\in\S_-^t} \nabla_2 f( f_2(s^{t-1}_i)  , b^t)\\
    &G^t_{3,\alpha}   = c+ \hat\E_{i\in\S_-^t}f_2(s^{t-1}_i) - \hat\E_{i\in\S_+^t}f_2(s^{t-1}_i),
\end{align*}
where $\nabla f_1(\w^t; \B_{i}^t)$ denotes the transposed Jacobian matrix of $f_1$ in terms of $\w$.   By plugging the explicit expression of (partial) gradients of $f_2$, $f$, we can compute these gradient estimators by backpropagation. 
With these stochastic gradient estimators, we will update the model parameter following the momentum update or the Adam update, which is presented in next subsection. 
\begin{algorithm}[t]
	\caption{The Unified MIDAM Algorithm}\label{alg:1}
	\begin{algorithmic}[1]
	\STATE Initialize $\w^0$,$s^0$,$\v^0$, $\eta, \eta'$, $\beta_1$, $\gamma_0$ 
		\FOR{$t=1,\ldots,T$}
		\STATE Sample a batch of positive bags $\S_+^t\subset\D_+$ and a batch of negative bags $\S_-^t\subset\D_-$
		\FOR{each $i \in \S^t=\S_+^t\cup \S_-^t$} 
        \STATE Sample a mini-batch of instances $\mathcal B^t_i\subset\X_i$ and update
	$s^t_i =  (1-\gamma_0)s^{t-1}_i  + \gamma_0 f_1(\w^t; \B_{i}^t)$
		\ENDFOR 
		\STATE Update stochastic gradient estimator of $(\w, a, b)$: 
		\begin{align*}
		   \v_1^t& =\beta_1\v_1^{t-1} + (1-\beta_1)(G^t_{1,\w} + G^t_{2,\w} + G^t_{3,\w})\\
		    \v_2^t &=\beta_1\v_2^{t-1} + (1-\beta_1)G^t_{1,a}\\
         \v_3^t &=\beta_1\v_3^{t-1} + (1-\beta_1)G^t_{2,b}
		\end{align*}
		\STATE Update $(\w^{t+1}, a^{t+1}, b^{t+1}) = (\w^t, a^t, b^t) - \eta \v^t$ (or the Adam-style update)
		\STATE Update 
		   $\alpha^{t+1}  = \Pi_{\Omega}[\alpha^t +  \eta' (G^t_{3,\alpha} - \alpha^t)]$ 
		\ENDFOR
	\end{algorithmic}
\end{algorithm}
\subsection{The Unified Algorithm}\vspace*{-0.05in}
Finally, we present the unified algorithm of MIDAM for using the two stochastic pooling estimators shown in Algorithm~\ref{alg:1}. The algorithm design is inspired by momentum-based methods for non-convex-strongly-concave min-max optimization~\cite{guo2021stochastic}. With stochastic gradient estimators in terms of the primal variable $(\w^t, a^t, b^t)$, we compute a moving average of their gradient estimators denoted by $\v^{t+1}$ in Step 7. Then we update the primal variable following the negative direction of $\v^{t+1}=(\v_1^{t+1}, \v_2^{t+1}, \v_3^{t+1})$, which is equivalent to a momentum update. The step size $\eta$ can be also replaced by the adaptive step size of Adam.  For updating the dual variable $\alpha$, the algorithm simply uses the stochastic gradient ascent update followed by a projection onto a feasible domain. 

{\bf Computational Costs:} Before ending this section, we discuss the per-iteration computational costs of the proposed MIDAM algorithm. The sampled instances include $\B^t=\bigcup_{i\in\S^t}\{\B_i^t\}$, where $\S^t=\S^t_+\cup\S^t_-$ denotes the sampled bags, and $\B_i^t$ denotes the sampled instances for the sampled bag $\X_i$. For updating the estimators $s_i^{t+1}, i\in\S^t$, we need to conduct the forward propagations on these sampled instances for computing  their prediction scores $\phi(\w^t; \x_i^j)$ (in smoothed-max pooling and attention-based pooling) and for computing their attentional factor $\exp(\phi_a(\w^t; \x_i^j))$.  For computing the gradient estimators, the main cost lies at the backpropagation for computing $\nabla f_1(\w^t; \B^t_i)$ of $i\in\S^t$, which are required in computing $G^t_{1,\w}, G^t_{2,\w}, G^t_{3,\w}$. Hence, with $S_+=|\S^t_+|$ and $S_-=|\S^t_-|$ and $B=|\B_i^t|$, the total costs of forward propogations and backpropogations are $O((S_++S_-)Bd)$, where $(S_++S_-)B$ is the number of instances of each mini-batch. Hence this cost is independent of the total size of each bag $N_i=|\X_i|$.

\vspace*{-0.05in}
\section{Convergence Analysis}\vspace*{-0.05in}

{\bf Approach of Analysis.} We first would like to point out the considered non-convex min-max multi-level compositional optimization problem is a new problem that has not been studied in the literature. To the best of our knowledge, the two related works are~\cite{yuan2022compositional,pmlr-v162-gao22c}. However, these two works only involve one inner functions to be estimated. In contrast, our problem involves many inner functions $f_1(\w; \X_i)$ to be estimated, while only a few of them are sampled for estimating their stochastic values. To tackle this challenge, we borrow a technique from~\cite{wang2022finite} which was developed for a minimization problem with two-level compositional functions and multiple inner functions. We leverage their error bound analysis for two-level stochastic pooling estimators and combine with that of momentum-based methods for min-max optimization~\cite{guo2021stochastic} to derive our final convergence.  

Since the objective $F(\w,a,b,\alpha)$ in \eqref{eq:MI_DAM} is 1-strongly concave w.r.t. $\alpha$, $\max_{\alpha \in \Omega} F(\w,a,b,\alpha)$ has unique solution and $\nabla \Phi(\w,a,b)$ is Lipschitz continuous if $\nabla F$ is Lipschitz continuous. Following \cite{lin2019gradient,rafique2018non}, we define $\Phi(\w,a,b)\coloneqq \max_{\alpha \in \Omega} F(\w,a,b,\alpha)$ and use $\Norm{\nabla \Phi(\w,a,b)}_2$ as an optimality measure.
\begin{definition}
$(\w,a,b)$ is called an $\epsilon$-stationary point ($\epsilon\geq 0$) of
a differentiable function $\Phi$ if $\Norm{\nabla \Phi(\w,a,b)}_2 \leq \epsilon$. 
\end{definition}\vspace*{-0.05in}
Our theory is established based on the following assumption.
\begin{ass}\label{asm:major}
({\bf Smoothed-max Pooling}) We assume that $\phi(\w;\x)$ is bounded, Lipschitz continuous, and has Lipschitz continuous gradient, i.e. there exist $B_\phi, C_\phi, L_\phi\geq 0$ such that $\Norm{\phi(\w;\x)}_2 \leq B_\phi$, $\Norm{\nabla \phi(\w;\x)}_2 \leq C_\phi$, $\Norm{\nabla^2 \phi(\w;\x)}_2 \leq L_\phi$ for each $\x$.

({\bf Attention-based Pooling}) We assume that $g(\w;\x)$ is bounded, Lipschitz continuous, and has Lipschitz continuous gradient and $\delta(\w;\x)$ is bounded, Lipschitz continuous, and has Lipschitz continuous gradient, i.e., there exist $B_g, C_g, L_g, B_\delta, C_\delta, L_\delta\geq 0$ such that $\Norm{g(\w;\x)}_2 \leq B_g$, $\Norm{\nabla g(\w;\x)}_2 \leq C_g$, $\Norm{\nabla^2 g(\w;\x)}_2 \leq L_g$, $\Norm{\delta(\w;\x)}_2 \leq B_g$, $\Norm{\nabla \delta(\w;\x)}_2 \leq C_g$, $\Norm{\nabla^2 \delta(\w;\x)}_2 \leq L_g$.
\end{ass}
\vspace*{-0.05in}We provide some examples in which the assumption above holds: First, objective \eqref{eq:MI_DAM} with smoothed-max pooling,  $\phi(\w;\x) =\sigma(\w^\top e(\w_e;\x))$, and pre-trained, fixed $\w_e$; Second, objective \eqref{eq:MI_DAM} with bounded weight norms (e.g. $\Norm{\w_e}, \Norm{\w_a}$, $\Norm{V}$) during the training process. Some prior works indicate that the weight norm may be bounded when  weight decay regularization is used~\citep{haochen2022theoretical}. 
\vspace*{-0.1in}\begin{thm}\label{thm:main}
Algorithm~\ref{alg:1} with stepsizes $\beta_1 = \O(\epsilon^2)$, $\gamma_0 = \O(\epsilon^2)$, $\eta = \O\left(\min\left\{\frac{S_+}{D_+},\frac{S_-}{D_-}\right\}\epsilon^2\right)$, $\eta' = \O(\epsilon^2)$ can find an $\epsilon$-stationary point in $T = \O\left(\max\left\{\frac{D_+}{S_+},\frac{D_-}{S_-}\right\}\frac{\epsilon^{-4}}{B}\right)$ iterations, where $S_+=|\S^t_+|$ and $S_-=|\S^t_-|$ and $B=|\B_i^t|$.  Besides, the average estimation error $\frac{1}{T}\sum_{t=0}^{T-1}\EX\left[\Norm{s_i^t - f_1(\w^t;\X_i)}^2\right] \leq \O(\epsilon^2),\forall i$.
\end{thm}
\vspace*{-0.1in}
{\bf Remark:} This theorem states that more bags and larger bag sizes lead to faster convergence of our algorithm, at the cost of more computational resources. The order of complexity $O(1/\epsilon^4)$ is the same as that of non-convex min-max optimization in~\cite{guo2021stochastic}. 
Proofs are in appendix. 
\section{Experiments}\vspace*{-0.05in}
In this section, we present some experimental results. We choose datasets from three categories, namely traditional tabular datasets, histopathological image datasets, and MRI/CT datasets. Statistics of these datasets are described in Tabel~\ref{tab:data-stats}. Details of these datasets will be presented later. 
\begin{table}[t]
  \centering
  \caption{Data statistics for the benchmark datasets}
    \resizebox{\columnwidth}{!}{%
    \begin{tabular}{llrrrr}
    \bottomrule
    \multicolumn{1}{l}{Data Format} & Dataset & \multicolumn{1}{l}{$D_+$} & \multicolumn{1}{l}{$D_-$} & \multicolumn{1}{l}{\makecell{average\\ bag size}} & \multicolumn{1}{l}{\#features} \\
    \hline
     & MUSK1 & 47    & 45    & 5.17  & 166 \\
          & MUSK2 & 39    & 63    & 64.69 & 166 \\
        Tabular  & Elephant & 100   & 100   & 6.1   & 230 \\
          & Fox   & 100   & 100   & 6.6   & 230 \\
          & Tiger & 100   & 100   & 6.96  & 230 \\
          
          \hline
          
    Histopathological & Breast Cancer & 26    & 32    & 672   & 32x32x3 \\
         Image  & Colon Ade. & 100  & 1000  & 256   & 32x32x3 \\
          \hline
    \multirow{3}[0]{*}{MRI/CT Scans} & PDGM  & 403   & 55    & 155   & 240x240x1 \\
          & OCT & 747   & 1935  & 31    & 256x256x1 \\
          \bottomrule
    \end{tabular}%
    }
  \label{tab:data-stats}%
\end{table}%

{\bf Baselines.} We mainly compare with two categories of approaches for MIL with different pooling operators. The first category is optimizing the CE loss by Adam optimizer with mean, smoothed-max (smx), max, attention-based (att) poolings, denoted by CE (XX), where XX is the name of a pooling.  The second category is optimizing the min-max margin AUC loss~\cite{yuan2021large} with the same set of poolings, denoted as DAM (XX). We note that for large-resolution medical image datasets, deterministic pooling is unrealistic due to limits of GPU memory. For example, the CE (att) method could consume about 22 Giga-Bytes GPU memory for PDGM dataset even with a single bag of data. Medical researchers also have raised concern for the GPU constraint of large size histopathologica images~\cite{tizhoosh2018artificial}. Hence, we implement the naive mini-batch based stochastic poolings for baselines, which are denoted by CE (MB-XX) and DAM (MB-XX) with XX being the name of a pooling. For medical image datasets, we also compare with two traditional baselines that treat multiple instances as 3D data in their given order and learn a 3D network by optimizing the CE loss  and the min-max margin AUC loss, which are denoted as CE (3D) and DAM (3D). Our methods are denoted as MIDAM (smx) and MIDAM (att) for using two stochastic pooling operations, respectively.  We fix the margin parameter as $0.1$ for DAM and MIDAM. For attention-based pooling, we use the one defined in~(\ref{eqn:attpoolf}) with an attentional factor $\exp(\w_a^{\top}\text{tanh}(Ve(\w_e; \x)))$ according to~\cite{pmlr-v80-ilse18a}.

\subsection{Results on Tabular Benchmarks}\vspace*{-0.05in}
Five benchmark datasets, namely, MUSK1, MUSK2, Fox, Tiger, Elephant~\cite{dietterich1997solving,andrews2002support}, are commonly used for evaluating MIL methods. For the MUSK1 and MUSK2 datasets, they contain drug molecules that will (or not) bind strongly to a target protein. Each molecule (a bag) may adopt a wide range of shapes or conformations (instances). A positive molecule has at least one shape that can bind well (although it is not known which one) and a negative molecule does not have any shapes that can make the molecule bind well~\cite{dietterich1997solving}. For Fox, Tiger, and Elephant datasets, each object contains features extracted from an image. Each positive bag is a bag that contains the animal of interest~\cite{andrews2002support}. 

We adopt a simple 2-layer feed-forward neural network (FFNN) as the backbone model, whose neuron number equals data dimension. We apply tanh as the activation function for the middle layer and sigmoid as a normalization function for prediction score for computing AUC loss function. We uniformly randomly split the data with 0.9/0.1 train/test ratio and run 5-fold-cross-validation experiments with 3 different random seeds (totally 15 different trials). The initial learning rate is tuned in \{1e-1,1e-2,1e-3\}, and is decreased by 10 fold at the end of the 50-th epoch and 75-th epoch over the 100-epoch-training period. 
For all experiments in this work, the weight decay is fixed as $1e-4$, and we fix $\eta'=1, (1-\beta_1)=0.9$ in our proposed algorithm decreasing by 2 fold at the same time with learning rate. We report the testing AUC based on a model with the largest validation AUC value. For each iteration, we sample 8 positive bags and 8 negative bags ($S_+=S_-=8$), and for each bag sample at most 4 instances for our  methods but use all instances for baselines, given that the dataset is small and bag size is not identical across all bags. The mean and standard deviation of testing AUC are presented in Table~\ref{Tab:tabular}~\footnote{The code is available at~\url{https://github.com/DixianZhu/MIDAM}}.

From the results, we observe that MIDAM (att) or MIDAM (smx) method achieves the best performance on these classical tabular benchmark datasets. This might sound surprising given that the DAM baselines use all instances for each bag for computing the pooling. To understand this, we plot the training and testing convergence curves (shown in Figure~\ref{fig:generalization} in Appendix~\ref{sec:mf} due to limit of space). We find that the better testing peformance of our MIDAM methods is probably due to that the stochastic sampling over instances  prevents overfitting (since training performance is worse) and hence improves the generalization (testing performance is better).  In addition, 
DAM is better than CE except for DAM (att). 

\begin{table}[t]
\caption{The testing AUC on benchmark datasets.}\label{Tab:tabular}
  \centering
     \resizebox{\columnwidth}{!}{\begin{tabular}{lllllll}
    \bottomrule
    Methods & MUSK1 & MUSK2 & Fox   & Tiger & Elephant \\
    \hline
    CE (mean) & 0.803(0.14) & 0.805(0.113) & 0.701(0.116) & 0.822(0.093) & 0.877(0.065) \\
    DAM (mean) & 0.832(0.147) & 0.818(0.079) & 0.647(0.111) & 0.842(0.085) & 0.897(0.053) \\
    CE (max) & 0.678(0.121) & 0.84(0.106) & 0.657(0.147) & 0.855(0.094) & 0.885(0.044) \\
    DAM (max) & 0.739(0.126) & 0.859(0.09) & 0.595(0.159) & 0.858(0.06) & 0.902(0.073) \\
    CE (smx) &  0.769(0.121) & 0.851(0.111)  &  0.668(0.117) & 0.865(0.078)  &  0.902(0.068) \\
    DAM (smx) & 0.806(0.118)  & 0.854(0.108)  &  0.66(0.138) &  0.867(0.07) & 0.902(0.052)  \\
    CE (att) & 0.808(0.112) & 0.76(0.122) & 0.705(0.13) & 0.834(0.09) & 0.883(0.092) \\
    DAM (att) & 0.768(0.139) & 0.757(0.154) & 0.69(0.123) & 0.848(0.067) & 0.872(0.074) \\
    \hline 
    MIDAM (smx) & \textbf{0.834(0.12)} & \textbf{0.905(0.068)} & 0.622(0.188) & 0.861(0.071) & 0.873(0.104) \\
    MIDAM (att) & 0.826(0.107) & 0.843(0.107) & \textbf{0.733(0.097)} & \textbf{0.867(0.066)} & \textbf{0.906(0.069)} \\
    \bottomrule
    \end{tabular}}%
\caption{The testing AUC on medical image datasets.}\label{Tab:meddata}
  \centering   
     \resizebox{\columnwidth}{!}{\begin{tabular}{lllll}
     \hline
          & \multicolumn{2}{c}{Histopathological Image} & \multicolumn{2}{c}{MRI/OCT 3D-Image} \\
          \hline
    Methods & Breast Cancer& Colon Ade.& PDGM& OCT\\
    \hline
    CE (3D) & 0.925(0.061)        &   0.724(0.165)    & 0.582(0.118) &  0.789(0.032)       \\
    DAM (3D) & 0.725(0.2)      &   0.846(0.075)    & 0.545(0.122) &    0.807(0.027)   \\
    \hline 
    CE (MB-mean) & 0.85(0.242)       &   0.883(0.042)    & 0.616(0.023) &  0.799(0.019)     \\
    DAM (MB-mean) & 0.875(0.137)      &   0.877(0.017)    & 0.635(0.113) &   0.839(0.029)     \\
    CE (MB-max) & 0.325(0.232)       &   0.856(0.032)    & 0.462(0.108) & 0.793(0.047)      \\
    DAM (MB-max) & 0.475(0.215)      &  0.825(0.044)    & 0.624(0.112) &  0.841(0.01)      \\
    CE (MB-smx) & 0.575(0.127)  & 0.863(0.031)  &  0.491(0.111) &  0.826(0.018)  \\
    DAM (MB-smx) & 0.725(0.184)  &  0.905(0.01) &  0.659(0.058) &  0.829(0.008)   \\
    CE (MB-att) & 0.9(0.146)      &    0.9(0.042)   & 0.564(0.072)&   0.823(0.017)     \\
    DAM (MB-att) & 0.875(0.112)      &  0.882(0.029)     & 0.624(0.112) &   0.842(0.013)    \\
    \hline 
    MIDAM-smx & 0.875(0.137)      &  \textbf{0.91(0.02)}   & \textbf{0.669(0.032)} &  \textbf{0.848(0.01)}       \\
    MIDAM-att & \textbf{0.95(0.1)}     &    0.893(0.08)   & 0.635(0.052) &   0.843(0.012)     \\
    \bottomrule
    \end{tabular}}%
\end{table}%

\setlength{\subfigtopskip}{-0.15cm}\begin{figure*}[h]
\vspace{-0.1 in}
    \centering
    \subfigure[att, Breast Cancer]{\includegraphics[scale=0.16]{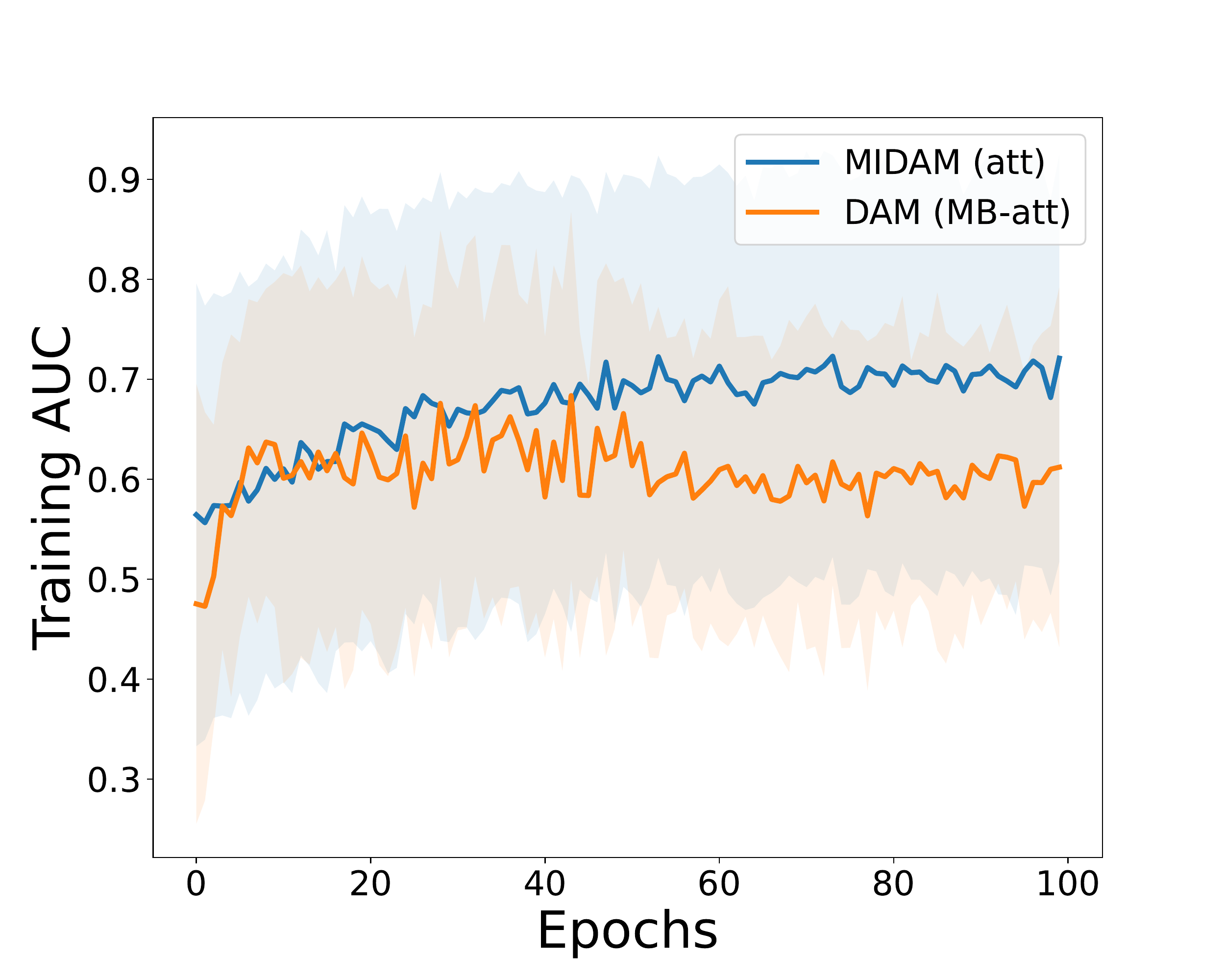}}
    \subfigure[att, Breast Cancer]{\includegraphics[scale=0.16]{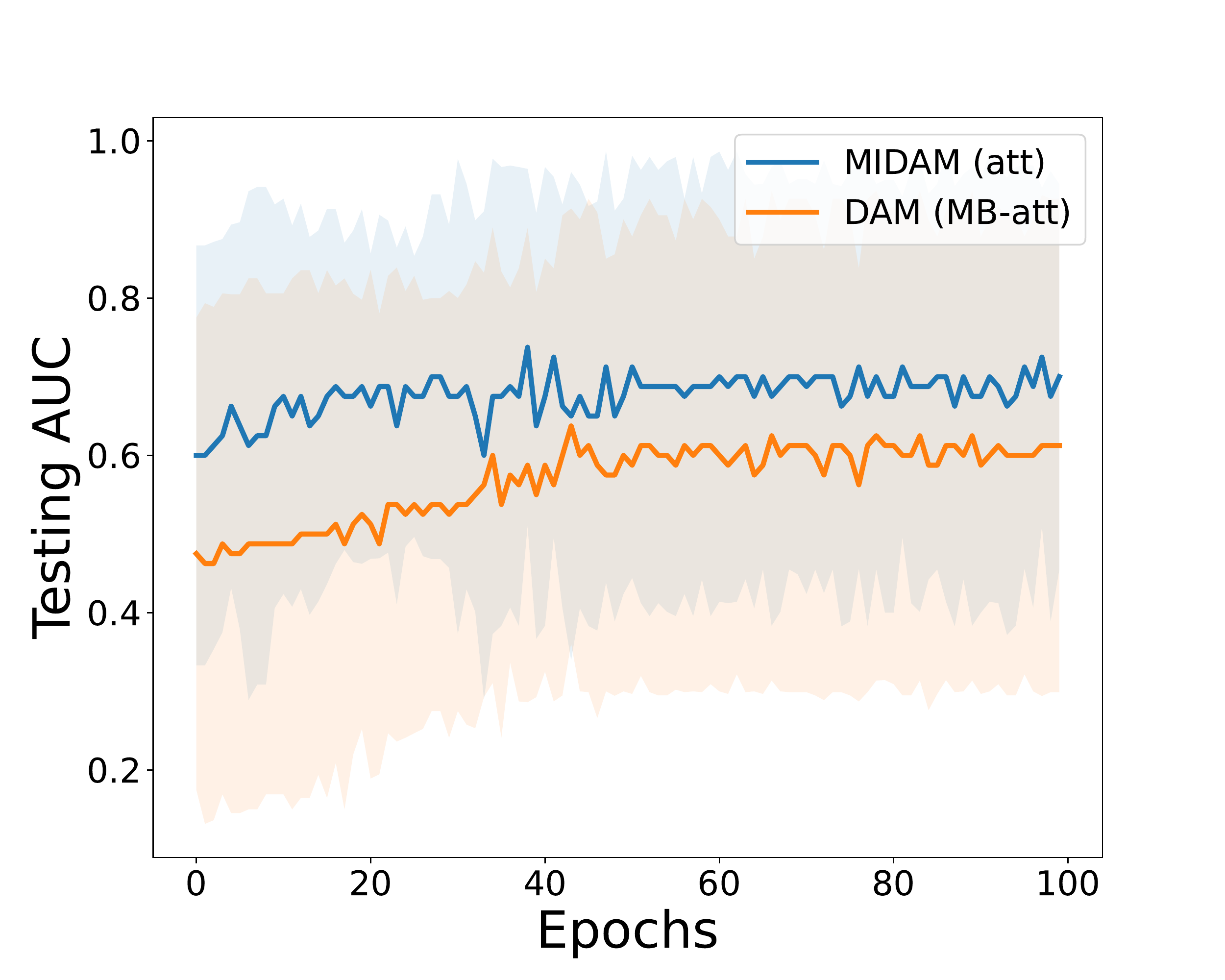}}
    \subfigure[MIDAM (smx), MUSK2]{\includegraphics[scale=0.16]{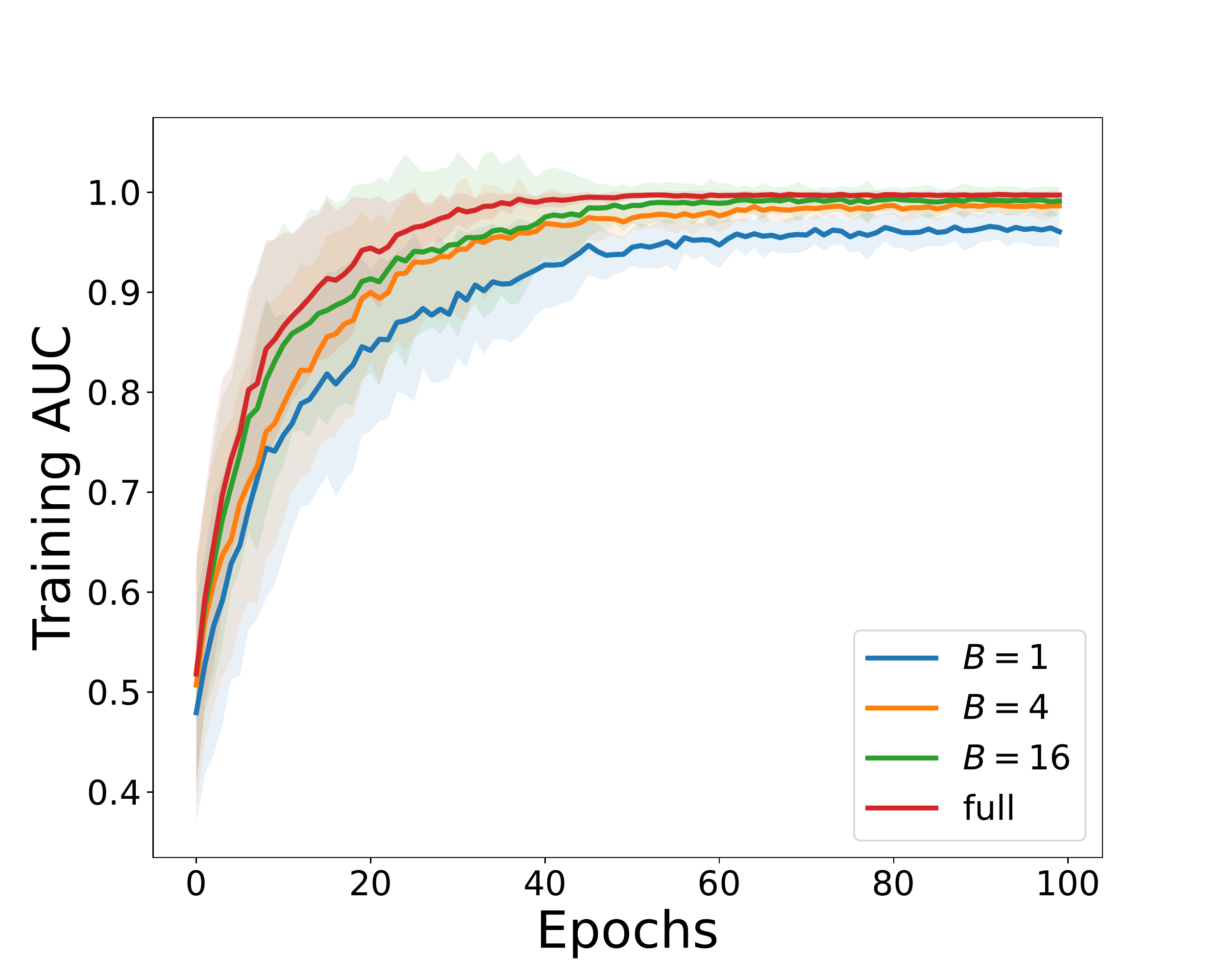}}
    \subfigure[MIDAM (att), MUSK2]{\includegraphics[scale=0.16]{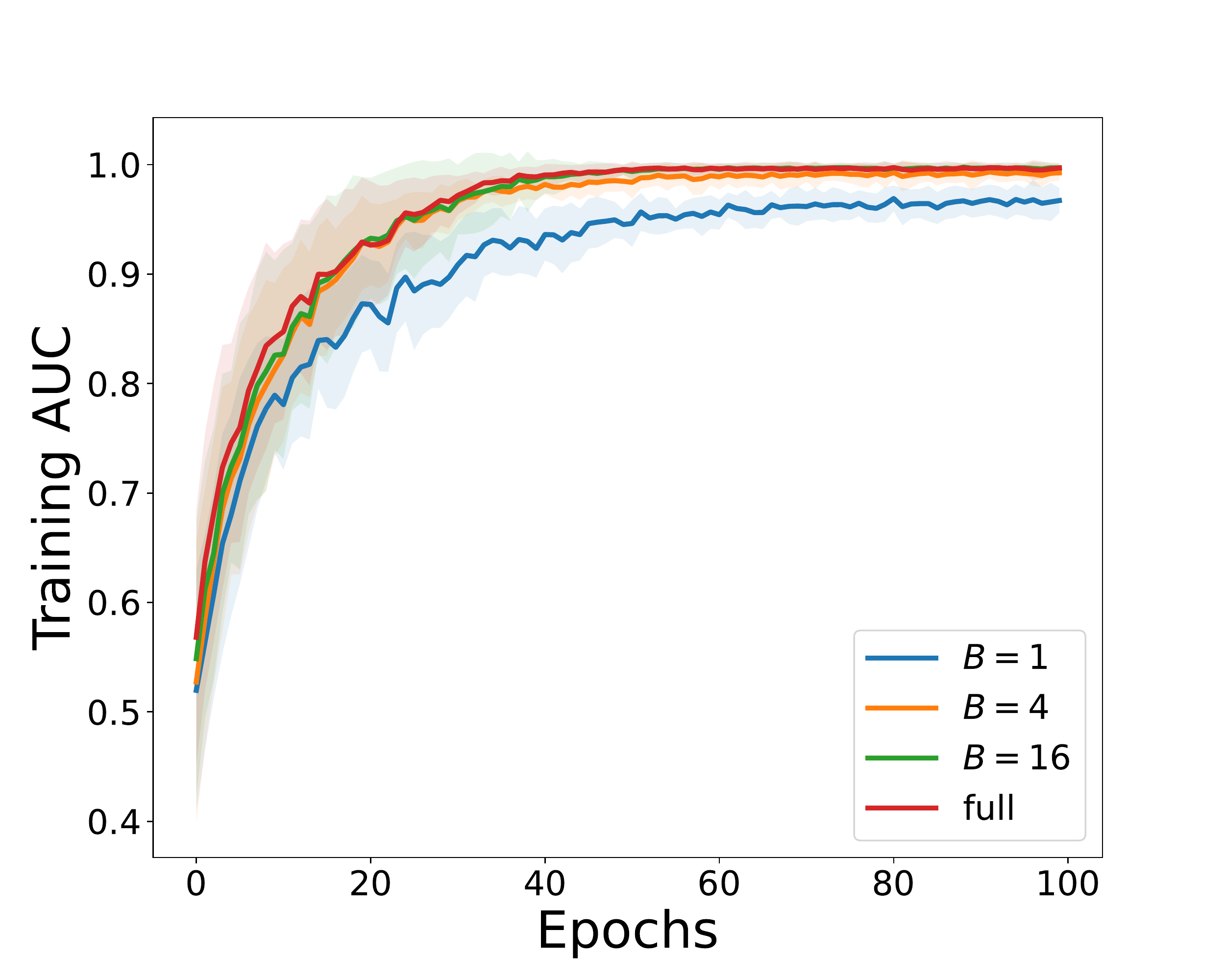}}
\vskip -0.15in
\subfigure[MIDAM (att),  Colon Ade.]{\includegraphics[scale=0.16]{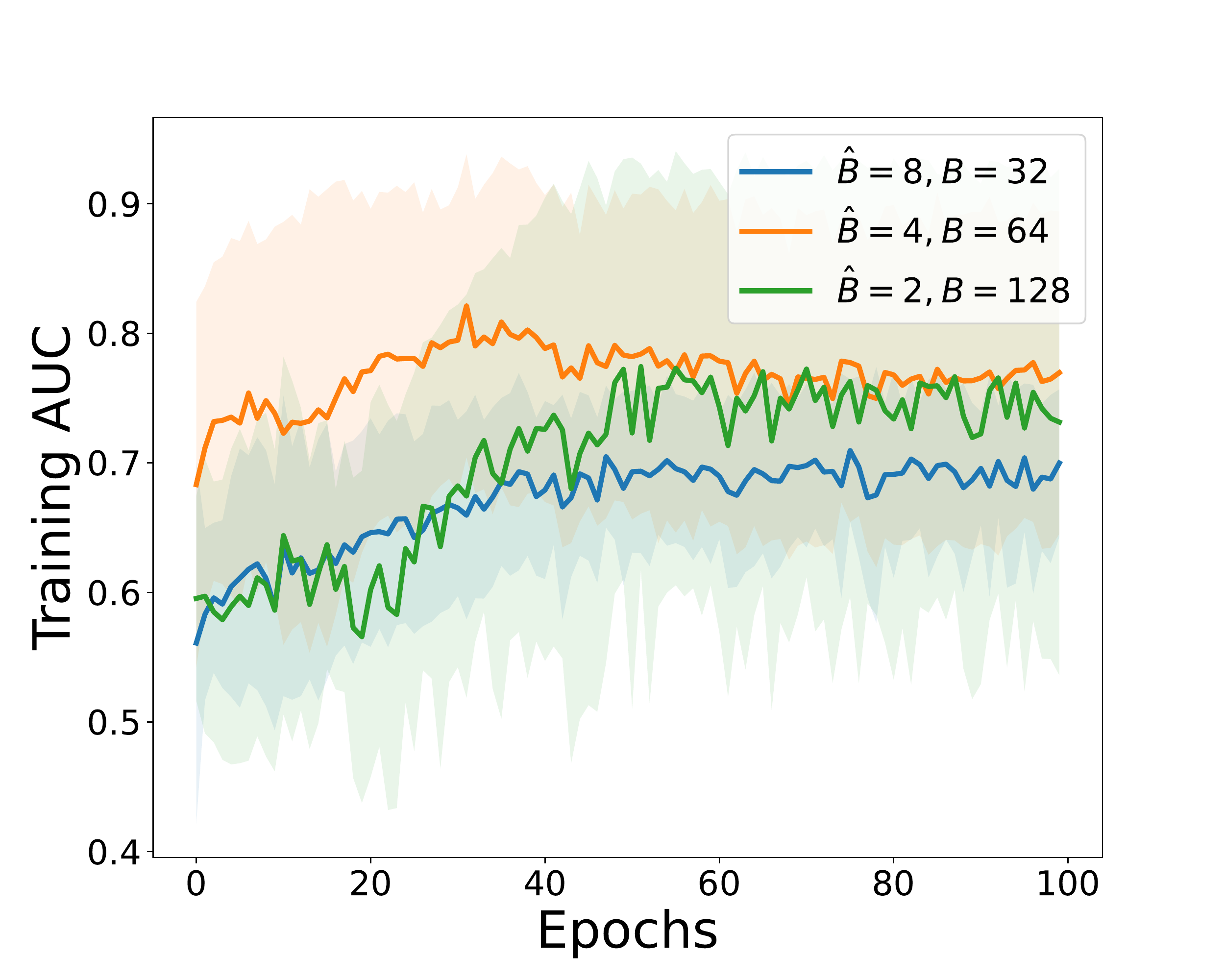}}
    \subfigure[MIDAM (att), Colon Ade.]{\includegraphics[scale=0.16]{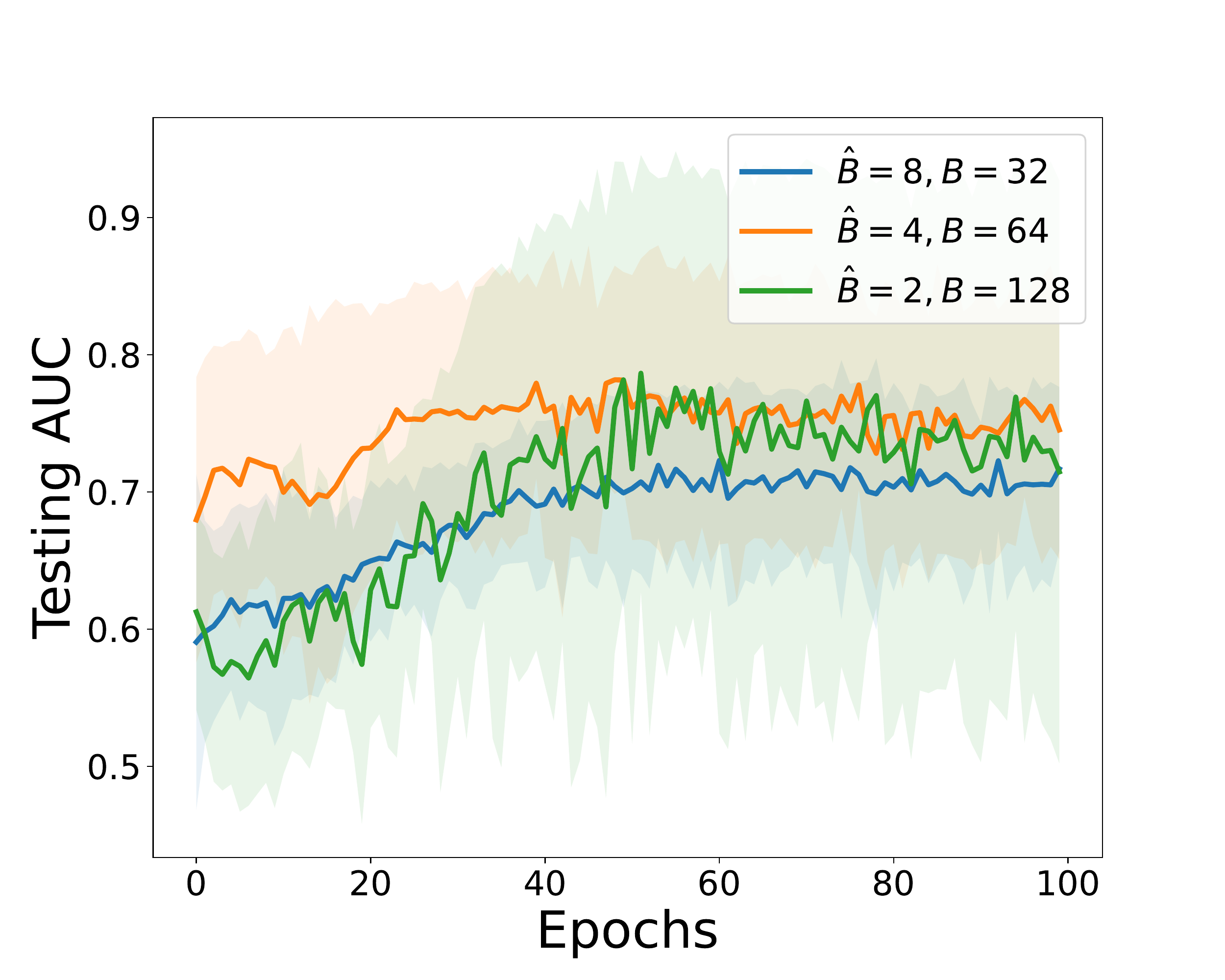}}
    \subfigure[MIDAM (smx), Colon Ade.]{\includegraphics[scale=0.16]{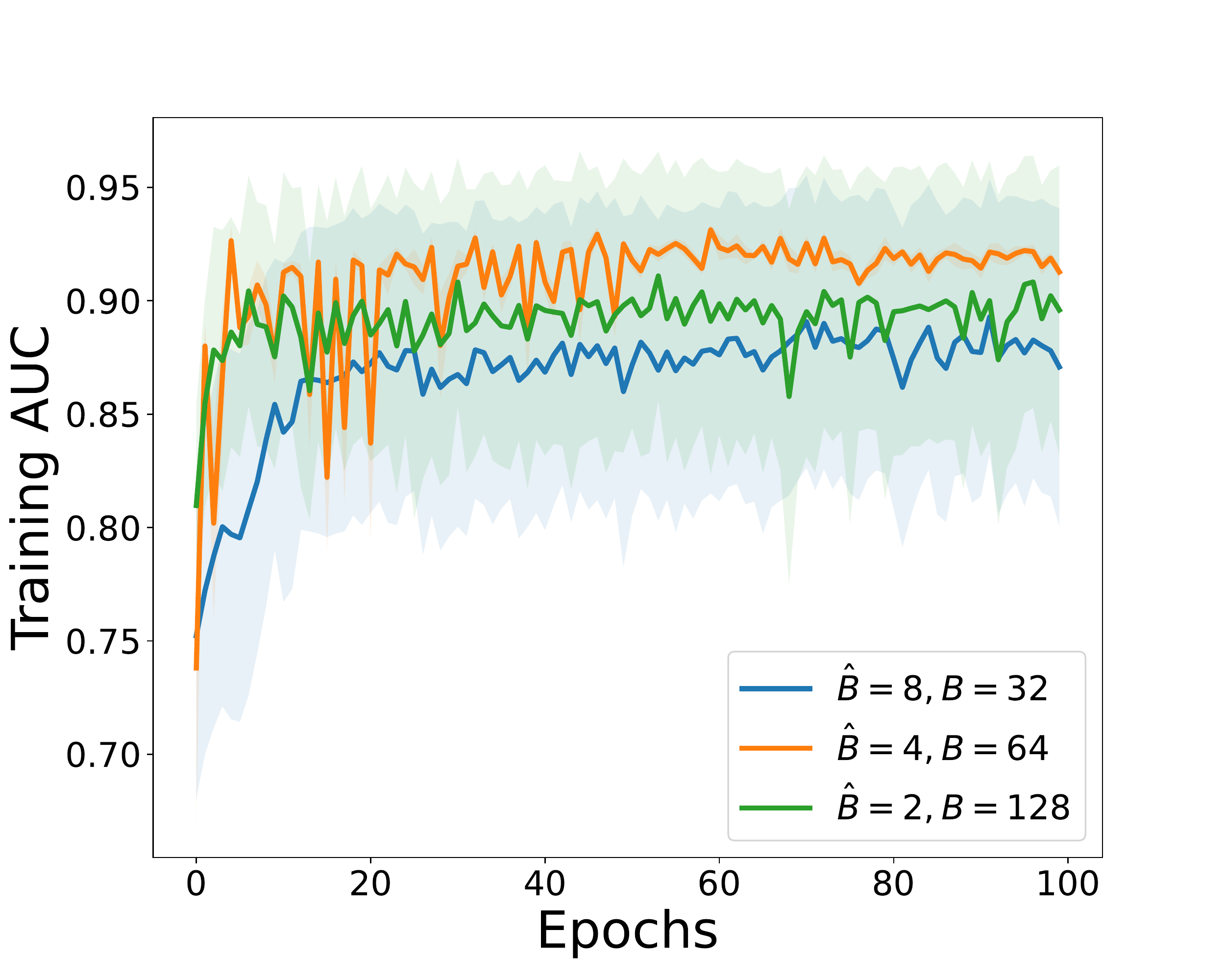}}
    \subfigure[MIDAM (smx), Colon Ade.]{\includegraphics[scale=0.16]{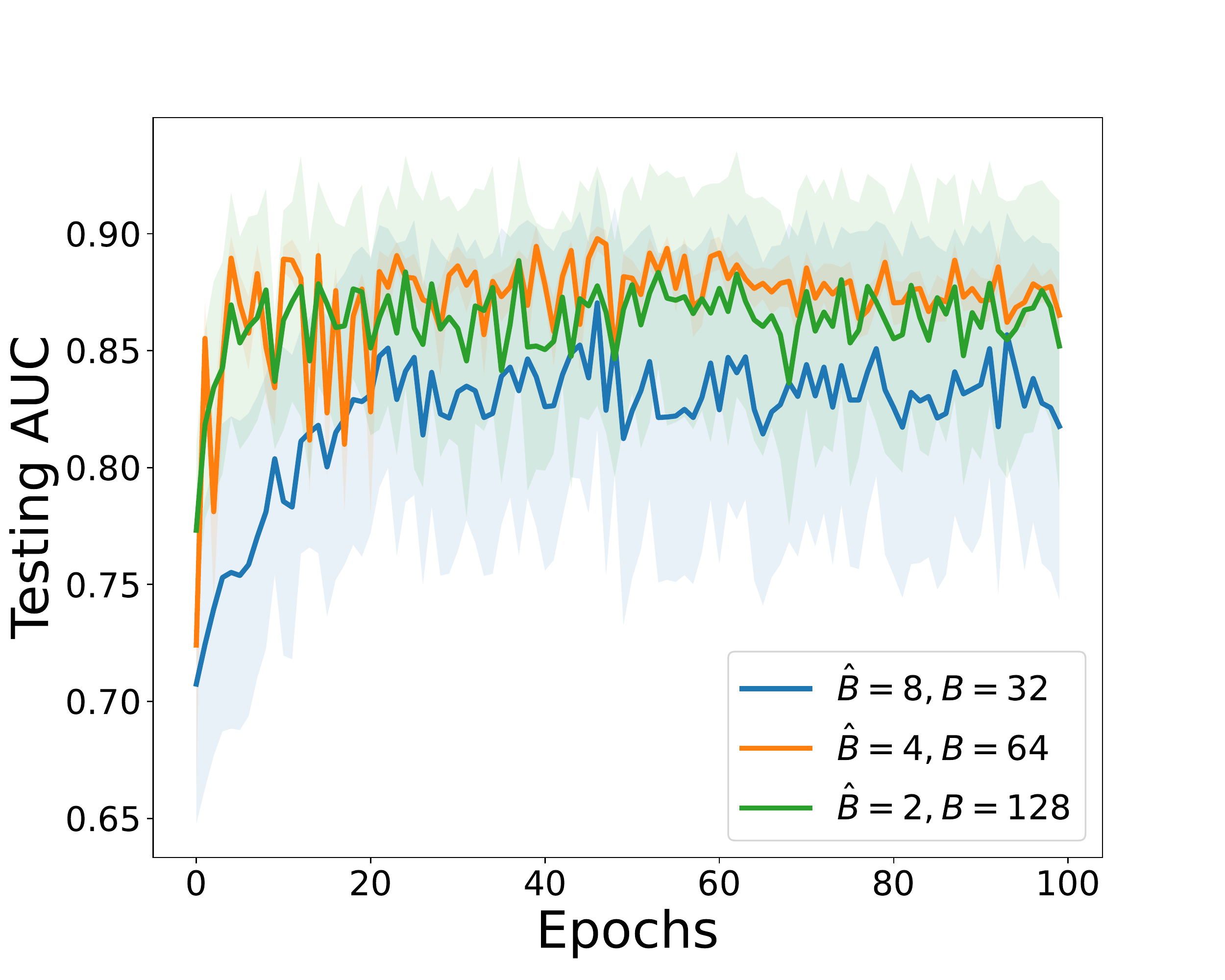}}
        \caption{(a, b): Convergence of training AUC for MIDAM (att) and DAM (MB-att) on Breast Cancer data with margin $c=0.1$ and learning rate tuned in \{1e-1,1e-2,1e-3\}; (c, d): Convergence of training AUC with different  instance-batch size by fixing bag-batch size $S_+=S_-=8$ on MUSK2 data; (e,f,g,h): Convergence of training and testing AUC with different bag-batch sizes $S_+=S_-=\frac{\hat B}{2}$ and  instance-batch sizes ($B$) on Colon Ade. data.} \label{fig:abs}
\end{figure*}


\begin{figure*}[htpb]
    \centering
   \subfigure[Positive image]{\includegraphics[scale=0.35]{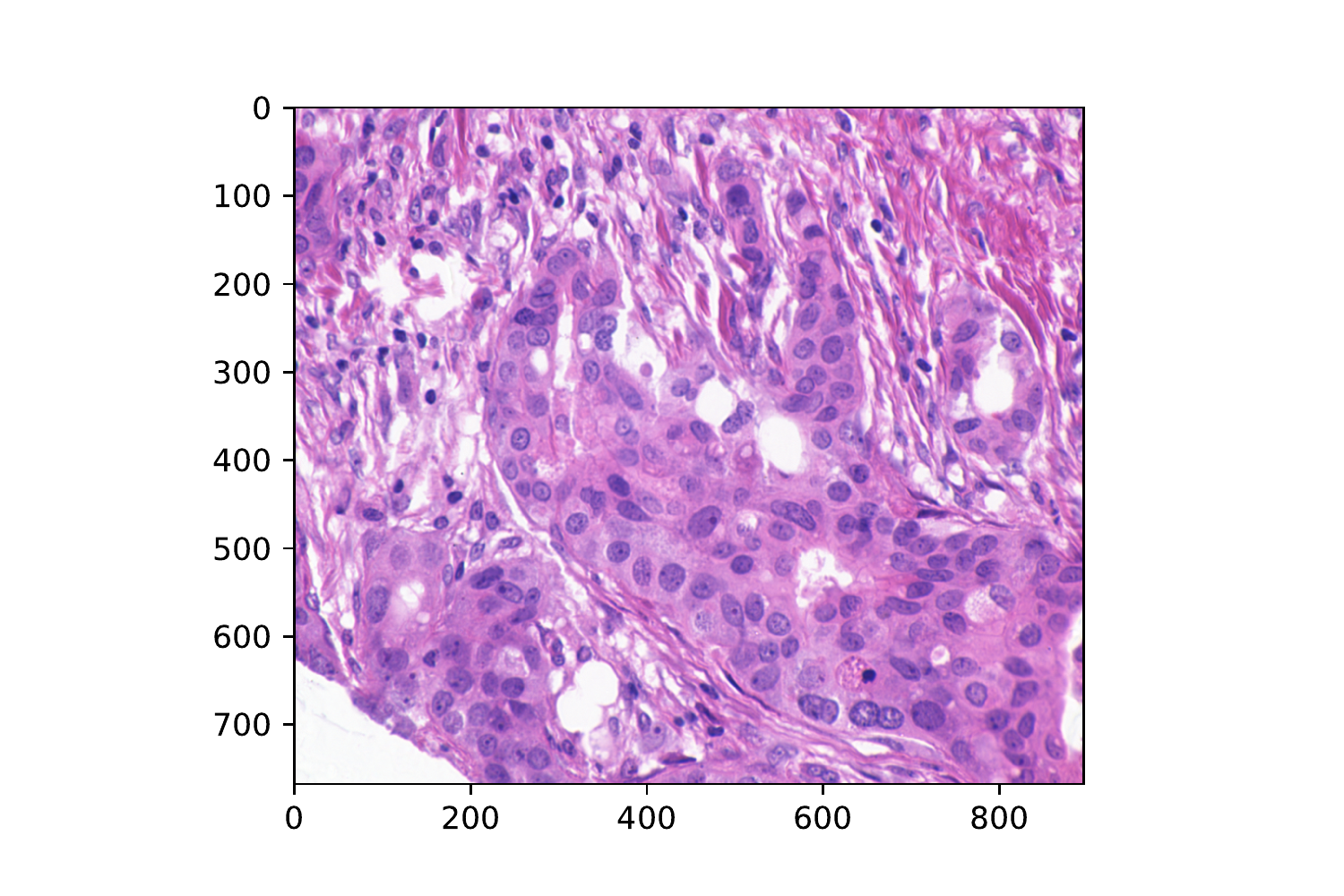}}
    \subfigure[Prediction scores]{\includegraphics[scale=0.35]{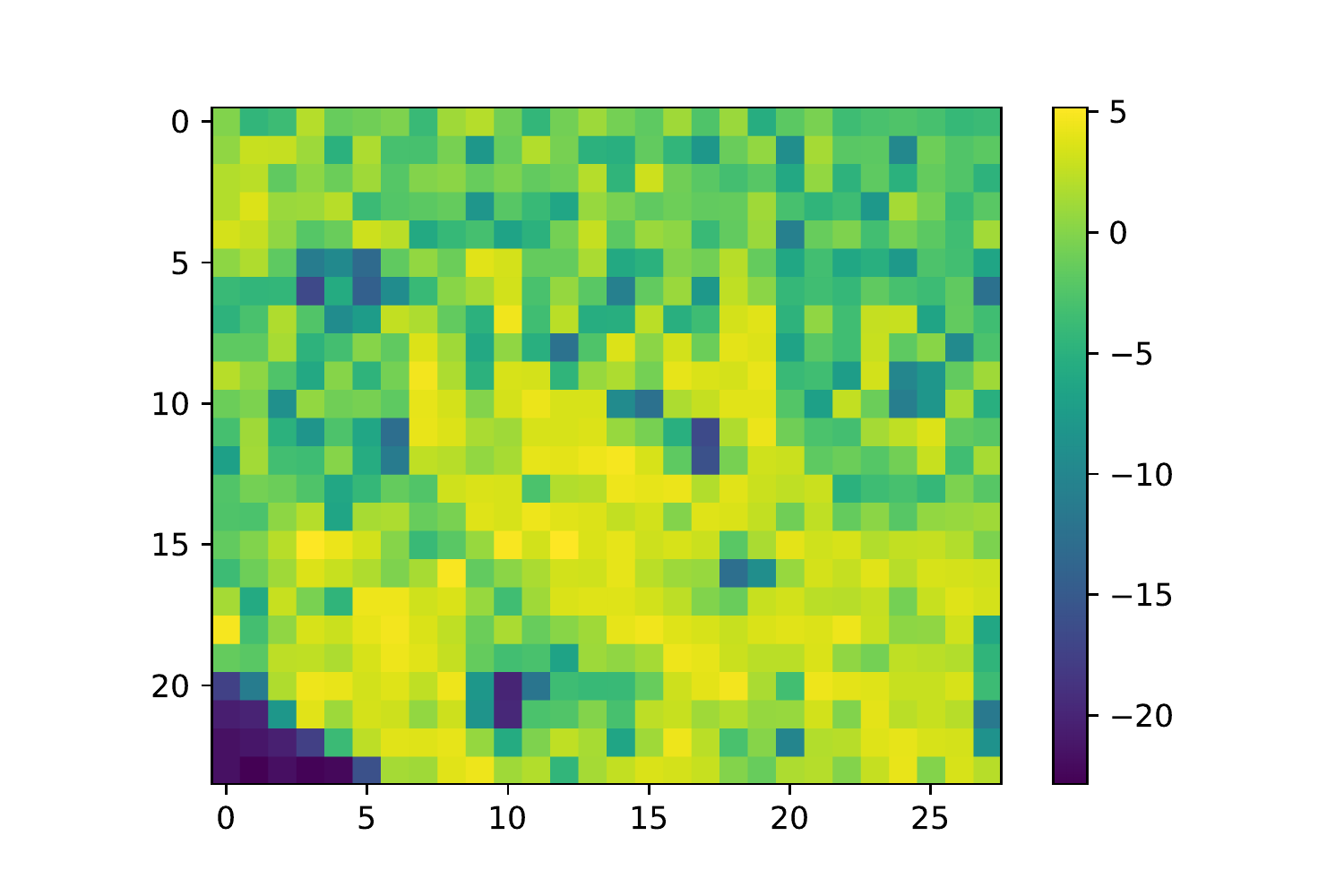}}
    \subfigure[Attention weights]{\includegraphics[scale=0.35]{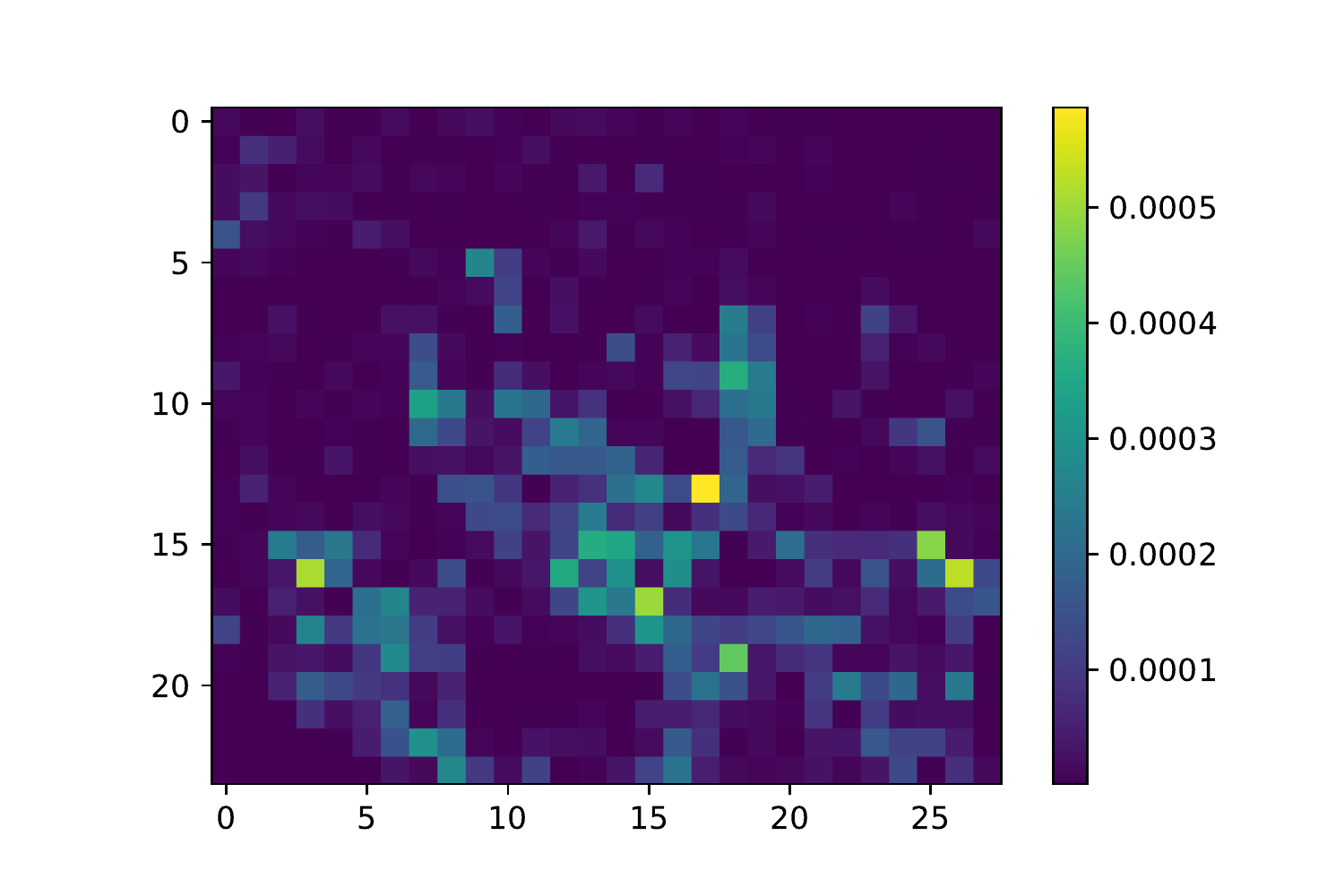}}

        \caption{Demonstration for positive example for Breast Cancer dataset. Left: original image. Middle: prediction scores for each patch. Right: attention weights for each patch.} \label{fig:attention-elaborate-positive}
\end{figure*}

\subsection{Experiments on Medical Image datasets}
We choose two histopathological image datasets, namely Breast Cancer and Colon Adenocarcinoma~\cite{gelasca2008evaluation,borkowski2019lung}. These have been used in previous deep MIL works~\cite{pmlr-v80-ilse18a} for evaluation. Histopathological images are microscopic images of the tissue for disease examination, which are prevalent for cancer diagnosis~\cite{borkowski2019comparing}.  Since  histopathological images  have a high resolution, it is difficulty to  process the whole image. Hence MIL approaches are appealing that treat each image as a bag of local small batches. For Breast Cancer, there are 58 weakly labeled $896\times768$ hematoxylin and eosin (H\&E) stained whole-slide images. An image is labeled malignant if it contains breast cancer cells, otherwise it is benign (examples shown in Figure~\ref{fig:bc-examples}). We divide every image into $32\times32$ patches. This results in 672 patches per bag. For Colon Adenocarcinoma dataset\footnote{\url{https://www.kaggle.com/datasets/biplobdey/lung-and-colon-cancer}}, there are originally 5000 (H\&E) images for benign colon tissue and 5000 for Colon Adenocarcinoma. We uniformly randomly sample  1000 benign images and 100 Adenocarcinoma images to form the new Colon Ade. dataset for our study. We divide every $512\times512$ image into $32\times32$ patches and get 256 patches for each image. We also use two real-world MRI/OCT image datasets. The first data set is from the University of California San Francisco Preoperative Diffuse Glioma MRI (UCSF-PDGM)~\cite{calabrese2022university}, short as PDGM in this work. The problem is to predict patients with grade II or grade IV diffuse gliomas. The second dataset contains multiple OCT images for a large number of patients~\cite{xie2022arterial}. The goal is to predict hypertension from OCT images, which is useful for physicians to understand the relationship between eye-diseases and Hypertension. Exemplar images of two datasets are shown in Figure~\ref{fig:bc-examples} in the Appendix~\ref{sec:mf}.

For all the medical images, we adopt ResNet20 as the backbone model. For AUC loss function, we apply sigmoid as normalization for the output. The weight decay is fixed as 1e-4. For all the experiments, we run 100 epochs for each trial and decrease learning rate by 10 fold at the end of the 50-th epoch
and 75-th epoch. For the Breast Cancer dataset, we generate data train/test (0.9/0.1) splitting 2 times with different random seeds and conduct five-fold cross validation (10 trials). For the other datasets, we do single random train/test (0.9/0.1) splitting and conduct five-fold cross validation (5 trials). The margin parameter is tuned in \{0.1,0.5,1.0\} for AUC loss function. The initial learning rate is tuned in \{1e-1,1e-2,1e-3\} for histopathological image datasets, and is fixed as 1e-2 for PDGM, 1e-1 for OCT.

The results are shown in Table~\ref{Tab:meddata}. From these results, we make the following observations. (1) Our MIDAM still performs the best. Although CE (3D) on Breast Cancer, and CE (MB-att) on the two histopathological image datasets are competitive, most of the CE based approaches are less competitive with DAM based approaches.  (2) By comparing MIDAM (att) with DAM (MB-att) and MIDAM (smx) with DAM (MB-smx), our methods perform consistently better. This confirms the importance of using variance-reduced stochastic pooling operations instead of the naive mini-batch based stochastic poolings. This can be also verified by comparing their training/testing convergence in Figure~\ref{fig:abs}(a,b) and Figure~\ref{fig:faster-cvg} in the Appendix. (3) For Breast Cancer data, our method MIDAM (att) performs better than all baseline methods.  It is notable that CE (3D) and CE (MB-att) are competitive approaches but still have worse performance than MIDAM (att). On Colon Ade. dataset, our method MIDAM (smx) performs the best and CE (MB-att) is still competitive. Finally, we see that there is no clear winner between MIDAM (att) and MIDAM (smx). (4) In general, MIL pooling based methods can achieve better performance than the traditional baseline using 3D data input. 
Hence, our MIDAM algorithms are a good fit for 3D medical images. 

\vspace*{-0.05in}
\subsection{Ablation Studies}\vspace*{-0.05in}
First, we conduct an experiment to study the influence for different instance-batch sizes ($B$) on four tabular datasets. The results on MUSK2 are shown in Figure~\ref{fig:abs} (c,d) with more plotted in Figure~\ref{fig:favor-larger-instance-batch-size}, which demonstrate our methods converge faster with a larger $B$ with fixed bag-batch size $S_+$ and $S_-$. In addition, we observe that with $B=4$  MIDAM converges to almost same level as using all instances in 100 epochs, even for the MUSK2 dataset with average bag size as 64.69.  Second, we show an ablation study on the two  histopathological image datasets that fixes the total budget for bag-batch-size$\times$instance-batch-size. Exemplar results are plotted in Figure~\ref{fig:abs} (e,f,g,h)  with more results plotted  in Figure~\ref{fig:ablation-budget}. We can see that due to sampling of instances per-bag, we have more flexibility to choose the bag-batch size $S_+=S_-=\hat B/2$ and instance-batch size $B$ to have faster training, e.g., with $\hat B=4, B=64$ MIDAM converges the fastest, which demonstrates the superiority of our design. {Third, we demonstrate the effectiveness of stochastic attention pooling based MIDAM on a Breast Cancer example by attention weights and prediction scores for each instance (image patch). The results are presented in Figure~\ref{fig:attention-elaborate-positive}, where we can observe the lesion parts for the histopathology tissue have larger prediction scores and attention weights (the brighter patches). More demonstration on a negative examples are included in Figure~\ref{fig:attention-elaborate} in Appendix~\ref{sec:mf}, where the attention module focus on a blank patch to generate low overall prediction score.}    

\section{Conclusions}
We have proposed efficient algorithms for multi-instance deep AUC maximization. Our algorithms are based on variance-reduced stochastic poolings in a spirit of compositional optimization to enjoy a provable convergence. We have demonstrated the effectiveness and superiority of our algorithms on benchmark datasets and real-world high-resolution medical image datasets.

\section{Acknowledgement}
{This work is partially supported by NSF Career Award 2246753, NSF Grant 2246757 and NSF Grant 2246756.}

\bibliography{reference}
\bibliographystyle{icml2023}

\appendix
\onecolumn

\section{Technical Lemmas}

\begin{lemma}
    Based on Assumption~\ref{asm:major}, we have that $h(\w;\X)$ is bounded, Lipschitz continuous, and has Lipschitz continuous gradient, i.e., there exists $B_h, C_h, L_h\geq 0$ such that $|h(\w;\X)|\leq B_h$, $\Norm{\nabla h(\w;\X)}_2\leq C_h$, and $\nabla^2 h(\w;\X) \preceq L_h I$.
\end{lemma}
\begin{proof}
{\bf Smoothed-max pooling:} Property of the LogSumExp (LSE) function implies that
\begin{align*}
|h(\w;\X)| \leq \tau \max_{\x \in \X} \frac{|\phi_\w(\x)|}{\tau} + (1-\tau)\log |\X|.      
\end{align*}
The norm of $\nabla h(\w;\X)$ can be bounded as
\begin{align*}
\Norm{\nabla h(\w;\X)}_2 = \Norm{\sum_{\x\in\X} \frac{\exp(\phi(\w;\x)/\tau)}{\sum_{\x\in\X}\exp(\phi(\w;\x)/\tau)}\nabla \phi(\w;\x)}_2\leq C_\phi. 
\end{align*}
The norm of $\nabla^2 h(\w;\x)$ can be bounded as
\begin{align*}
& \Norm{\nabla^2 h(\w;\x)}_2  \leq \Norm{\sum_{\x\in\X} \frac{\exp(\phi(\w;\x)/\tau)}{\sum_{\x\in\X}\exp(\phi(\w;\x)/\tau)}\left(\nabla \phi(\w;\x) [\nabla \phi(\w;\x)]^\top/\tau + \nabla^2 \phi(\w;\x)\right)}_2\\
& \quad\quad + \Norm{\left(\sum_{\x\in\X} \frac{\exp(\phi(\w;\x)/\tau)}{\sum_{x\in \X}\exp(\phi(\w;\x)/\tau)}\nabla \phi(\w;\x)\right)\left(\sum_{\x\in\X} \frac{\exp(\phi(\w;\x)/\tau)}{\sum_{x\in \X}\exp(\phi(\w;\x)/\tau)}[\nabla \phi(\w;\x)]^\top/\tau\right)}_2 \leq 2C_\phi^2/\tau + L_\phi.
\end{align*}
{\bf Attention-based pooling:} According to \eqref{eqn:attpool}, it is clear that $|h(\w;\X)| \leq 1$. The norm of $\nabla h(\w;\X)$ can be bounded as
\begin{align*}
\Norm{\nabla h(\w;\X)}_2 & \leq 0.25 \Norm{\sum_{x\in\X}\frac{\exp(g(\w;\x))\delta(\w;\x)}{\sum_{\x'\in \X} \exp(g(\w;\x'))}\nabla g(\w;\x)}_2 + 0.25\Norm{\sum_{\x\in\X}\frac{\exp(g(\w;\x))}{\sum_{\x'\in\X}\exp(g(\w;\x'))}\nabla \delta(\w;\x)}_2\\
& \quad\quad + 0.25\left|\sum_{\x\in\X} \frac{\exp(g(\w;\x))\delta(\w;\x)}{\sum_{\x'\in\X}\exp(g(\w;\x'))}\right|\Norm{\sum_{\x\in\X}\frac{\exp(g(\w;\x))}{\sum_{\x'\in \X}\exp(g(\w;\x'))}\nabla g(\w;\x)}_2 \leq 0.5 C_g B_\delta + 0.25 C_\delta.   
\end{align*}
For brevity, we denote the softmax function $s(\w;\x) \coloneqq \frac{\exp(g(\w;\x))}{\sum_{x'\in \X} \exp(g(\w;\x'))}$. The norm of $\nabla^2 h(\w;\x)$ can be bounded as
\begin{align*}
& \Norm{\nabla^2 h(\w;\x)}_2\\ &\leq 0.1 \Norm{\nabla h(\w;\X)[\nabla h(\w;\X)]^\top}_2 + 0.25 \Norm{\sum_{\x\in\X}s(\w;\x)(\delta(\w;\x)\nabla g(\w;\x) + \nabla \delta(\w;\x))[\nabla g(\w;\x)]^\top}_2 \\
& + 0.25 \Norm{\sum_{\x\in\X}s(\w;\x)\left(\nabla g(\w;\x)[\nabla \delta(\w;\x)]^\top + \delta(\w;\x)\nabla^2 g(\w;\x) + \nabla^2\delta(\w;\x)\right)}_2\\
&  + 0.25\Norm{\left(\sum_{\x\in\X}s(\w;\x) (\delta(\w;\x)\nabla g(\w;\x) + \nabla \delta(\w;\x))\right)\left(\sum_{\x\in\X}s(\w;\x) (\delta(\w;\x)\nabla g(\w;\x) + \nabla \delta(\w;\x))[\nabla g(\w;\x)]^\top\right)}_2\\
& + 0.25 \Norm{\left(\sum_{\x\in\X}s(\w;\x)\nabla g(\w;\x)\right)\left(\sum_{\x\in\X}s(\w;\x)(\delta(\w;\x)\nabla g(\w;\x) + \nabla \delta (\w;\x))\right)^\top}_2\\
& +0.25\Norm{\left(\sum_{\x\in\X}s(\w;\x)\nabla g(\w;\x)\right)\left(\sum_{\x\in\X} s(\w;\x) \delta(\w;\x)\right)\left(\sum_{\x\in\X} s(\w;\x) \nabla g(\w;\x)\right)^\top}_2\\
& +0.25\Norm{\left(\sum_{\x\in\X}s(\w;\x)\delta(\w;\x)\right)\left(\sum_{x\in\X} s(\w;\x)(\nabla g(\w;\x)[\nabla g(\w;\x)]^\top + \nabla^2 g(\w;\x)) \right)}_2\\
& + 0.25\Norm{\left(\sum_{\x\in\X}s(\w;\x)\nabla g(\w;\x)\right)\left(\sum_{\x\in\X}s(\w;\x) \nabla g(\w;\x)\right)^\top}_2\\
& \leq 0.1 C_h^2 + 0.5 (B_\delta C_g + C_\delta)C_g + 0.25 (C_g C_\delta + B_\delta L_g + L_\delta) \\
&\quad\quad+ 0.25 (B_\delta C_g + C_\delta)^2C_g + 0.25 C_g^2 (B_\delta + 1) + 0.25B_\delta (C_g^2 + L_g).
\end{align*}
\end{proof}

\begin{lemma}\label{lem:s_bounded}
Under Assumption~\ref{asm:major},  MIDAM with $\gamma\in (0,1)$, $s_i^0=0$, we have $|s_i^t|\leq B_s$ for all $t>0$.
\end{lemma}
\begin{proof}
This lemma follows from Assumption~\ref{asm:major} and the facts that $f_1$ is continuously differentiable on its domain and the update formula of $s_i$ is a convex combination.
\end{proof}

\begin{lemma}\label{lem:a_b_bounded}
If $\eta \in (0,0.5)$ and $a^0 = 0$, $b^0 = 0$,  there exist $B_a, B_b > 0$ $|a^t|\leq B_a$, $|b^t|<B_b$ for all $t>0$.  
\end{lemma}
\begin{proof}
Note that $G_{1,a}^t = - \frac{2}{|\S_+^t|}\sum_{i\in \S_+^t} (f_2(s_i^t) - a^t)$ and $G_{2,b}^t = - \frac{2}{|\S_-^t|}\sum_{i\in \S_-^t} (f_2(s_i^t) - b^t)$. Thus, the update formulae of $a$ and $b$ can be re-rewritten as
\begin{align*}
& a^{t+1} = a^t - \eta G_{1,a}^t = (1-2\eta)a^t + 2\eta \frac{1}{|\S_+^t|}\sum_{i\in \S_+^t} f_2(s_i^t),\\
& b^{t+1} = b^t - \eta G_{2,b}^t = (1-2\eta)b^t + 2\eta \frac{1}{|\S_-^t|}\sum_{i\in \S_-^t} f_2(s_i^t).
\end{align*}
Due to Lemma~\ref{lem:s_bounded} and the fact that $f_2$ is continuously differentiable on its domain, $a^t$ and $b^t$ are bounded in all iterations as long as $\eta\in(0,0.5)$ such that the update formulae of $a^t$ and $b^t$ are convex combinations.
\end{proof}

\begin{lemma}\label{lem:smooth_F}
Under Assumption~\ref{asm:major}, there exists $L_F>0$ such that $\nabla F$ is $L_F$-Lipschitz continuous.   
\end{lemma}
\begin{proof}
Note that $h(\w;\X_i) = f_2(f_1(\w;\X_i))$, $\nabla h(\w;\X_i) = \nabla f_1(\w;\X_i) \nabla f_2(f_1(\w;\X_i))$. 
For distinct $(\w,a,b,\alpha)$ and  $(\w',a',b',\alpha')$, we have
\begin{align*}
& \Norm{\nabla_{(\w,a,b)} F(\w,a,b,\alpha) - \nabla_{(\w,a,b)}  F(\w',a',b',\alpha')}_2 +|\nabla_\alpha F(\w,a,b,\alpha) - \nabla_\alpha F(\w',a',b',\alpha')|\\
& \leq \Norm{\frac{2}{|\D_+|}\sum_{i\in \D_+} \left(
\nabla h(\w;\X_i)(h(\w;\X_i) - a) -  \nabla h(\w';\X_i) (h(\w';\X_i) - a')\right)}_2\\
& + \Norm{\frac{2}{|\D_-|}\sum_{i\in \D_-} \left(\nabla h(\w;\X_i) (h(\w;\X_i) - b) - \nabla h(\w';\X_i) (h(\w';\X_i) - b')\right)}_2\\
& + \Norm{\alpha \frac{1}{|\D_-|}\sum_{i\in \D_-}\nabla h(\w;\X_i) - \alpha' \frac{1}{|\D_-|}\sum_{i\in \D_-}\nabla h(\w';\X_i)}_2 + \Norm{\alpha \frac{1}{|\D_+|}\sum_{i\in \D_+}\nabla h(\w;\X_i) - \alpha' \frac{1}{|\D_+|}\sum_{i\in \D_+}\nabla h(\w';\X_i)}_2\\
& + \Norm{\frac{2}{|\D_+|}\sum_{i\in\D_+} (h(\w;\X_i) - a) - \frac{2}{|\D_+|}\sum_{i\in\D_+} (h(\w';\X_i) - a')}_2 \\
& + \Norm{\frac{2}{|\D_+|}\sum_{i\in\D_+} (h(\w;\X_i) - b) - \frac{2}{|\D_+|}\sum_{i\in\D_+} (h(\w';\X_i) - b')}_2 \\
& + |\alpha - \alpha'| + \Norm{\frac{1}{|\D_-|}\sum_{i\in \D_-}h(\w;\X_i) - \frac{1}{|\D_-|}\sum_{i\in \D_-}h(\w';\X_i)}_2 + \Norm{\frac{1}{|\D_+|}\sum_{i\in \D_+}h(\w;\X_i) - \frac{1}{|\D_+|}\sum_{i\in \D_+}h(\w';\X_i)}_2\\
&\leq 2(2L_h B_h + 2C_h^2 + (B_a + B_b) L_h + B_\Omega L_h + 3 C_h) \Norm{\w-\w'}_2\\
&\quad\quad+ 2(C_h + 1)|a-a'| + 2(C_h + 1)|b-b'| + 2(C_h + 1)|\alpha-\alpha'|. 
\end{align*}
\end{proof}

\begin{lemma}[Lemma 4.3 in \citet{lin2019gradient}]\label{lem:smooth_Phi}
For an $F$ defined in \eqref{eq:MI_DAM} that has Lipschitz continuous gradient and $\Phi(\w,a,b)\coloneqq \max_{\alpha\in\Omega}F(\w,a,b,\alpha)$ with a convex and bounded $\Omega$, we have that $\Phi(\w,a,b)$ is $L_\Phi$-smooth and $\nabla \Phi(\w,a,b) = \nabla_{(\w,a,b)} F(\w,a,b, \alpha^*(\w,a,b))$. Besides, $\alpha^*(\w,a,b)$ is $1$-Lipschitz continuous.    
\end{lemma}

We define $\v \coloneqq \begin{bmatrix}
 \v_1\\
 \v_2\\
\v_3   
\end{bmatrix}$, $W \coloneqq \begin{bmatrix}
 \w\\
 a\\
 b
\end{bmatrix}$, and $G_W^t =\begin{bmatrix}
G_{1,\w}^t + G_{2,\w}^t + G_{3,\w}^t\\
G_{1,a}^t\\
G_{2,b}^t
\end{bmatrix}$, $\bar{G}_W^t =\begin{bmatrix}
\bar{G}_{1,\w}^t + \bar{G}_{2,\w}^t + \bar{G}_{3,\w}^t\\
\bar{G}_{1,a}^t\\
\bar{G}_{2,b}^t
\end{bmatrix}$, where 
\begin{align*}
&\bar{G}^t_{1,\w} = \hat\E_{i\in\S_+^t}\nabla f_1(\w^t; \B_i^t) \nabla f_2(f_1(\w^t; \X_i))\nabla_1 f( f_2(f_1(\w^t; \X_i)), a^t), \\
&\bar{G}^t_{2,\w} = \hat\E_{i\in\S_-^t}\nabla f_1(\w^t; \B_i^t) \nabla f_2(f_1(\w^t; \X_i))\nabla_1 f( f_2(f_1(\w^t; \X_i)), b^t),\\
&\bar{G}^t_{3,\w} = \alpha^t \cdot\left(\hat\E_{i\in\S_-^t}\nabla f_1(\w^t; \B_i^t) \nabla f_2(f_1(\w^t; \X_i))- \hat\E_{i\in\S_+^t}\nabla f_1(\w^t; \B_i^t) \nabla f_2(f_1(\w^t; \X_i))\right),\\
&\bar{G}^t_{1,a}  = \hat\E_{i\in\S_+^t} \nabla_2 f( f_2(f_1(\w^t; \X_i))  , a^t),\\
 &\bar{G}^t_{2, b} =\hat\E_{i\in\S_-^t} \nabla_2 f( f_2(f_1(\w^t; \X_i))  , b^t).
\end{align*}


\begin{lemma}[Lemma 11 in \citet{wang2021memory}]\label{lem:batch}
Suppose that $X=\frac{1}{n} \sum_{i=1}^n X_i$. If we sample a size-$B$ minibatch $\B$ from $\{1,\ldots,n\}$ uniformly at random, we have $\mathbb{E}\left[\frac{1}{B} \sum_{i \in \mathcal{B}} (X_i - X)\right]=0$ 
and 
\begin{align*}
\mathbb{E}\left[\left\|\frac{1}{B} \sum_{i \in \mathcal{B}} (X_i-X)\right\|^2\right] \leq \frac{n-B}{B(n-1)} \frac{1}{n} \sum_{i=1}^n\left\|X_i-X\right\|^2 \leq \frac{n-B}{B(n-1)} \frac{1}{n} \sum_{i=1}^n\left\|X_i\right\|^2.
\end{align*}
\end{lemma}

\begin{lemma}\label{lem:grad_recursion}
Under Assumption~\ref{asm:major}, there exists $C_G, C_\Upsilon>0$ for MIDAM  such that
\begin{align*}
\sum_{t=0}^{T-1} \E\left[\Delta^t\right] & \leq \frac{\Delta^0}{\beta_1}  + 2 T \beta_1 C_G + 5 L_F^2 \sum_{t=0}^{T-1} \E\left[\Psi^{t+1}\right] + \frac{3\eta^2 L_\Phi^2}{\beta_1^2} \sum_{t=0}^{T-1}\E\left[\Norm{\v^t}_2^2\right] + 5 C_\Upsilon \sum_{t=0}^{T-1} \E\left[\Upsilon_+^{t+1}\right] + 5 C_\Upsilon \sum_{t=0}^{T-1}\E\left[\Upsilon_-^{t+1}\right]\\
& \quad\quad + 5C_\Upsilon \sum_{t=0}^{T-1} \frac{1}{D_+}\E\left[\sum_{i\in \S_+^t} \Norm{s_i^{t+1} - s_i^t}_2^2\right] +  5C_\Upsilon \sum_{t=0}^{T-1}\frac{1}{D_-}\E\left[\sum_{i\in \S_-^t} \Norm{s_i^{t+1} - s_i^t}_2^2\right].
\end{align*}
where $\Delta^t \coloneqq \Norm{\v^t - \nabla \Phi(W^t)}_2^2$, $\Upsilon_+^t \coloneqq \frac{1}{D_+}\sum_{i\in \D_+} \Norm{s_i^t - f_1(\w^t;\X_i)}_2^2$, $\Upsilon_-^t \coloneqq \frac{1}{D_-}\sum_{i\in \D_+} \Norm{s_i^t - f_1(\w^t;\X_i)}_2^2$, $\Psi^t \coloneqq  \Norm{\alpha^t - \alpha^*(W^t)}_2^2$. 
\end{lemma}
\begin{proof}
Based on the update rule of $\v^t$, we have
\begin{align*}
& \E_{t+1}\left[\Delta^{t+1}\right] = \E_t\left[\Norm{\v^{t+1} - \nabla \Phi(W^{t+1})}_2^2\right] = \E_t\left[\Norm{\v^{t+1} - \nabla_W F(W^{t+1},\alpha^*(W^{t+1}))}_2^2\right] \\
& = \E_{t+1}\left[\Norm{(1-\beta_1)\v^t + \beta_1 G_W^{t+1} - \nabla_W F(W^{t+1},\alpha^*(W^{t+1}))}_2^2\right]\\
& = \E_{t+1}\left[\left\|(1-\beta_1)\underbrace{(\v^t - \nabla \Phi(W^t))}_{\clubsuit} + (1-\beta_1)\underbrace{(\nabla \Phi(W^t) - \nabla \Phi(W^{t+1}))}_{\heartsuit} + \beta_1(G_W^{t+1} - \bar{G}_W^{t+1})\right.\right.\\
& \quad\quad\quad\quad \left.\left.+ \beta_1\underbrace{(\bar{G}_W^{t+1} - \nabla_W F(W^{t+1}, \alpha^{t+1}))}_{\diamondsuit} + \beta_1 \underbrace{(\nabla_W F(W^{t+1}, \alpha^{t+1}) - \nabla \Phi(W^{t+1}))}_{\spadesuit}\right\|_2^2\right].
\end{align*}
Note that $\E_{t+1}[\clubsuit \cdot \diamondsuit] =0$, $\E_{t+1}[\heartsuit\cdot \diamondsuit] =0$, $\E_{t+1}[\spadesuit\cdot \diamondsuit] =0$. Then,
\begin{align*}
& \E_{t+1}\left[\Delta^{t+1}\right] \\
& = (1-\beta_1)^2\Delta^t + (1-\beta_1)^2 \Norm{\nabla \Phi(W^t) - \nabla \Phi(W^{t+1})}_2^2 + \beta_1^2 \E_{t+1}[\Norm{G_W^{t+1} - \bar{G}_W^{t+1}}_2^2] \\
& \quad\quad + \beta_1^2 \E_{t+1}[\Norm{\bar{G}_W^{t+1} - \nabla_W F(W^{t+1}, \alpha^{t+1})}_2^2] + \beta_1^2 \Norm{\nabla_W F(W^{t+1}, \alpha^{t+1}) - \nabla \Phi(W^{t+1})}_2^2\\
& \quad\quad + 2 (1-\beta_1)^2 \inner{\v^t - \nabla \Phi(W^t)}{\nabla \Phi(W^t) - \nabla \Phi(W^{t+1})} + 2\beta_1(1-\beta_1)\E_{t+1}[\inner{\v^t - \nabla \Phi(W^t)}{G_W^{t+1} - \bar{G}_W^{t+1}}]\\
& \quad\quad + 2\beta_1(1-\beta_1)\E_{t+1}[\inner{\v^t - \nabla \Phi(W^t)}{\nabla_W F(W^{t+1}, \alpha^{t+1}) - \nabla \Phi(W^{t+1})}] \\
& \quad\quad + 2\beta_1(1-\beta_1)\E_{t+1}[\inner{\nabla \Phi(W^t) - \nabla \Phi(W^{t+1})}{G_W^{t+1} - \bar{G}_W^{t+1}}] \\
& \quad\quad + 2\beta_1(1-\beta_1) \E_{t+1}[\inner{\nabla \Phi(W^t) - \nabla \Phi(W^{t+1})}{\nabla_W F(W^{t+1}, \alpha^{t+1}) - \nabla \Phi(W^{t+1})}]\\
& \quad\quad + 2\beta_1^2 \E_{t+1}[\inner{G_W^{t+1} - \bar{G}_W^{t+1}}{\bar{G}_W^{t+1} - \nabla_W F(W^{t+1}, \alpha^{t+1})}] \\
& \quad\quad + 2\beta_1^2 \E_{t+1}[\inner{G_W^{t+1} - \bar{G}_W^{t+1}}{\nabla_W F(W^{t+1}, \alpha^{t+1}) - \nabla \Phi(W^{t+1})}]. 
\end{align*}
Use Young's inequality for products.
\begin{align*}
& \E_{t+1}\left[\Delta^{t+1}\right] \\
& \leq (1-\beta_1)^2(1+\beta_1) \Delta^t + \frac{3(1-\beta_1)^2(1+\beta_1)}{\beta_1}\Norm{\nabla \Phi(W^t) - \nabla \Phi(W^{t+1})}_2^2 + 2 \beta_1^2 \E_{t+1}[\Norm{\bar{G}_W^{t+1} - \nabla_W F(W^{t+1}, \alpha^{t+1})}_2^2]\\
& \quad\quad  + (3\beta_1 + 5\beta_1^2/3) \Norm{G_W^{t+1} - \bar{G}_W^{t+1}}_2^2 + (3\beta_1 + 5\beta_1^2/3) \E_{t+1}[\Norm{\nabla_W F(W^{t+1},\alpha^{t+1}) - \nabla \Phi(W^{t+1})}_2^2] \\
& \leq (1-\beta_1)\Delta^t + \frac{3L_\Phi^2 \eta^2}{\beta_1} \Norm{\v^t}_2^2 + 2\beta_1^2 \E_{t+1}[\Norm{\bar{G}_W^{t+1} - \nabla_W F(W^{t+1}, \alpha^{t+1})}_2^2] \\
& \quad\quad + 5\beta_1\underbrace{\E_{t+1}[\Norm{G_W^{t+1} - \bar{G}_W^{t+1}}_2^2]}_{	\triangle} + 5\beta_1 \Norm{\nabla_W F(W^{t+1},\alpha^{t+1}) - \nabla \Phi(W^{t+1})}_2^2.
\end{align*}
Note that $s_i^t$, $a^t$, $b^t$, $\alpha^t$ are bounded due to Lemma~\ref{lem:s_bounded}, Lemma~\ref{lem:a_b_bounded} and the projection step of updating $\alpha$. Besides, there exist $B_{f_1}, C_{f_1}, B_{f_2}, C_{f_2}, L_{f_2} > 0$ such that $\Norm{f_1}_2 \leq B_{f_1}, \Norm{\nabla f_1}_2\leq C_{f_1}$, $\Norm{f_2}_2 \leq B_{f_2},\Norm{\nabla f_2}_2\leq C_{f_2}, \Norm{\nabla^2 f_2}_2\leq L_{f_2}$ due to Assumption~\ref{asm:major}. Then, the definition of $\bar{G}_{1,\w}^{t+1}$, $\bar{G}_{2,\w}^{t+1}$, $\bar{G}_{3,\w}^{t+1}$, $\bar{G}_{1,a}^{t+1}$, $\bar{G}_{2,b}^{t+1}$ leads to
\begin{align*}
& \E_{t+1}[\Norm{\bar{G}_W^{t+1} - \nabla_W F(W^{t+1}, \alpha^{t+1})}_2^2] \\
& = \E_{t+1}[\Norm{\bar{G}_{1,\w}^{t+1}+ \bar{G}_{2,\w}^{t+1} + \bar{G}_{3,\w}^{t+1} -\nabla_\w F_1(\w^{t+1},a^{t+1})- \nabla_\w F_2(\w^{t+1},b^{t+1}) - \nabla_\w F_3(\w^{t+1},\alpha^{t+1})}_2^2]\\
& \quad\quad + \E_{t+1}[\Norm{\bar{G}_{1,a}^{t+1} - \nabla_a F_1(\w^{t+1},a^{t+1})}_2^2] + \E_{t+1}[\Norm{\bar{G}_{2,b}^{t+1} - \nabla_a F_2(\w^{t+1},b^{t+1})}_2^2]\\
& \leq 3 \E_{t+1}[\Norm{\bar{G}_{1,\w}^{t+1} - \nabla_\w F_1(\w^{t+1},a^{t+1})}_2^2] + 3 \E_{t+1}[\Norm{\bar{G}_{2,\w}^{t+1} - \nabla_\w F_2(\w^{t+1},b^{t+1})}_2^2]\\
& \quad\quad+ 3 \E_{t+1}[\Norm{\bar{G}_{3,\w}^{t+1} - \nabla_\w F_3(\w^{t+1},\alpha^{t+1})}_2^2]\\
& \quad\quad + \E_{t+1}[\Norm{\bar{G}_{1,a}^{t+1} - \nabla_a F_1(\w^{t+1},a^{t+1})}_2^2] + \E_{t+1}[\Norm{\bar{G}_{2,b}^{t+1} - \nabla_a F_2(\w^{t+1},b^{t+1})}_2^2]\\
& \leq 12 C_{f_1}^2C_{f_2}^2(2B_{f_2}^2 + B_a^2 + B_b^2) + 6 B_\Omega^2 C_{f_1}^2 C_{f_2}^2 + 4 (2B_{f_2}^2 + B_a^2 + B_b^2).
\end{align*}
We define $C_G \coloneqq 12 C_{f_1}^2C_{f_2}^2(2B_{f_2}^2 + B_a^2 + B_b^2) + 6 B_\Omega^2 C_{f_1}^2 C_{f_2}^2 + 4 (2B_{f_2}^2 + B_a^2 + B_b^2)$. Besides,
\begin{align*}
\Norm{\nabla_W F(W^{t+1},\alpha^{t+1}) - \nabla \Phi(W^{t+1})}_2^2 & = \Norm{\nabla_W F(W^{t+1},\alpha^{t+1}) - \nabla_W F(W^{t+1}, \alpha^*(W^{t+1}))}_2^2\\
& \leq L_F^2 \Norm{\alpha^{t+1} - \alpha^*(W^{t+1})}_2^2.
\end{align*}
Next, we turn to bound the $\triangle$ term. 
\begin{align*}
& \E_{t+1}[\Norm{G_W^{t+1} - \bar{G}_W^{t+1}}_2^2]\\
& \leq 3 \E_{t+1}\left[\Norm{\frac{2}{|\S_+^{t+1}|}\sum_{i\in \S_+^{t+1}}\nabla f_1(\w^{t+1};\B_i^{t+1})\left(\nabla f_2(s_i^t)f_2(s_i^t)  - \nabla f_2(f_1(\w^{t+1};\X_i)) f_2(f_1(\w^{t+1};\X_i))\right)}_2^2\right]\\
& \quad\quad + 3\E_{t+1}\left[\Norm{\frac{2}{|\S_-^{t+1}|}\sum_{i\in \S_-^{t+1}}\nabla f_1(\w^{t+1};\B_i^{t+1})\left(\nabla f_2(s_i^t)f_2(s_i^t)  - \nabla f_2(f_1(\w^{t+1};\X_i)) f_2(f_1(\w^{t+1};\X_i))\right)}_2^2\right]\\
& \quad\quad + 6 \E_{t+1}\left[\Norm{\alpha^{t+1}\left(\frac{1}{|\S_-^{t+1}|}\sum_{i\in \S_-^{t+1}} \nabla f_1(\w^t;\B_i^t)\left(\nabla f_2(s_i^t) - \nabla f_2(f_1(\w^{t+1};\X_i))\right)\right)}_2^2\right]\\
& \quad\quad + 6 \E_{t+1}\left[\Norm{\alpha^{t+1}\left(\frac{1}{|\S_+^{t+1}|}\sum_{i\in \S_+^{t+1}} \nabla f_1(\w^t;\B_i^t)\left(\nabla f_2(s_i^t) - \nabla f_2(f_1(\w^{t+1};\X_i))\right)\right)}_2^2\right]\\
& \quad\quad + \E_{t+1}\left[\Norm{\frac{2}{|\S_+^{t+1}|}\sum_{i\in \S_+^{t+1}}(f_2(s_i^t) - f_2(f_1(\w^{t+1};\X_i)))}_2^2\right] + \E_{t+1}\left[\Norm{\frac{2}{|\S_-^{t+1}|}\sum_{i\in \S_-^{t+1}}(f_2(s_i^t) - f_2(f_1(\w^{t+1};\X_i)))}_2^2\right].
\end{align*}
 Note that $s_i^t$ and $\w^{t+1}$ are independent of $\S_+^{t+1}$ and $\S_-^{t+1}$.
\begin{align*}
& \E_{t+1}[\Norm{G_W^{t+1} - \bar{G}_W^{t+1}}_2^2] \\
& \leq 12 C_{f_1}^2 \frac{1}{D_+}\sum_{i\in \D_+} \Norm{\nabla f_2(s_i^t) f_2(s_i^t) - \nabla f_2(f_1(\w^{t+1};\X_i)) f_2(f_1(\w^{t+1};\X_i))}_2^2 \\
& \quad\quad + 12 C_{f_1}^2 \frac{1}{D_-} \sum_{i\in \D_-} \Norm{\nabla f_2(s_i^t) f_2(s_i^t) - \nabla f_2(f_1(\w^{t+1};\X_i)) f_2(f_1(\w^{t+1};\X_i))}_2^2\\
& \quad\quad + 6 B_\Omega^2  C_{f_1}^2 L_{f_2}^2 \frac{1}{D_+}\sum_{i\in \D_+} \Norm{s_i^t - f_1(\w^{t+1};\X_i)}_2^2 + 6 B_\Omega^2  C_{f_1}^2 L_{f_2}^2 \frac{1}{D_-}\sum_{i\in \D_-} \Norm{s_i^t - f_1(\w^{t+1};\X_i)}_2^2\\
& \quad\quad + 4 C_{f_2}^2 \frac{1}{D_+}\sum_{i\in \D_+} \Norm{s_i^t - f_1(\w^{t+1};\X_i)}_2^2 + 4 C_{f_2}^2 \frac{1}{D_-}\sum_{i\in \D_-} \Norm{s_i^t - f_1(\w^{t+1};\X_i)}_2^2\\
& \leq C_\Upsilon\left(\frac{1}{D_+}\E_{t+1}\left[\sum_{i\in\D_+} \Norm{s_i^{t+1} - f_1(\w^{t+1};\X_i)}_2^2\right] + \frac{1}{D_-}\E_{t+1}\left[\sum_{i\in\D_-} \Norm{s_i^{t+1} - f_1(\w^{t+1};\X_i)}_2^2\right]\right)\\
&\quad\quad + C_\Upsilon\left( \frac{1}{D_+} \E_{t+1}\left[\sum_{i\in \D_+} \Norm{s_i^{t+1} - s_i^t}_2^2\right] + \frac{1}{D_-} \E_{t+1}\left[\sum_{i\in \D_-} \Norm{s_i^{t+1} - s_i^t}_2^2\right]\right),
\end{align*}
where we define $C_\Upsilon \coloneqq 48C_{f_1}^2(C_{f_2}^4 + B_{f_2}^2 L_{f_2}^2) + 12C_{f_1}^2 B_\Omega^2  L_{f_2}^2 + 8C_{f_2}^2$. Note that $s_i^{t+1} = s_i^t$ for those $i\not\in \S_+^t\cup \S_-^t$. Then,
\begin{align*}
\sum_{i\in \D_+} \Norm{s_i^{t+1} - s_i^t}_2^2 = \sum_{i\in \S_+^t} \Norm{s_i^{t+1} - s_i^t}_2^2,\quad \sum_{i\in \D_-} \Norm{s_i^{t+1} - s_i^t}_2^2 = \sum_{i\in \S_-^t} \Norm{s_i^{t+1} - s_i^t}_2^2
\end{align*}
We define 
\begin{align*}
    & \Upsilon_+^t \coloneqq \frac{1}{D_+}\sum_{i\in \D_+} \Norm{s_i^t - f_1(\w^t;\X_i)}_2^2, \quad \Upsilon_-^t \coloneqq \frac{1}{D_-}\sum_{i\in \D_+} \Norm{s_i^t - f_1(\w^t;\X_i)}_2^2,\quad  \Psi^t = \Norm{\alpha^t - \alpha^*(W^t)}_2^2,
 \end{align*}
 such that
 \begin{align*}
 \E_{t+1}[\Delta^{t+1}] &\leq (1-\beta_1)\Delta^t + \frac{3L_\Phi^2 \eta^2}{\beta_1} \Norm{\v^t}_2^2 + 2\beta_1^2 C_G + 5\beta_1 C_\Upsilon \left(\E_{t+1}[\Upsilon_+^{t+1}] + \E_{t+1}[\Upsilon_-^{t+1}]\right) + 5\beta_1 L_F^2 \Psi^{t+1} \\
 & \quad\quad + 5\beta_1 C_\Upsilon \left( \frac{1}{D_+} \E_{t+1}\left[\sum_{i\in \S_+^t} \Norm{s_i^{t+1} - s_i^t}_2^2\right] + \frac{1}{D_-} \E_{t+1}\left[\sum_{i\in \S_-^t} \Norm{s_i^{t+1} - s_i^t}_2^2\right]\right).
 \end{align*}
 Sum over $t=0,\dotsc, T-1$. 
 \begin{align*}
\sum_{t=0}^{T-1} \E\left[\Delta^t\right] & \leq \frac{\Delta^0}{\beta_1}  + 2 T \beta_1 C_G + 5 L_F^2 \sum_{t=0}^{T-1} \E\left[\Psi^{t+1}\right] + \frac{3\eta^2 L_\Phi^2}{\beta_1^2} \sum_{t=0}^{T-1}\E\left[\Norm{\v^t}_2^2\right] + 5 C_\Upsilon \sum_{t=0}^{T-1} \E\left[\Upsilon_+^{t+1}\right] + 5 C_\Upsilon \sum_{t=0}^{T-1}\E\left[\Upsilon_-^{t+1}\right]\\
& \quad\quad + 5C_\Upsilon \sum_{t=0}^{T-1} \frac{1}{D_+}\E\left[\sum_{i\in \S_+^t} \Norm{s_i^{t+1} - s_i^t}_2^2\right] +  5C_\Upsilon \sum_{t=0}^{T-1}\frac{1}{D_-}\E\left[\sum_{i\in \S_-^t} \Norm{s_i^{t+1} - s_i^t}_2^2\right].
 \end{align*}
\end{proof}

\begin{lemma}[Lemma 1 in \citet{wang2022finite}]\label{lem:func_val_recursion}
Suppose that $|\S_+^t|\equiv S_+$, $|\S_-^t|\equiv S_-$ and we define $D_+ = |\D_+|$, $D_- = |\D_-|$. Under Assumption~\ref{asm:major}, MIDAM satisfies that
\begin{align*}
\sum_{t=0}^{T-1} \E\left[\Upsilon_+^t\right] & \leq \frac{4 D_+\Upsilon_+^0}{\gamma_0  S_+} + \frac{8 T \gamma_0 B_{f_1}^2(N-B)}{B(N-1)} + \frac{20D_+^2 \eta^2 C_{f_1}^2}{\gamma_0^2 S_+^2}\sum_{t=0}^{T-1} \E\left[\Norm{\v^t}_2^2\right] - \frac{1}{\gamma_0 S_+}\sum_{t=0}^{T-1} \E\left[ \sum_{i\in \S_+^t} \Norm{s_i^{t+1} - s_i^t}_2^2\right] ,\\
\sum_{t=0}^{T-1} \E\left[\Upsilon_-^t\right] & \leq \frac{4 D_-\Upsilon_-^0}{\gamma_0 S_-} +  \frac{8 T \gamma_0 B_{f_1}^2(N-B)}{B(N-1)} + \frac{20D_-^2 \eta^2 C_{f_1}^2}{\gamma_0^2 S_-^2}\sum_{t=0}^{T-1} \E\left[\Norm{\v^t}_2^2\right] - \frac{1}{\gamma_0 S_-}\sum_{t=0}^{T-1}\E\left[ \sum_{i\in \S_-^t} \Norm{s_i^{t+1} - s_i^t}_2^2\right].
\end{align*}
\end{lemma}

\begin{lemma}\label{lem:alpha_recursion}
Under Assumption~\ref{asm:major}, MIDAM satisfies that 
\begin{align*}
\sum_{t=0}^{T-1} \E\left[\Psi^t\right] &\leq \frac{4\Psi^0}{\eta'} + 64 \eta' T (B_{f_2}^2 + B_h^2) + 32 \sum_{t=0}^{T-1} \E\left[\Upsilon_+^t\right] + 32 \sum_{t=0}^{T-1} \E\left[\Upsilon_-^t\right] + \frac{20\eta^2}{(\eta')^2} \sum_{t=0}^{T-1} \E\left[\Norm{\v^t}_2^2\right]\\
&\quad\quad + \frac{32}{ D_+}\sum_{t=0}^{T-1} \E\left[\sum_{i\in\S_+^{t-1}}\Norm{s_i^t - s_i^{t-1}}_2^2\right] + \frac{32}{D_-} \E\left[\sum_{i\in\S_-^{t-1}}\Norm{s_i^t - s_i^{t-1}}_2^2\right],
\end{align*}  
 where $\Psi^t \coloneqq \Norm{\alpha^t - \alpha^*(\v^t)}_2^2$.
\end{lemma}
\begin{proof}
We define $G_\alpha^t \coloneqq G_{3,\alpha}^t - \alpha^t$ and $\bar{G}_\alpha^t\coloneqq c + \frac{1}{D_-}\sum_{i\in \D_-} f(s_i^{t-1}) - \frac{1}{D_+}\sum_{i\in \D_+} f(s_i^{t-1})- \alpha^t $. The update formula of $\alpha$ and the 1-strong convexity of $F(W,\cdot)$ implies that
\begin{align*}
& \E_t\left[\Norm{\alpha^{t+1} - \alpha^*(W^t)}_2^2\right]\\
& =\E_t\left[\Norm{\Pi_{\Omega}[\alpha^t + \eta' G_\alpha^t] - \Pi_{\Omega}[\alpha^*(W^t) + \eta' \nabla_\alpha 
 F(W^t,\alpha^t)]}_2^2\right]\\
& \leq \E_t\left[\Norm{\alpha^t + \eta' G_\alpha^t - \alpha^*(W^t) - \eta' \nabla_\alpha 
 F(W^t,\alpha^t)}_2^2\right]\\
 & = \Norm{\alpha^t + \eta'\nabla_\alpha F(W^t,\alpha^t)- \alpha^*(W^t) -\eta'\nabla_\alpha F(W^t,\alpha^*(W^t))}_2^2 + (\eta')^2 \E_t\left[\Norm{G_\alpha^t - \nabla_\alpha F(W^t,\alpha^t)}^2 \right]\\
 & \quad\quad + 2\eta' \E_t\left[\inner{\alpha^t + \eta'\nabla_\alpha F(W^t,\alpha^t)- \alpha^*(W^t) -\eta'\nabla_\alpha F(W^t,\alpha^*(W^t))}{G_\alpha^t - \nabla_\alpha F(W^t,\alpha^t)}\right]\\
 & = \Norm{\alpha^t + \eta'\nabla_\alpha F(W^t,\alpha^t)- \alpha^*(W^t) -\eta'\nabla_\alpha F(W^t,\alpha^*(W^t))}_2^2 + (\eta')^2 \E_t\left[\Norm{G_\alpha^t - \nabla_\alpha F(W^t,\alpha^t)}^2 \right]\\
 & \quad\quad + 2\eta' \E_t\left[\inner{\alpha^t + \eta' \nabla_\alpha F(W^t,\alpha^t)- \alpha^*(W^t) -\eta'\nabla_\alpha F(W^t,\alpha^*(\v^t))}{\bar{G}_\alpha^t - \nabla_\alpha F(W^t,\alpha^t)}\right]\\
 & \leq (1+\eta'/2)\Norm{\alpha^t + \eta_\alpha \nabla_\alpha F(W^t,\alpha^t)- \alpha^*(W^t) -\eta'\nabla_\alpha F(W^t,\alpha^*(W^t))}_2^2 + (\eta')^2 \E_t\left[\Norm{G_\alpha^t - \nabla_\alpha F(W^t,\alpha^t)}^2 \right] \\
 & \quad\quad + \eta' \Norm{\bar{G}_\alpha^t - \nabla_\alpha F(W^t,\alpha^t)}^2.  
\end{align*}  
We have
\begin{align*}
& \Norm{\bar{G}_\alpha^t - \nabla_\alpha F(W^t,\alpha^t)}_2^2 \\
& = \Norm{\frac{1}{D_-}\sum_{i\in \D_-}f_2(s_i^{t-1}) -\frac{1}{D_+}\sum_{i\in \D_+}f_2(s_{+,i}^t)  - \frac{1}{D_-}\sum_{i\in \D_-}f_2(f_1(\w^t;\X_i)) +\frac{1}{D_+}\sum_{i\in \D_+}f_2(f_1(\w^t;\X_i))}_2^2 \\
& \leq \frac{2}{D_+}\sum_{i\in\D_+}\Norm{s_i^{t-1} - f_1(\w^t;\X_i)}_2^2 + \frac{2}{D_-}\sum_{i\in\D_-} \Norm{s_i^{t-1} - f_1(\w^t;\X_i)}_2^2\\
& \leq \frac{4}{D_+}\sum_{i\in\D_+}\Norm{s_i^t - f_1(\w^t;\X_i)}_2^2 + \frac{4}{D_-}\sum_{i\in\D_-} \Norm{s_i^t - f_1(\w^t;\X_i)}_2^2 + \frac{4}{D_+}\sum_{i\in\D_+} \Norm{s_i^t - s_i^{t-1}}_2^2 + \frac{4}{D_-}\sum_{i\in\D_-} \Norm{s_i^t - s_i^{t-1}}_2^2\\
& = \frac{4}{D_+}\sum_{i\in\D_+}\Norm{s_i^t - f_1(\w^t;\X_i)}_2^2 + \frac{4}{D_-}\sum_{i\in\D_-} \Norm{s_i^t - f_1(\w^t;\X_i)}_2^2 \\
& \quad\quad + \frac{4}{D_+}\sum_{i\in\S_+^{t-1}} \Norm{s_i^t - s_i^{t-1}}_2^2 + \frac{4}{D_-}\sum_{i\in\S_-^{t-1}} \Norm{s_i^t - s_i^{t-1}}_2^2,
\end{align*} 
where the last step is due to $s_i^t = s_i^{t-1}$ for those $i\not\in \S_+^{t-1}\cup \S_-^{t-1}$. Besides, we have
\begin{align*}
& \Norm{G_\alpha^t - \nabla_\alpha F(W^t,\alpha^t)}_2^2 \\
& = \Norm{\frac{1}{S_-^t} \sum_{i\in \S_-^t}f_2(s_i^{t-1}) - \frac{1}{S_+}\sum_{i\in \S_+^t}f_2(s_i^{t-1}) - \frac{1}{D_-}\sum_{i\in\D_-}h(\w^t;\X_i) - \frac{1}{D_+}\sum_{i\in\D_+} h(\w;\X_i)}_2^2 \leq 8 (B_{f_2}^2 + B_h^2).
\end{align*}
Due to the 1-strong convexity of $F(W,\cdot)$, we have
\begin{align*}
& \E\left[\Norm{\alpha^t + \eta'\nabla_\alpha F(W^t,\alpha^t)- \alpha^*(W^t) -\eta'\nabla_\alpha F(W^t,\alpha^*(W^t))}_2^2\right] \leq (1-\eta') \E\left[\Norm{\alpha^t - \alpha^*(W^t)}_2^2\right].
\end{align*}
Note that $\alpha^*(\cdot)$ is 1-Lipschitz (Lemma~\ref{lem:smooth_Phi}) such that 
\begin{align*}
& \E\left[\Norm{\alpha^{t+1} - \alpha^*(W^{t+1})}_2^2\right]\\
& \leq (1+\eta'/4)\E\left[\Norm{\alpha^{t+1} - \alpha^*(W^t)}_2^2\right] + (1+4/\eta')\E\left[\Norm{\alpha^*(W^t) - \alpha^*(W^{t+1})}_2^2\right]\\
& \leq (1-\eta'/4) \E\left[\Norm{\alpha^t - \alpha^*(W^t)}_2^2\right] + 16(\eta')^2(B_{f_2}^2 + B_h^2) + 8\eta' (\E[\Upsilon_+^t] + \E[\Upsilon_-^t])  + \frac{5\eta^2}{\eta'}\E\left[\Norm{\v^t}_2^2\right]\\
& \quad\quad + \frac{8\eta'}{D_+} \E\left[\sum_{i\in\S_+^{t-1}}\Norm{s_i^t - s_i^{t-1}}_2^2\right] + \frac{8\eta'}{D_-} \E\left[\sum_{i\in\S_-^{t-1}}\Norm{s_i^t - s_i^{t-1}}_2^2\right].
\end{align*}
Define $\Psi^t \coloneqq \Norm{\alpha^t - \alpha^*(\v^t)}_2^2$. Then, we have
\begin{align*}
\sum_{t=0}^{T-1} \E\left[\Psi^t\right] &\leq \frac{4\Psi^0}{\eta'} + 64 \eta' T (B_{f_2}^2 + B_h^2) + 32 \sum_{t=0}^{T-1} \E\left[\Upsilon_+^t\right] + 32 \sum_{t=0}^{T-1} \E\left[\Upsilon_-^t\right] + \frac{20\eta^2}{(\eta')^2} \sum_{t=0}^{T-1} \E\left[\Norm{\v^t}_2^2\right]\\
&\quad\quad + \frac{32}{ D_+}\sum_{t=0}^{T-1} \E\left[\sum_{i\in\S_+^{t-1}}\Norm{s_i^t - s_i^{t-1}}_2^2\right] + \frac{32}{D_-} \E\left[\sum_{i\in\S_-^{t-1}}\Norm{s_i^t - s_i^{t-1}}_2^2\right].
\end{align*}
\end{proof}

\section{Proof of Theorem~\ref{thm:main}}\label{sec:thm_proof}

According to Lemma~\ref{lem:smooth_Phi}, we have
\begin{align*}
\Phi(W^{t+1}) - \Phi(W^t) 
& = \inner{\nabla \Phi(W^t)}{W^{t+1} - W^t} + \frac{L_{\Phi}}{2}\Norm{W^{t+1} - W^t}_2^2  = -\inner{\nabla \Phi(W^t)}{\eta \v^t} + \frac{L_{\Phi}}{2}\Norm{\eta \v^t}_2^2\\
& \leq \frac{\eta}{2}\Norm{\v^t - \nabla \Phi(W^t)}_2^2 - \frac{\eta}{2}\Norm{\nabla \Phi(W^t)}_2^2 - \frac{\eta(1 - \eta L_{\Phi})}{2}\Norm{\v^t}_2^2.
\end{align*}
If $\eta_v \leq \frac{1}{2L_\Phi}$, Lemma~\ref{lem:grad_recursion} implies that
\begin{align*}
& \sum_{t=0}^{T-1} \E\left[\Norm{\nabla \Phi(W^t)}_2^2\right] \leq \frac{2(\Phi(W^0) - \inf \Phi)}{\eta} + \sum_{t=0}^{T-1} \E\left[\Delta^t\right] - \frac{1}{2} \sum_{t=0}^{T-1} \E\left[\Norm{\v^t}_2^2\right]\\
&\leq \frac{2(\Phi(W^0) - \inf \Phi)}{\eta} + \frac{\Delta^0}{\beta_1}  + 2 T \beta_1 C_G + 5 L_F^2 \sum_{t=0}^{T-1} \E\left[\Psi^{t+1}\right] + 5 C_\Upsilon \sum_{t=0}^{T-1} \E\left[\Upsilon_+^{t+1}\right] + 5 C_\Upsilon \sum_{t=0}^{T-1}\E\left[\Upsilon_-^{t+1}\right]\\
& \quad\quad - \left(\frac{1}{2}-\frac{3\eta^2 L_\Phi^2}{\beta_1^2}\right) \sum_{t=0}^{T-1}\E\left[\Norm{\v^t}_2^2\right] + 5C_\Upsilon \sum_{t=0}^{T-1} \frac{1}{D_+}\E\left[\sum_{i\in \S_+^t} \Norm{s_i^{t+1} - s_i^t}_2^2\right] +  5C_\Upsilon \sum_{t=0}^{T-1}\frac{1}{D_-}\E\left[\sum_{i\in \S_-^t} \Norm{s_i^{t+1} - s_i^t}_2^2\right].
\end{align*}
Apply Lemma~\ref{lem:func_val_recursion} and Lemma~\ref{lem:alpha_recursion}.
\begin{align*}
& \sum_{t=0}^{T-1} \E\left[\Norm{\nabla \Phi(W^t)}_2^2\right] \\
& \leq \frac{2(\Phi(W^0) - \inf \Phi)}{\eta} + \frac{\Delta^0}{\beta_1} + \frac{20 L_F^2 \E\left[\Psi^1\right]}{\eta'} + \frac{20 (C_\Upsilon + 32 L_F^2)D_+\E\left[\Upsilon_+^1\right]}{\gamma_0  S_+} + \frac{20 (C_\Upsilon + 32 L_F^2)D_+\E\left[\Upsilon_-^1\right]}{\gamma_0  S_-} \\
& \quad\quad + 2 T \beta_1 C_G  + 320 L_F^2 \eta' T (B_{f_2}^2 + B_h^2) + \frac{80 T \gamma_0 (C_\Upsilon + 32 L_F^2) B_{f_1}^2(N-B)}{B(N-B)} \\
&\quad\quad + \frac{5(C_\Upsilon + 32 L_F^2)}{ D_+} \E\left[\sum_{i\in\S_+^t}\Norm{s_i^1 - s_i^0}_2^2\right] + \frac{5(C_\Upsilon + 32 L_F^2)\left(1-\frac{D_+}{\gamma_0 S_+}\right)}{ D_+}\sum_{t=1}^{T-1} \E\left[\sum_{i\in\S_+^t}\Norm{s_i^{t+1} - s_i^t}_2^2\right] \\
&\quad\quad + \frac{5(C_\Upsilon + 32 L_F^2)}{D_-} \E\left[\sum_{i\in\S_-^t}\Norm{s_i^1 - s_i^0}_2^2\right] + \frac{5(C_\Upsilon + 32 L_F^2)\left(1-\frac{D_-}{\gamma_0 S_-}\right)}{D_-} \sum_{t=1}^{T-1}\E\left[\sum_{i\in\S_-^t}\Norm{s_i^{t+1} - s_i^t}_2^2\right]\\
& \quad\quad -\left(\frac{1}{2}-\frac{3\eta^2 L_\Phi^2}{\beta_1^2} - \frac{100 L_F^2 \eta^2}{(\eta')^2} -  \frac{100\eta^2(C_\Upsilon + 32 L_F^2)C_{f_1}^2 D_+^2}{\gamma_0^2 S_+^2} -  \frac{100\eta^2(C_\Upsilon + 32 L_F^2)C_{f_1}^2 D_-^2}{\gamma_0^2 S_-^2}\right)  \sum_{t=0}^{T-1} \E\left[\Norm{\v^t}_2^2\right]\\
& \quad\quad + \left(\frac{100 L_F^2 \eta^2}{(\eta')^2} +  \frac{100\eta^2(C_\Upsilon + 32 L_F^2)C_{f_1}^2 D_+^2}{\gamma_0^2 S_+^2} +  \frac{100\eta^2(C_\Upsilon + 32 L_F^2)C_{f_1}^2 D_-^2}{\gamma_0^2 S_-^2}\right)\E\left[\Norm{\v^T}_2^2\right].
\end{align*} 
Due to the update formula of $\v^t$, we have $\Norm{\v^t}_2 \leq C_\v$ for all $t\geq 0$, $C_\v \coloneqq 2C_{f_1} C_{f_2} (2B_{f_2} + B_a + B_b) + 2 B_\Omega C_{f_1} C_{f_2} + 2(2B_{f_2} + B_a + B_b) $. We choose $s_i^0 = 0$ for all $i\in \D_+\cup \D_-$ and the step sizes as follows
\begin{align*}
& \beta_1 \leq \frac{\epsilon^2}{22 C_G}, \quad \eta' \leq \frac{\epsilon^2}{3520 L_F^2(B_{f_2}^2 + B_h^2)},\quad \gamma_0 \leq \frac{\epsilon^2B(N-1)}{880(C_\Upsilon + 32L_F^2)B_{f_1}^2(N-B)},\\
& \eta \leq \min\left\{\frac{\beta_1}{4\sqrt{3} L_\Phi}, \frac{\eta'}{40 L_F}, \frac{\gamma_0 S_+}{40 \sqrt{C_\Upsilon + 32L_F^2}C_{f_1}D_+},  \frac{\gamma_0 S_-}{40 \sqrt{C_\Upsilon + 32L_F^2}C_{f_1}D_-}\right\}.
\end{align*}
After $T= \max\left\{\frac{22(\Phi(W^0) - \inf\Phi)}{\eta \epsilon^2}, \frac{11\Delta^0}{\beta_1 \epsilon^2}, \frac{220 L_F^2 \E\left[\Upsilon^1\right]}{\eta'\epsilon^2}, \frac{220(C_\Upsilon + 32L_F^2) D_+ \E[\Upsilon_+^1]}{\gamma_0 S_+ \epsilon^2}, \frac{220(C_\Upsilon + 32L_F^2) D_- \E[\Upsilon_-^1]}{\gamma_0 S_- \epsilon^2}\right\}$ iterations, we have
\begin{align*}
\frac{1}{T}\sum_{t=0}^{T-1}\E\left[\Norm{\Phi(W^t)}_2^2\right] + \frac{1}{4T}   \sum_{t=0}^{T-1}\E\left[\Norm{\v^t}_2^2\right] \leq \epsilon^2.
\end{align*}
According to Lemma~\ref{lem:func_val_recursion}, we have $\frac{1}{T}\sum_{t=0}^{T-1}\E\left[\Upsilon_+^t\right] = \O(\epsilon^2)$, $\frac{1}{T}\sum_{t=0}^{T-1}\E\left[\Upsilon_-^t\right] = \O(\epsilon^2)$ with $\gamma_0 = \O(\epsilon^2)$, $T = \O(\max\{\frac{D_+}{S_+}, \frac{D_-}{S_-}\}\frac{\epsilon^{-4}}{B})$.

\section{More Figures}\label{sec:mf}
\begin{figure}
    \centering
        \scalebox{.4}{\includegraphics{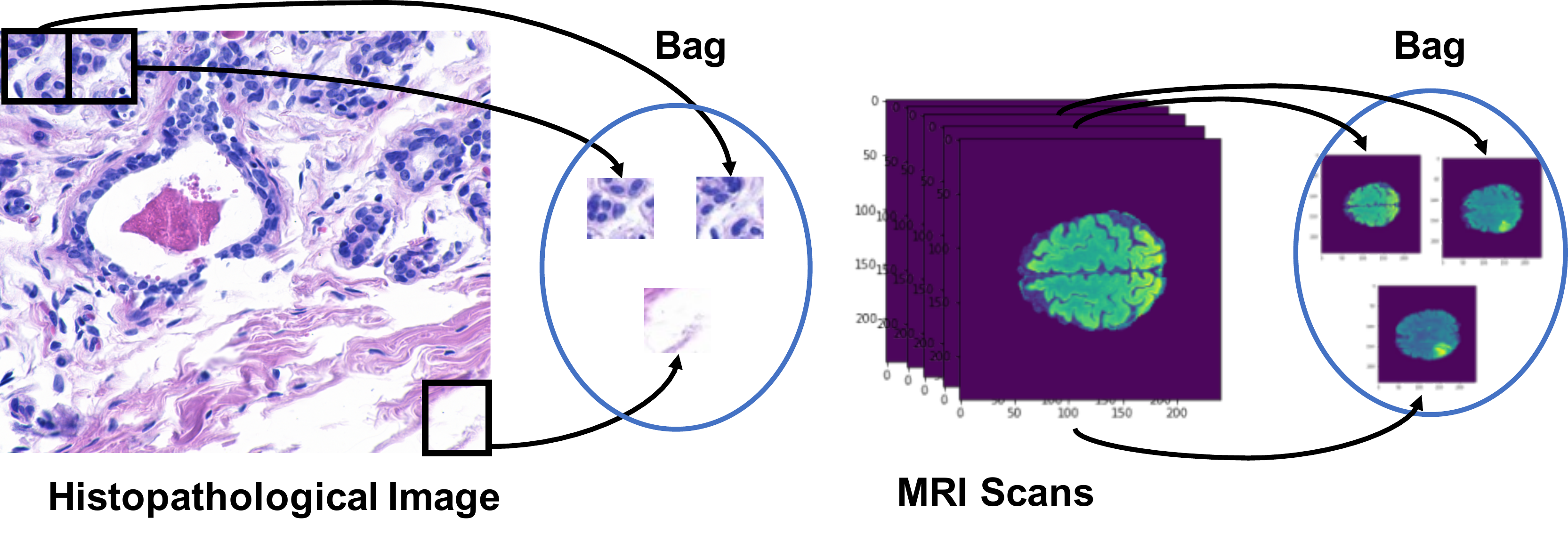}}

    \caption{Illustration of MIL for medical data (Breast Cancer on the left and PDGM on the right).}
    \label{fig:bc-examples}
\end{figure}


\begin{figure*}[t]
    \centering

    \subfigure[smx, training, MUSK1]{\includegraphics[scale=0.16]{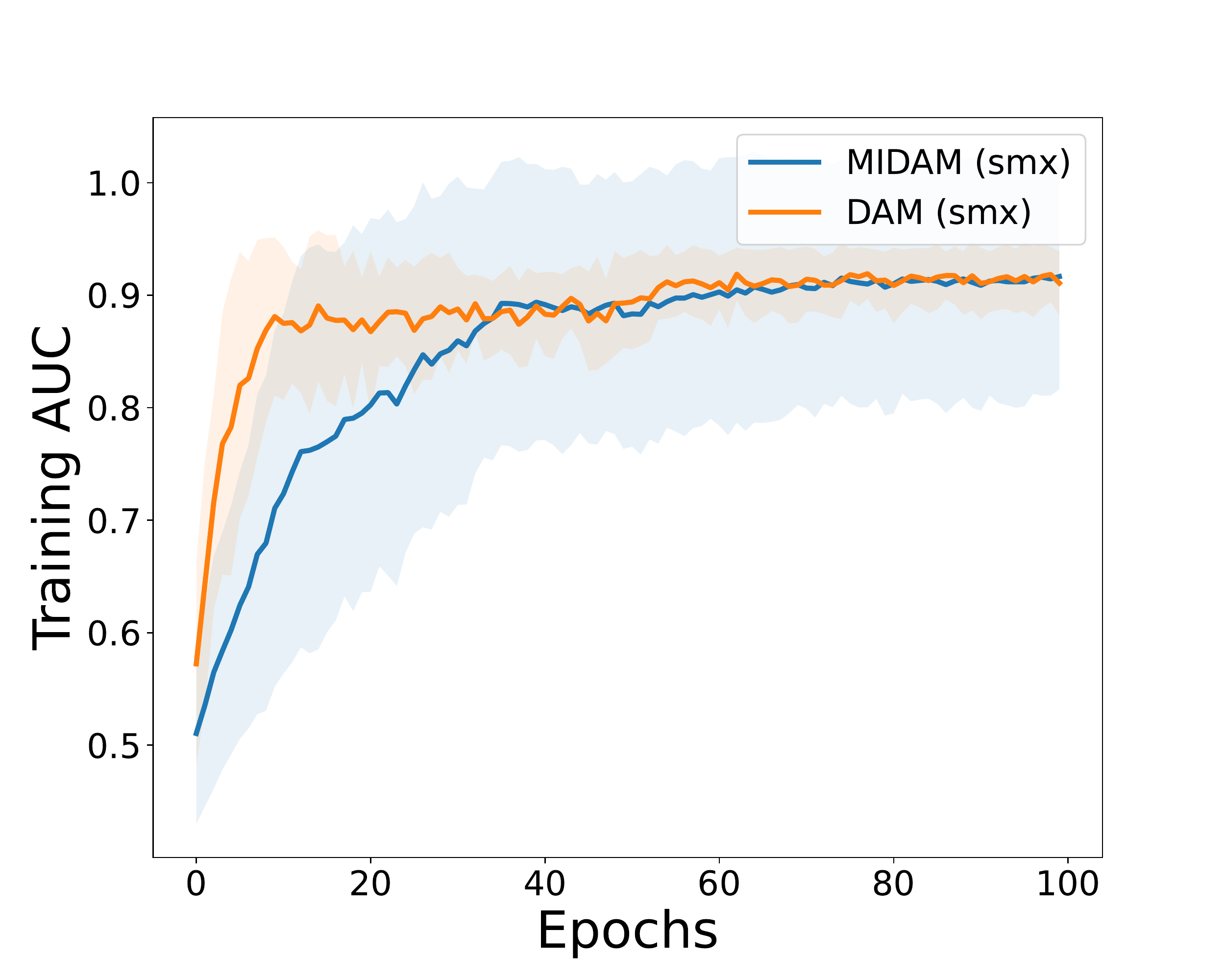}}
    \subfigure[att, training, MUSK1]{\includegraphics[scale=0.16]{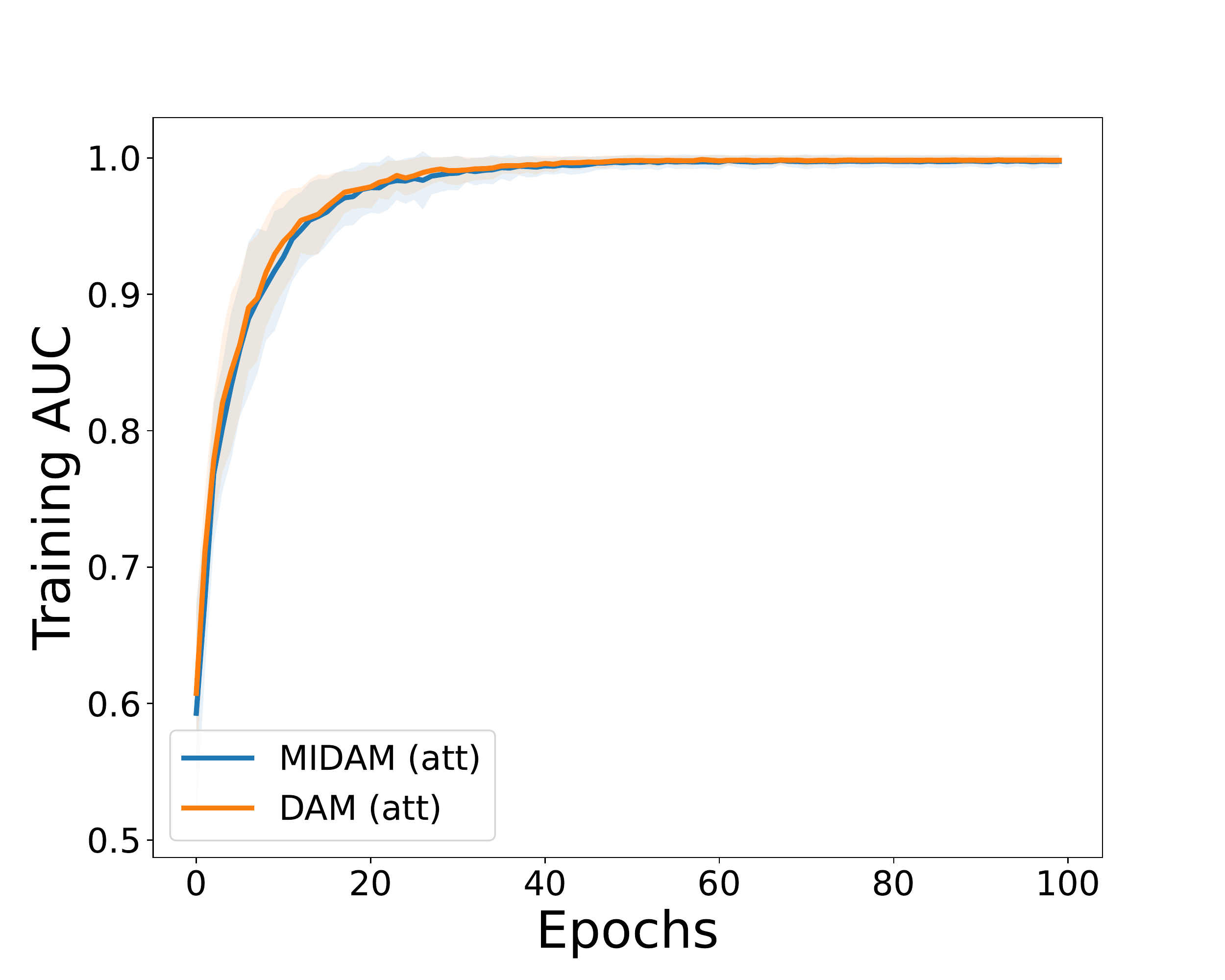}}
    \subfigure[smx, training, MUSK2]{\includegraphics[scale=0.16]{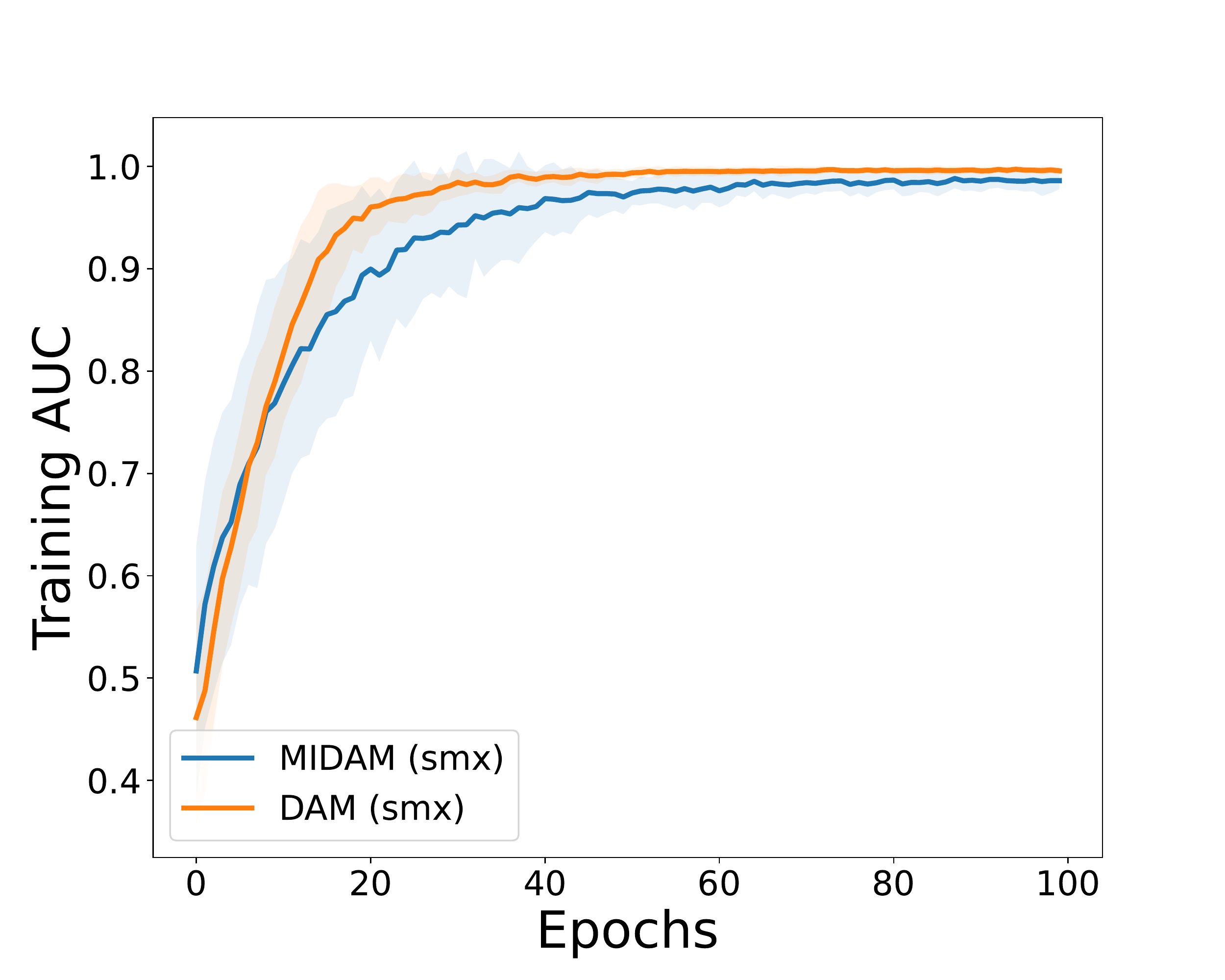}}
    \subfigure[att, training, MUSK2]{\includegraphics[scale=0.16]{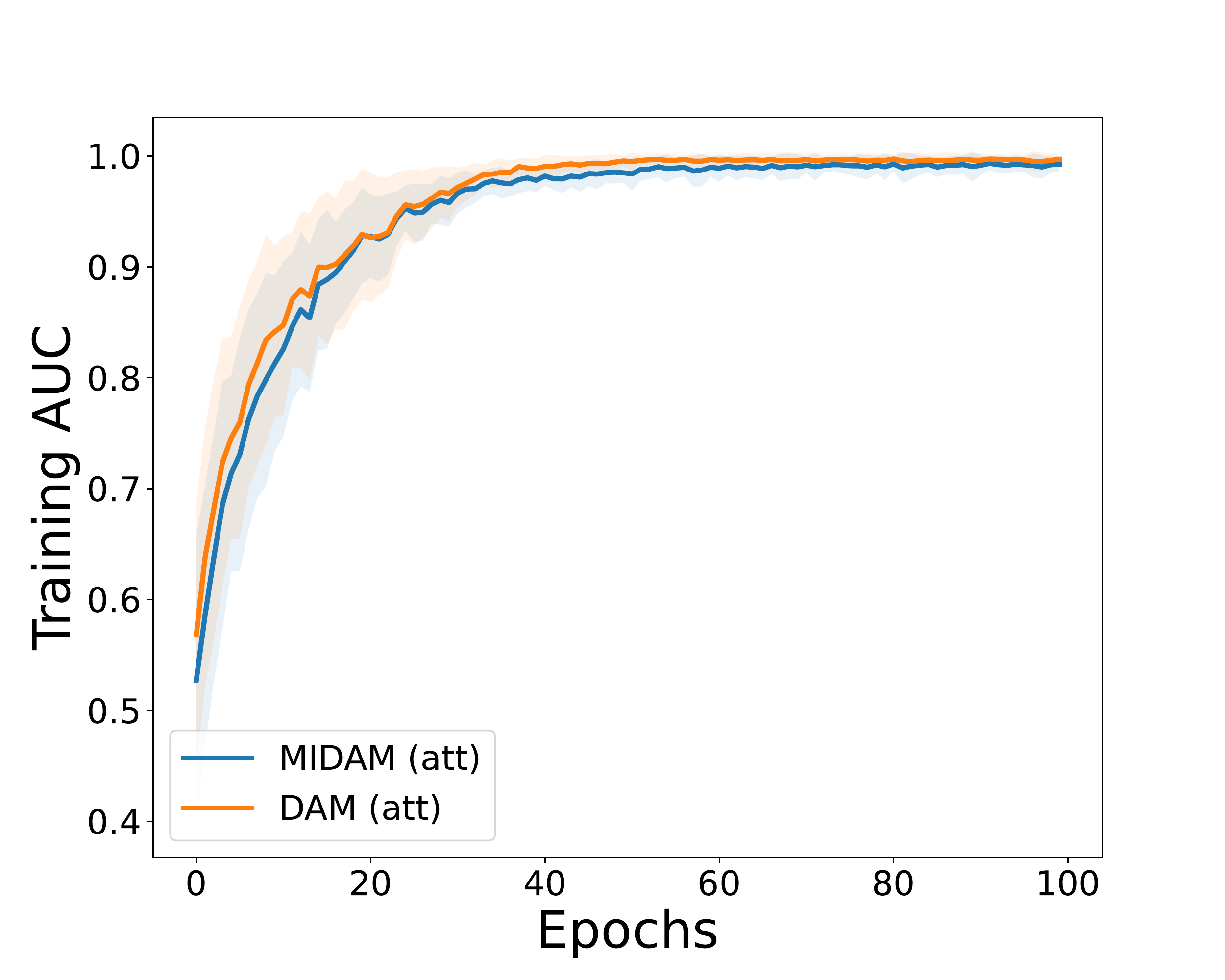}}

     \subfigure[smx, testing, MUSK1]{\includegraphics[scale=0.16]{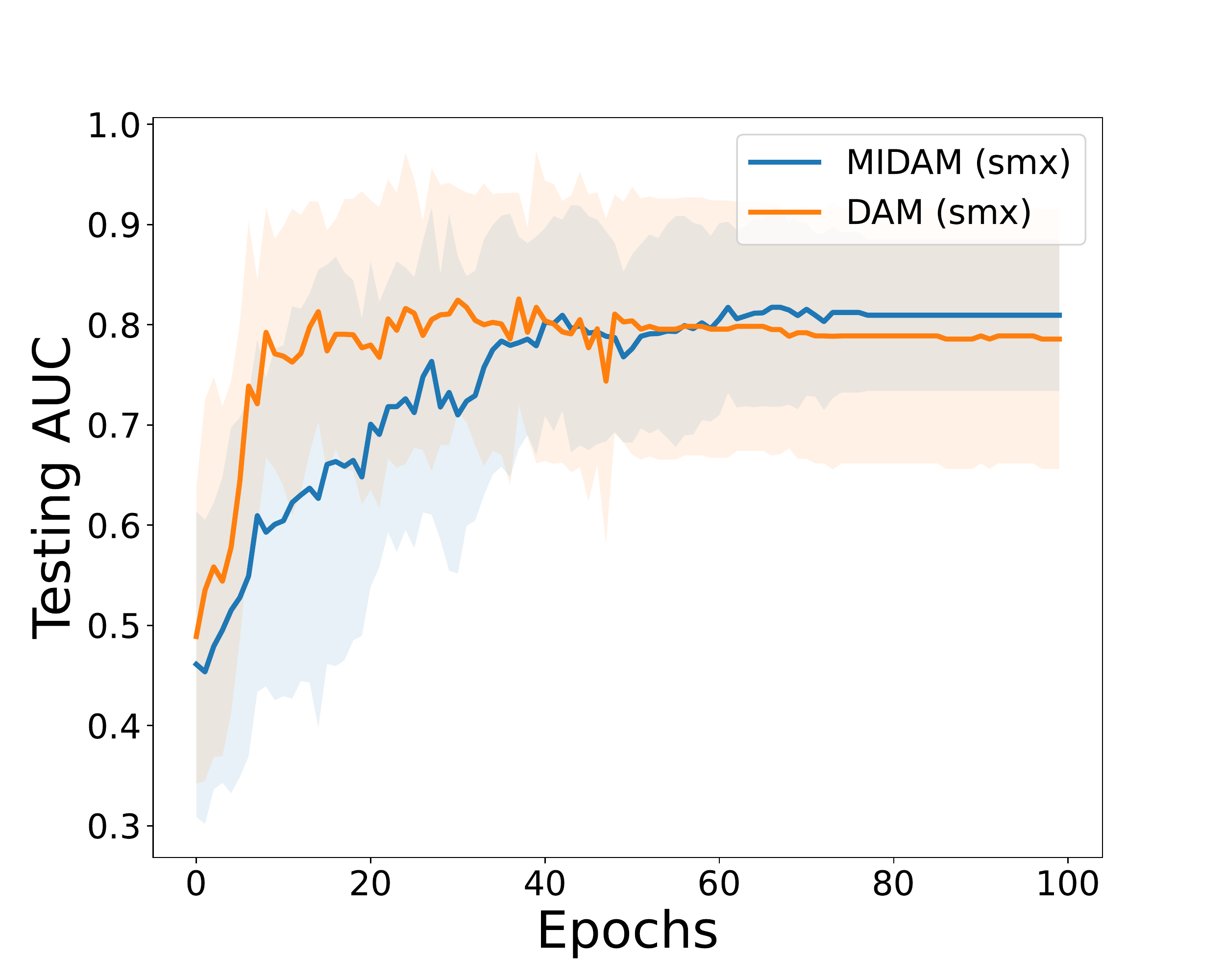}}
    \subfigure[att, testing, MUSK1]{\includegraphics[scale=0.16]{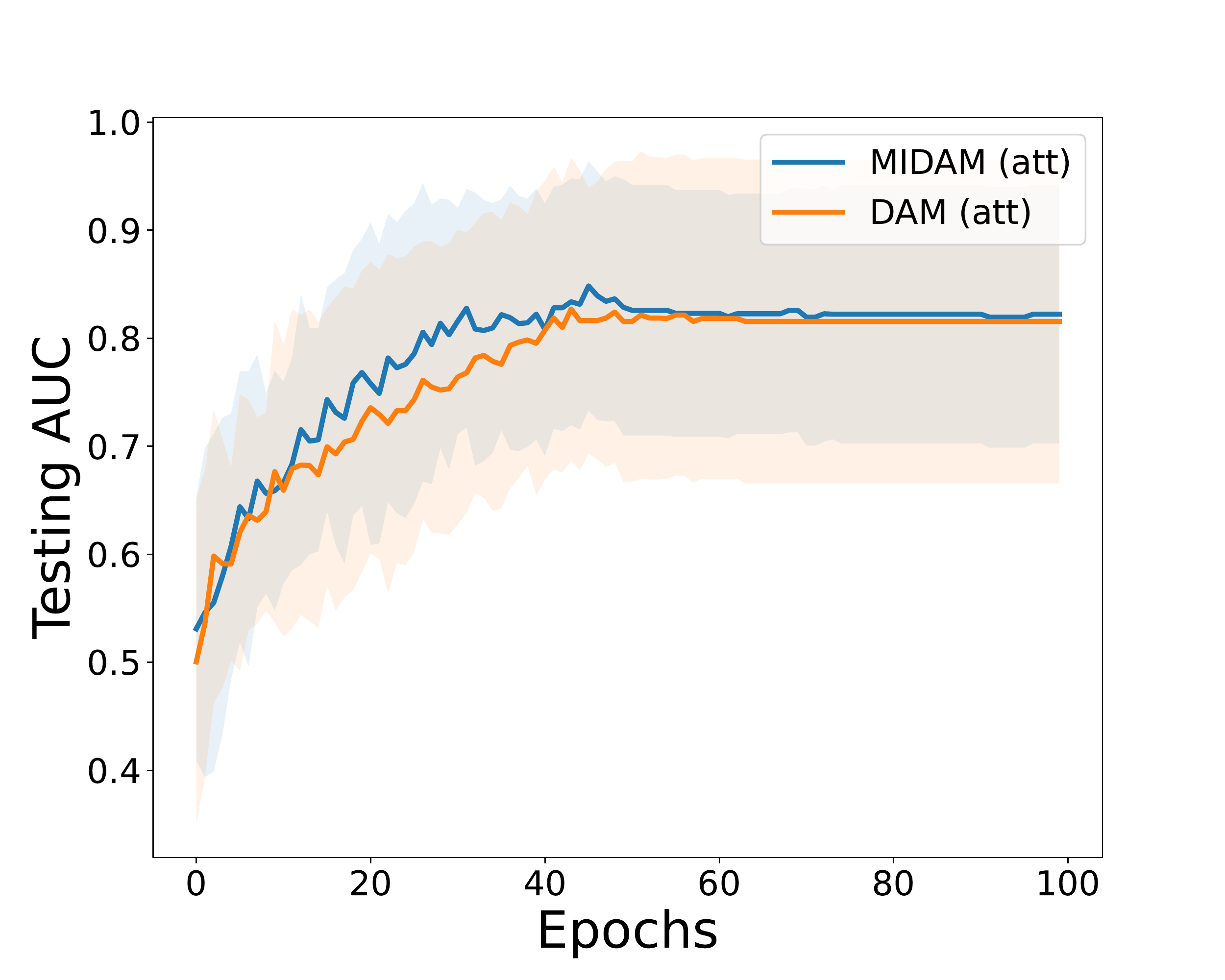}}
    \subfigure[smx, testing, MUSK2]{\includegraphics[scale=0.16]{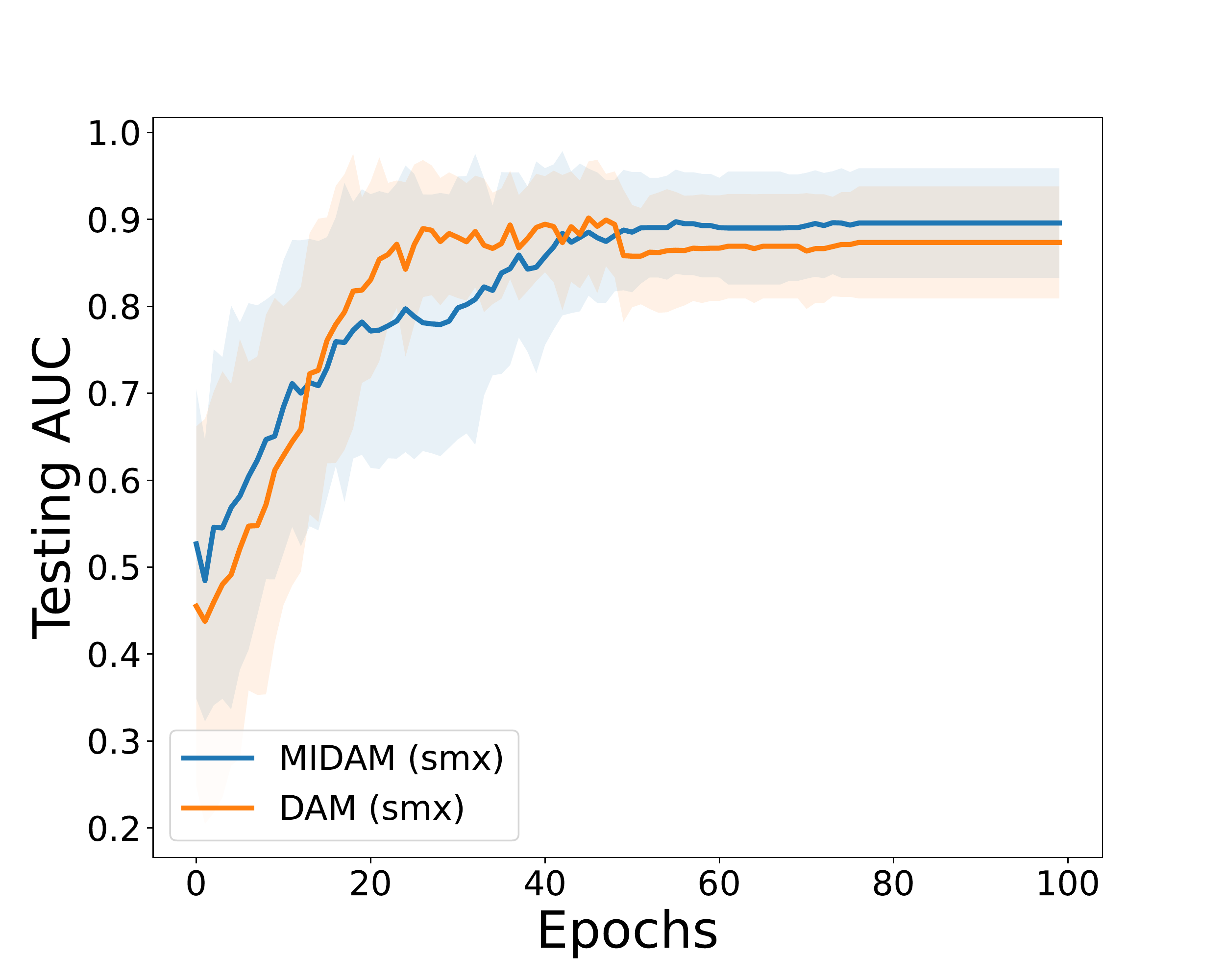}}
    \subfigure[att, testing, MUSK2]{\includegraphics[scale=0.16]{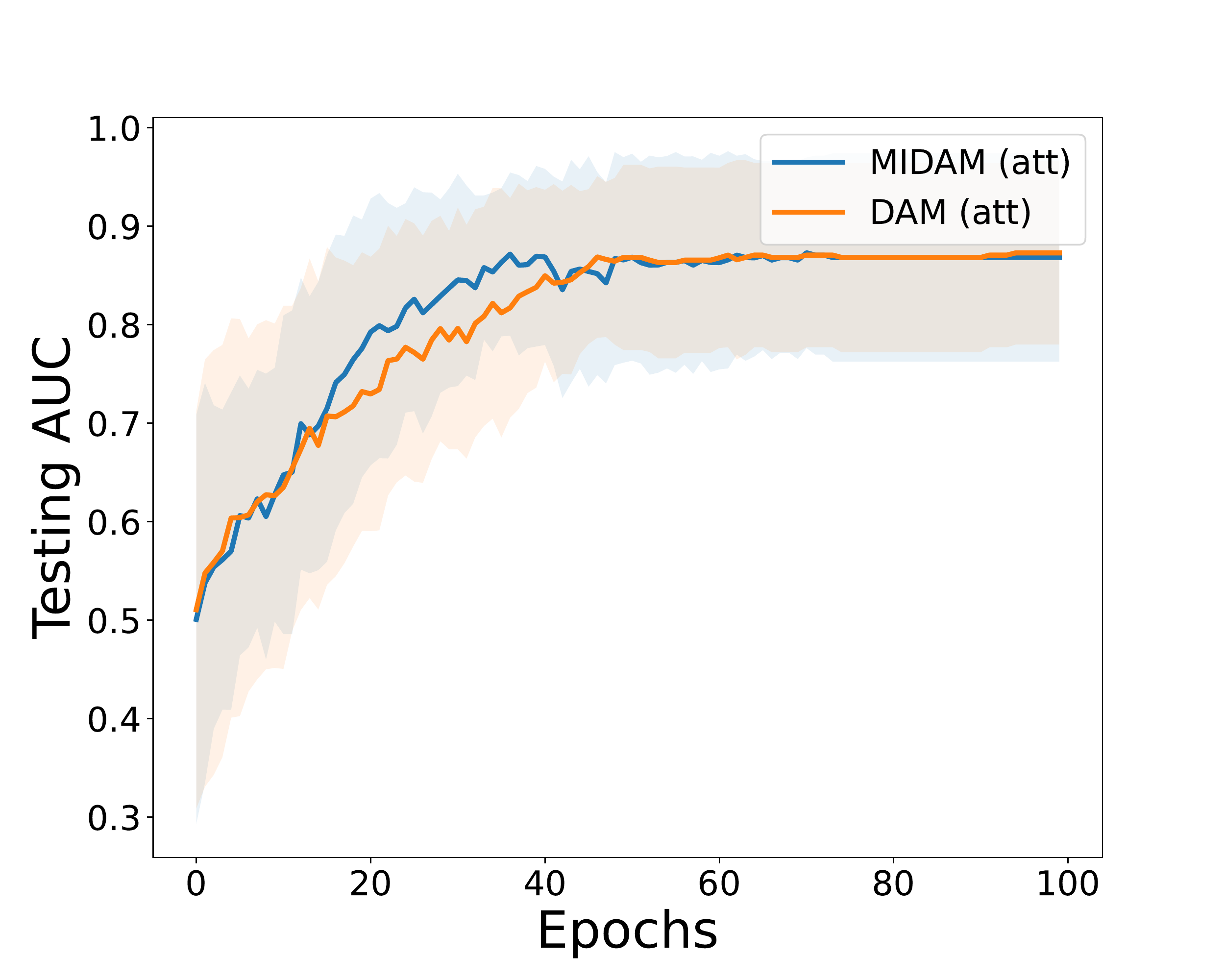}}

        \caption{Training and testing convergence of MIDAM (XX) vs DAM (XX). The top is for training AUC, and the bottom is for testing AUC.}\label{fig:generalization}
\end{figure*}

\begin{figure*}[t]
\vspace{-0.1 in}
    \centering
    \subfigure[smx, Breast Cancer]{\includegraphics[scale=0.16]{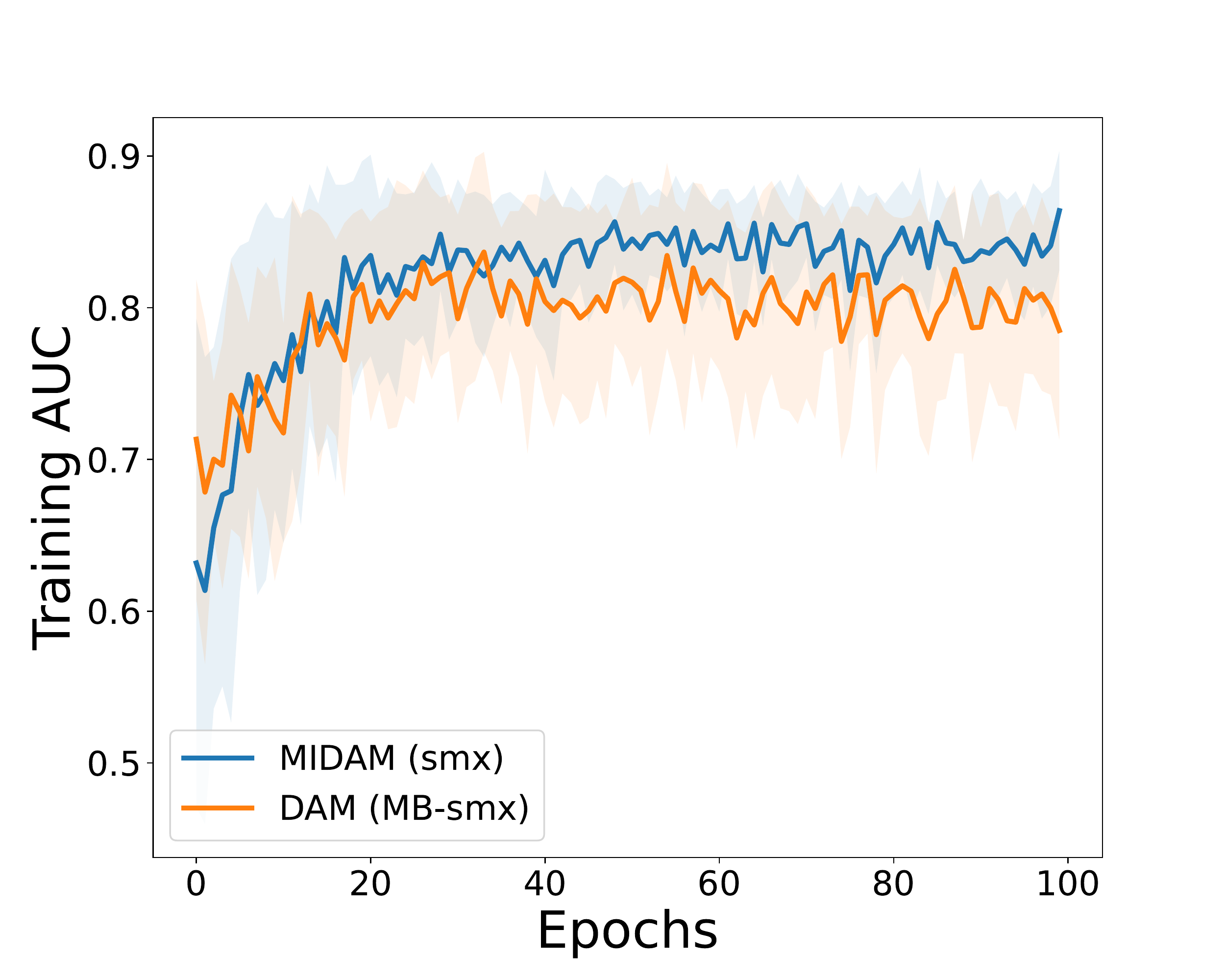}}
    \subfigure[att, Breast Cancer]{\includegraphics[scale=0.16]{Figures/BreastCancer-AUCMatt-s-Tr-AUC-Margin=01-cmp.pdf}}
    \subfigure[smx, Colon Ade.]{\includegraphics[scale=0.16]{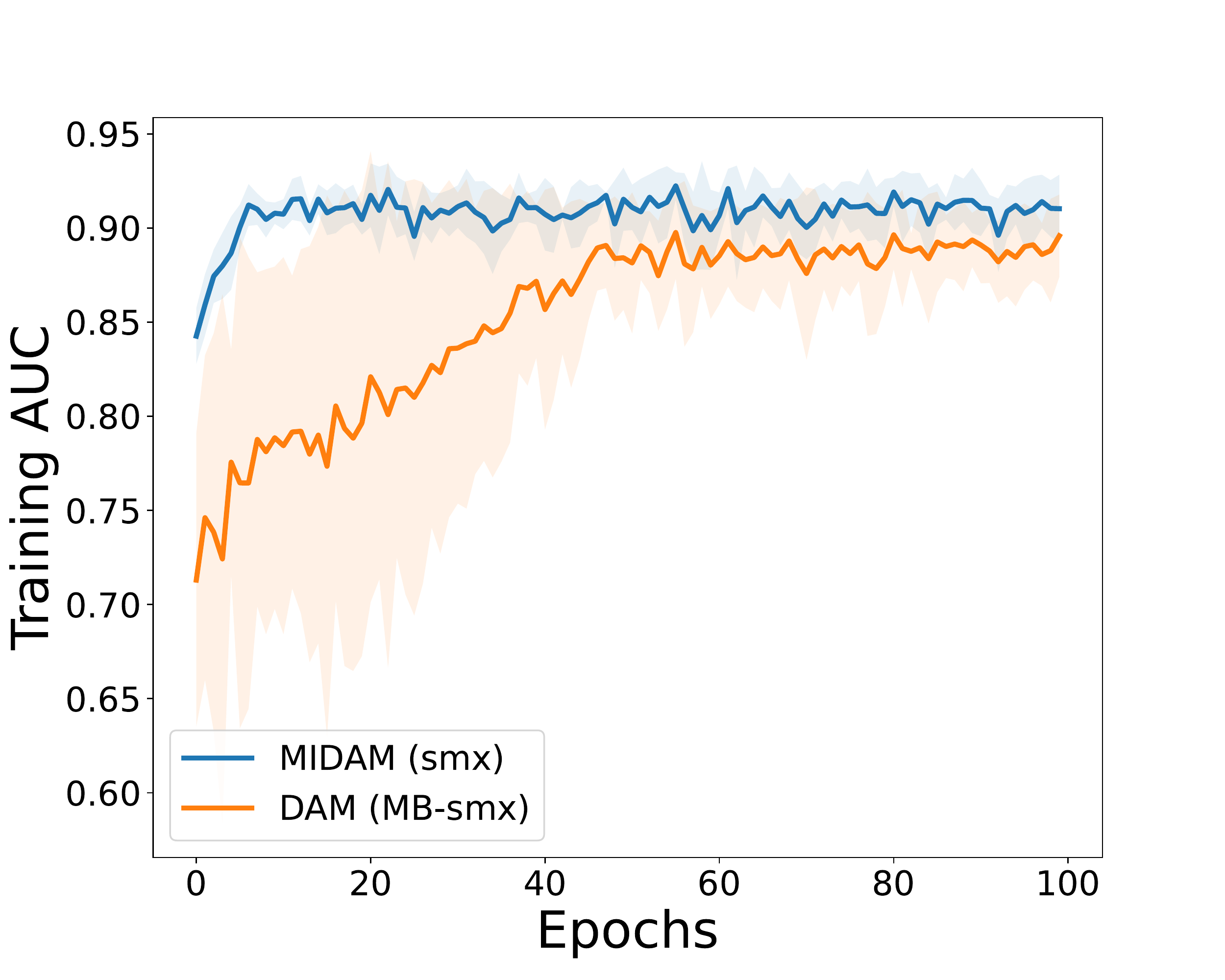}}
    \subfigure[att, Colon Ade.]{\includegraphics[scale=0.16]{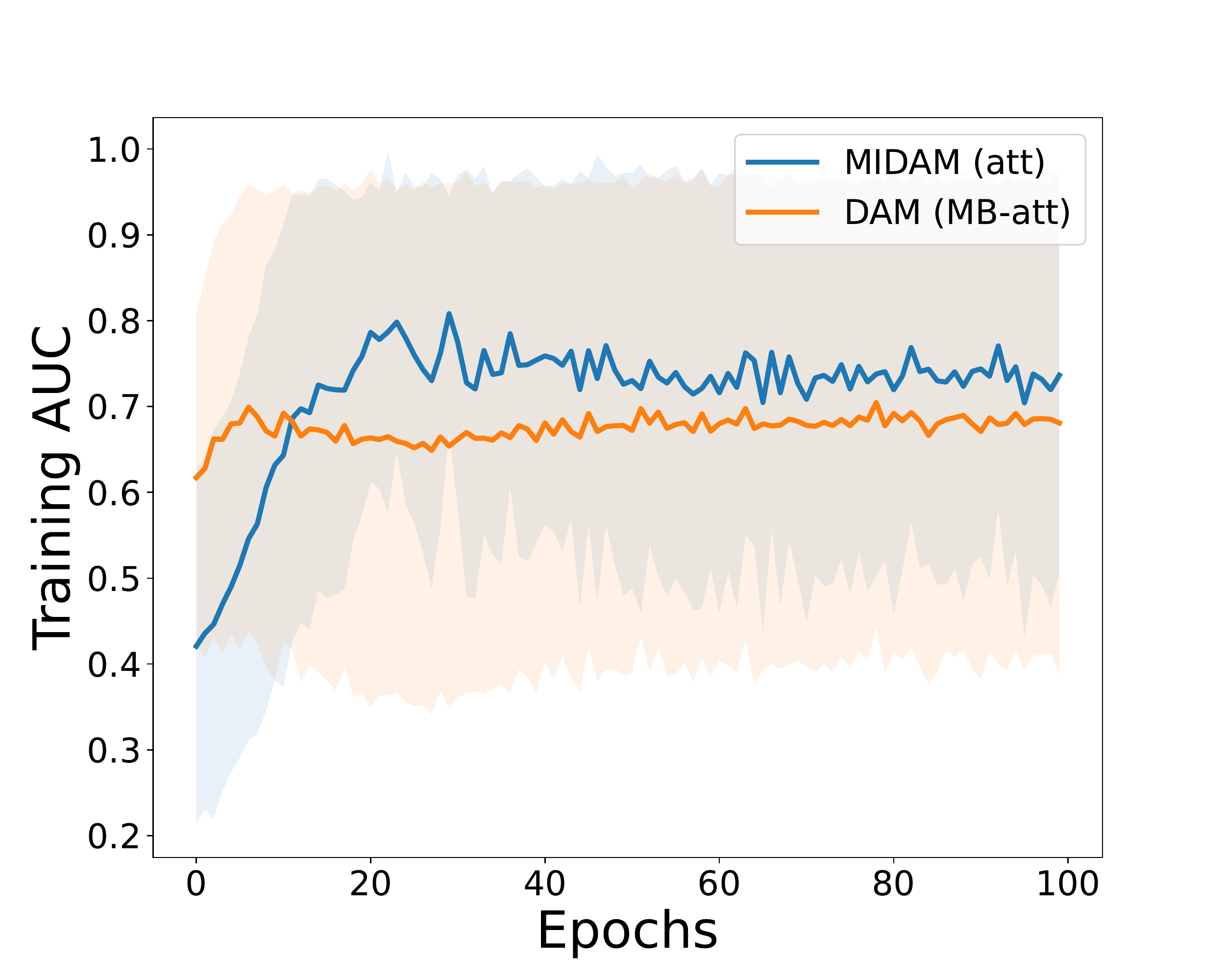}}

 \subfigure[smx, testing, Breast Cancer]{\includegraphics[scale=0.16]{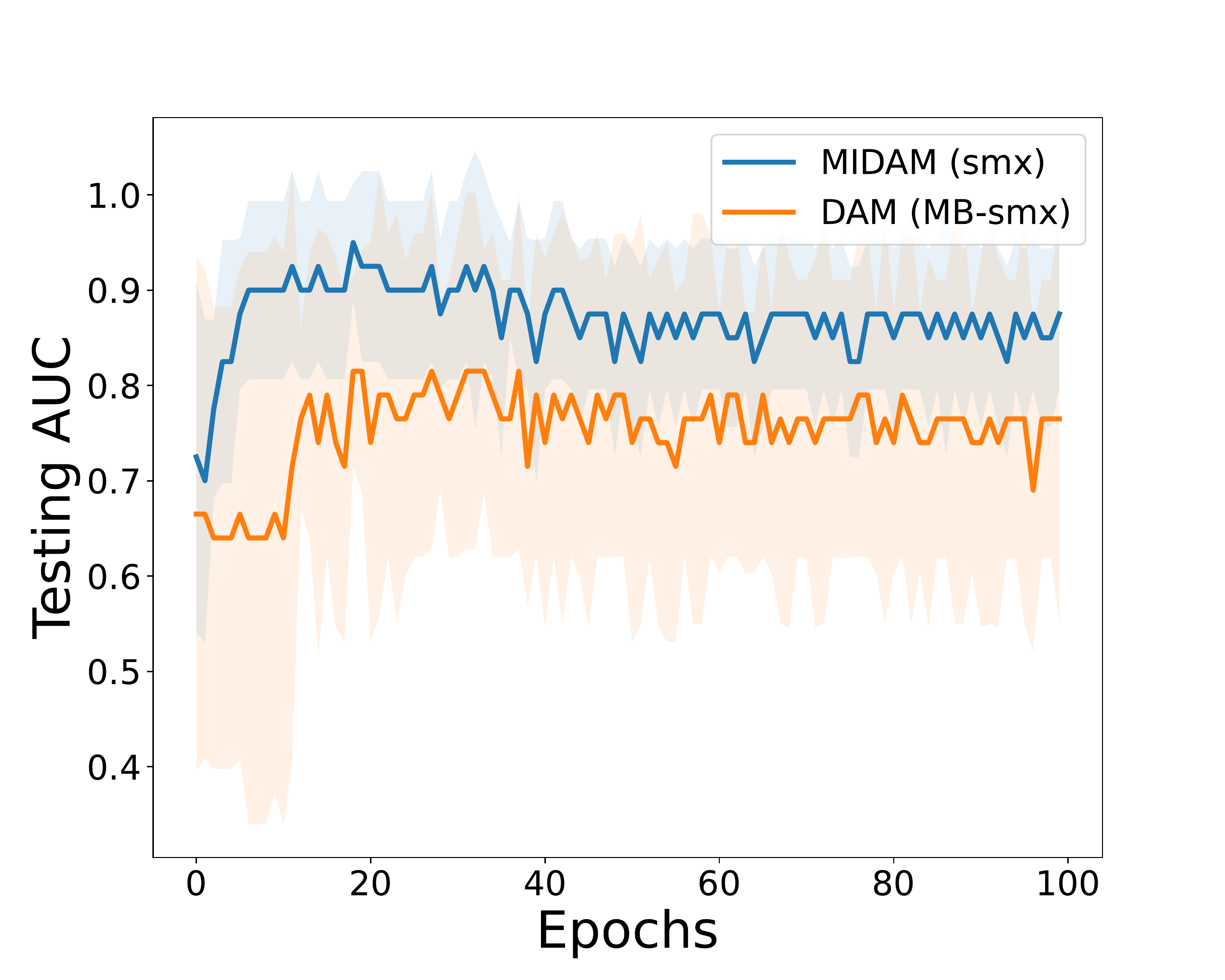}}
    \subfigure[att, testing, Breast Cancer]{\includegraphics[scale=0.16]{Figures/BreastCancer-AUCMatt-s-Test-AUC-Margin=01-cmp.pdf}}
    \subfigure[smx, testing, Colon Ade.]{\includegraphics[scale=0.16]{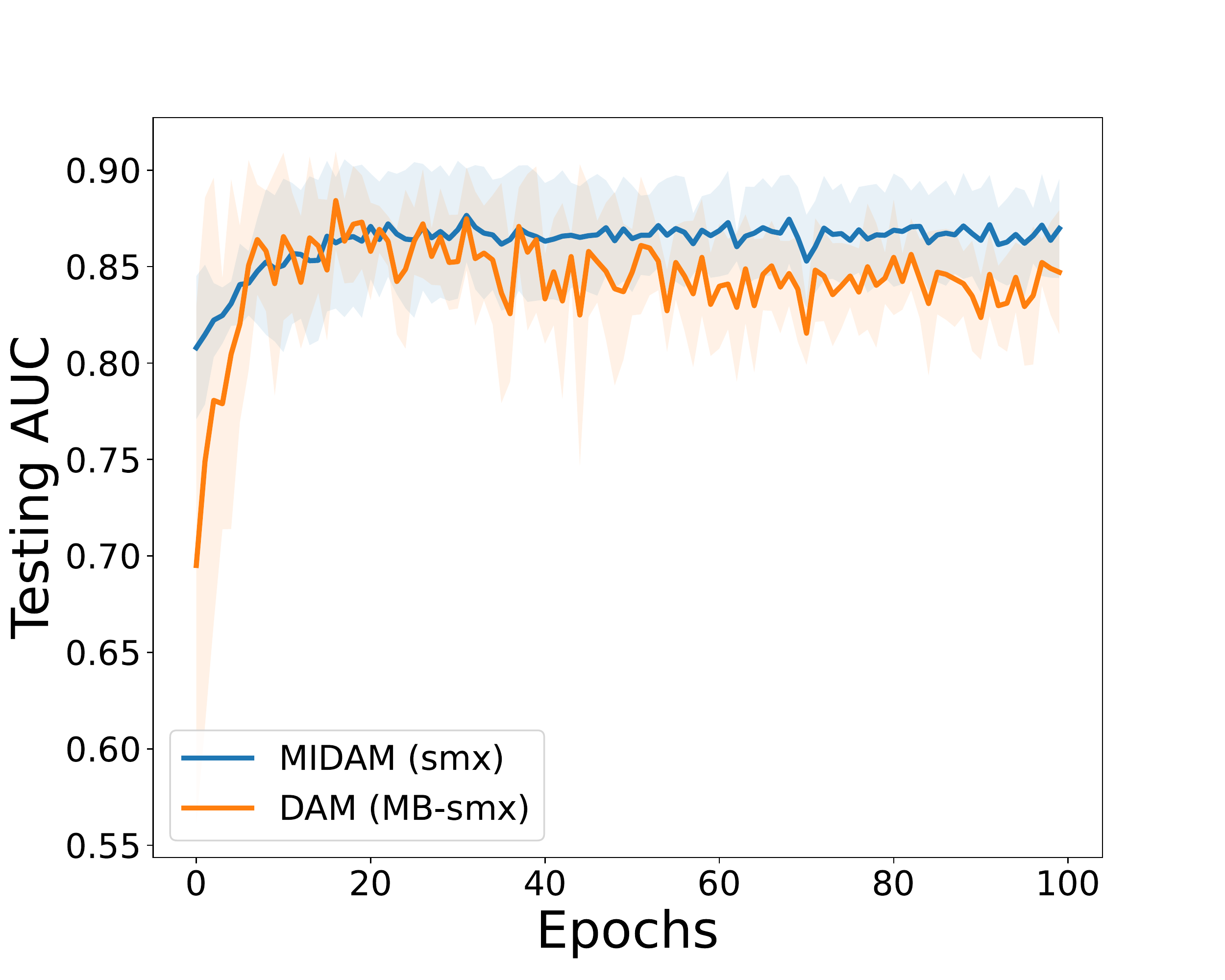}}
    \subfigure[att, testing, Colon Ade.]{\includegraphics[scale=0.16]{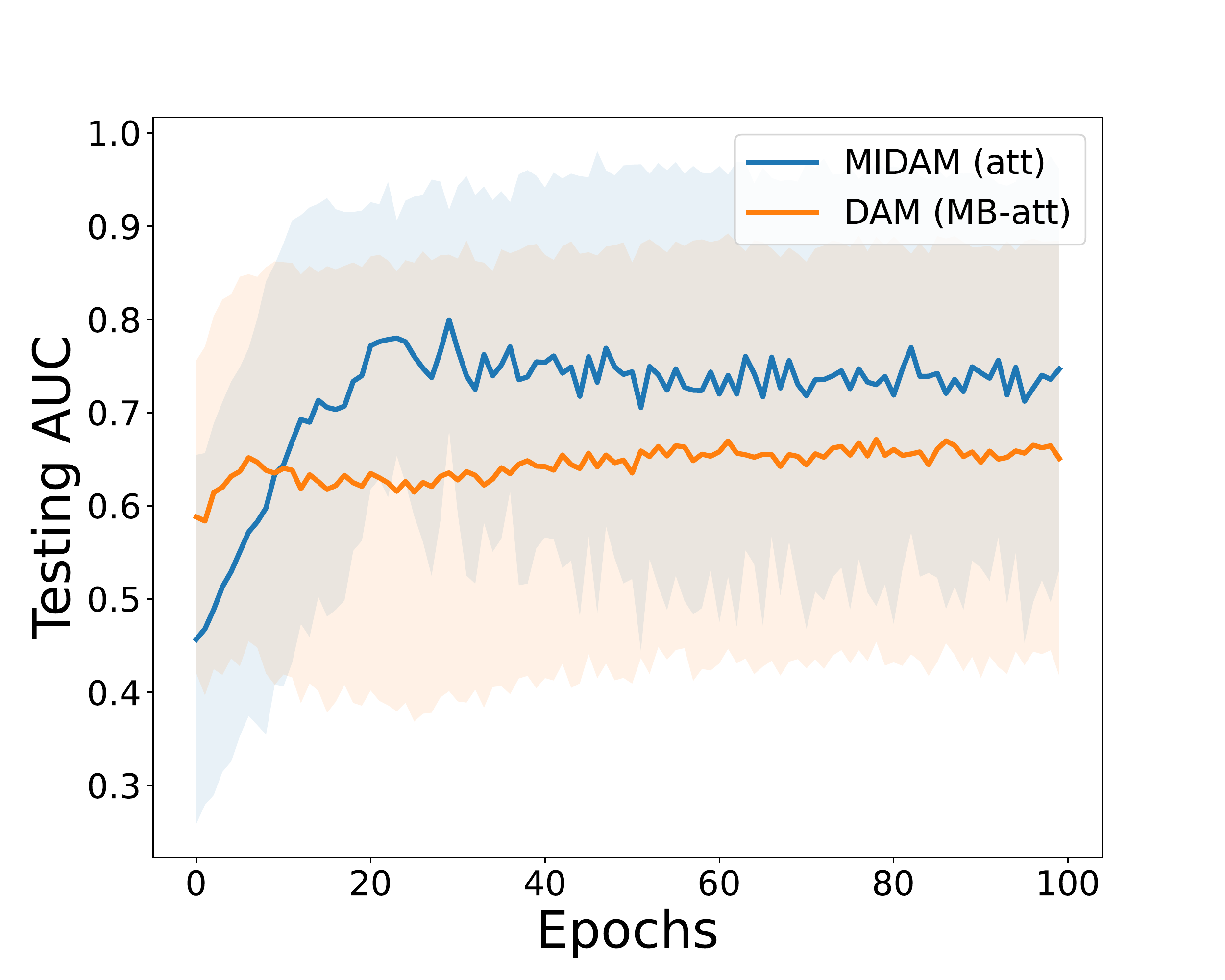}}
        \caption{Training and testing convergence of MIDAM (XX) vs DAM (MB-XX). The margin $c=0.1$ and learning rate is tuned in \{1e-1,1e-2,1e-3\}. The top is for training AUC, and the bottom is for testing AUC. }\label{fig:faster-cvg}
\end{figure*}

\begin{figure*}[t]
\vspace{-0.1 in}
    \centering
    \subfigure[MIDAM-smx, MUSK2]{\includegraphics[scale=0.16]{Figures/MUSK2-AUCMmax-s-Tr-AUC.pdf}}
    \subfigure[MIDAM-smx, Fox]{\includegraphics[scale=0.16]{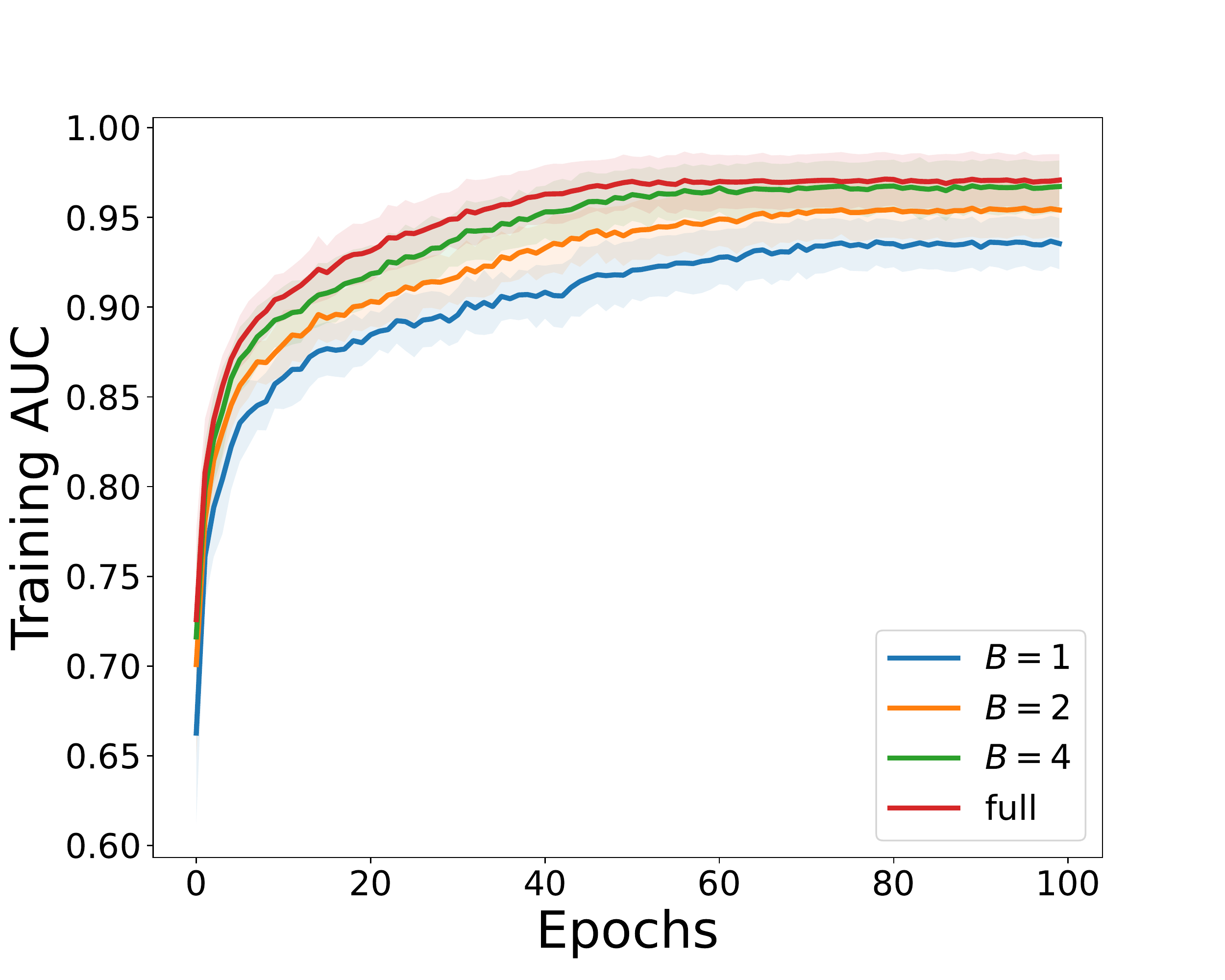}}
    \subfigure[MIDAM-smx, Tiger]{\includegraphics[scale=0.16]{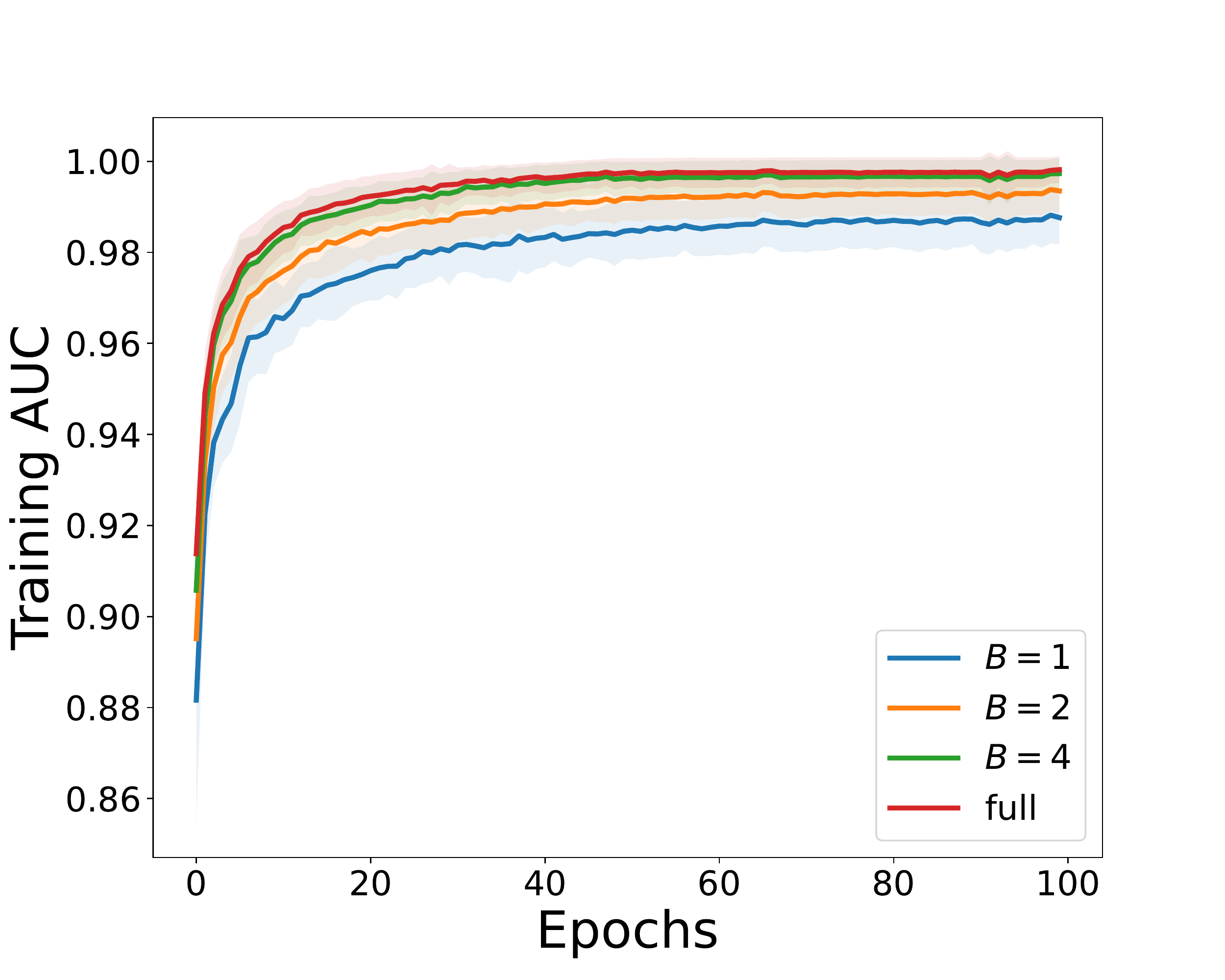}}
    \subfigure[MIDAM-smx, Elephant]{\includegraphics[scale=0.16]{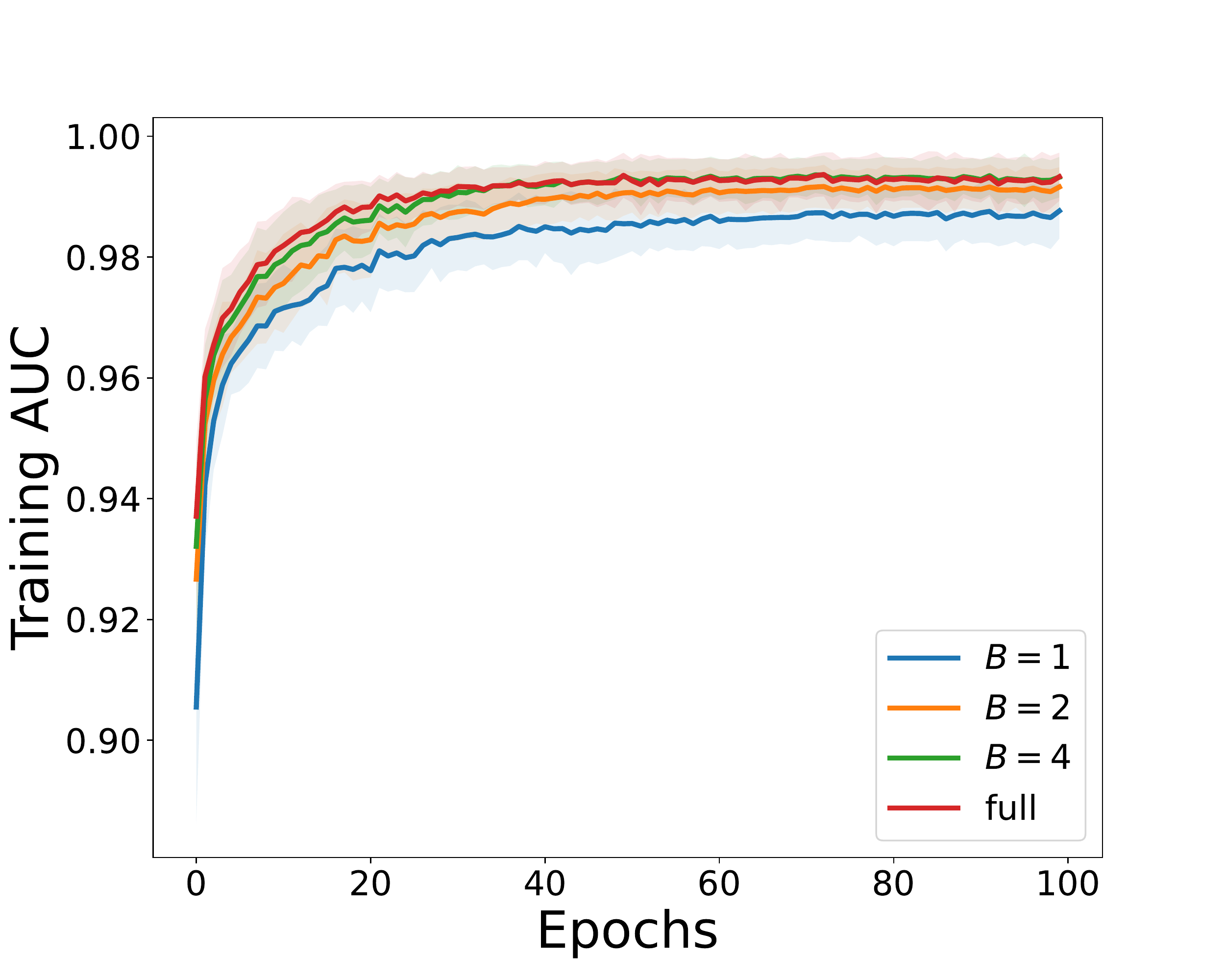}}

      \subfigure[MIDAM-att, MUSK2]{\includegraphics[scale=0.16]{Figures/MUSK2-AUCMatt-s-Tr-AUC.pdf}}
    \subfigure[MIDAM-att, Fox]{\includegraphics[scale=0.16]{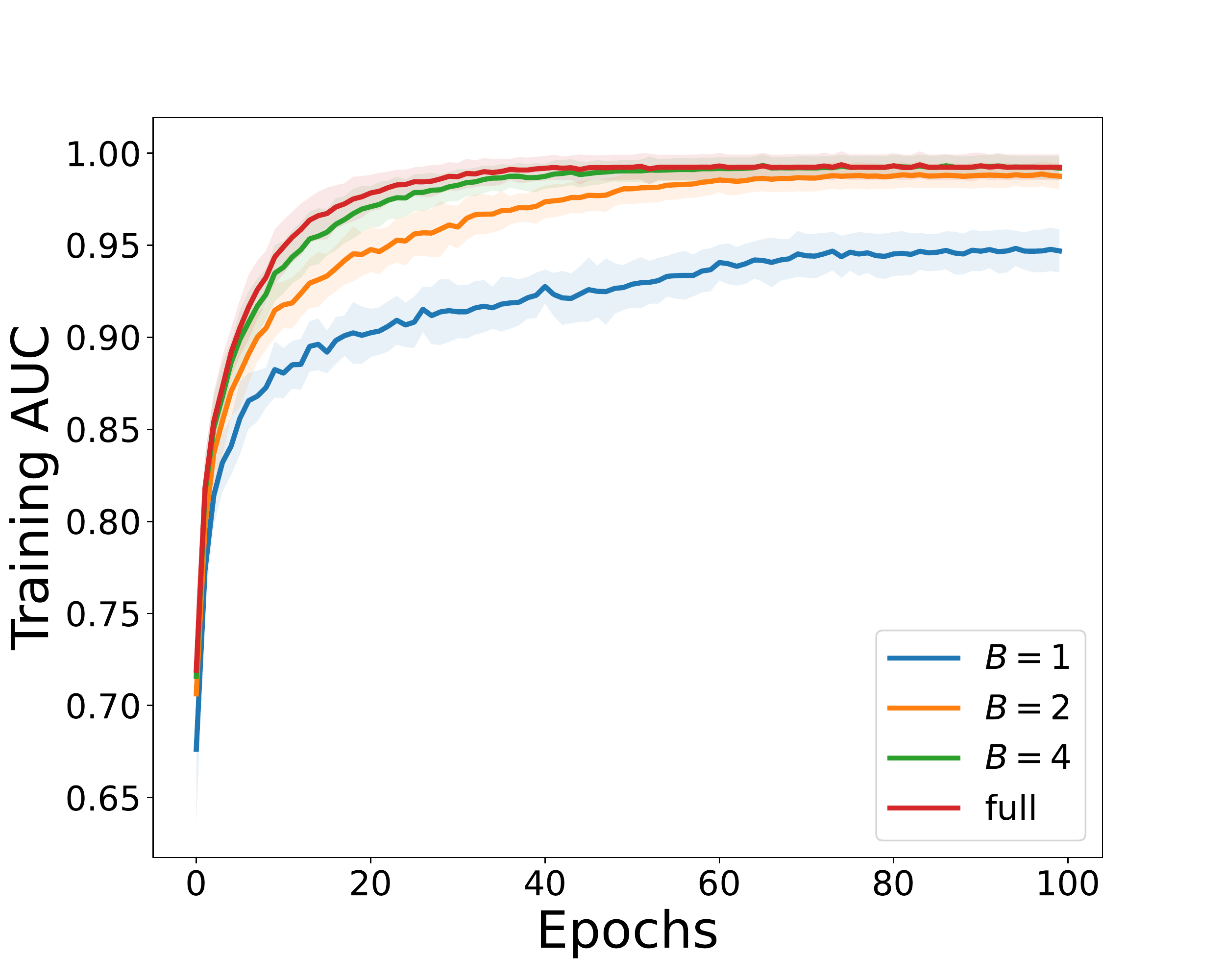}}
    \subfigure[MIDAM-att, Tiger]{\includegraphics[scale=0.16]{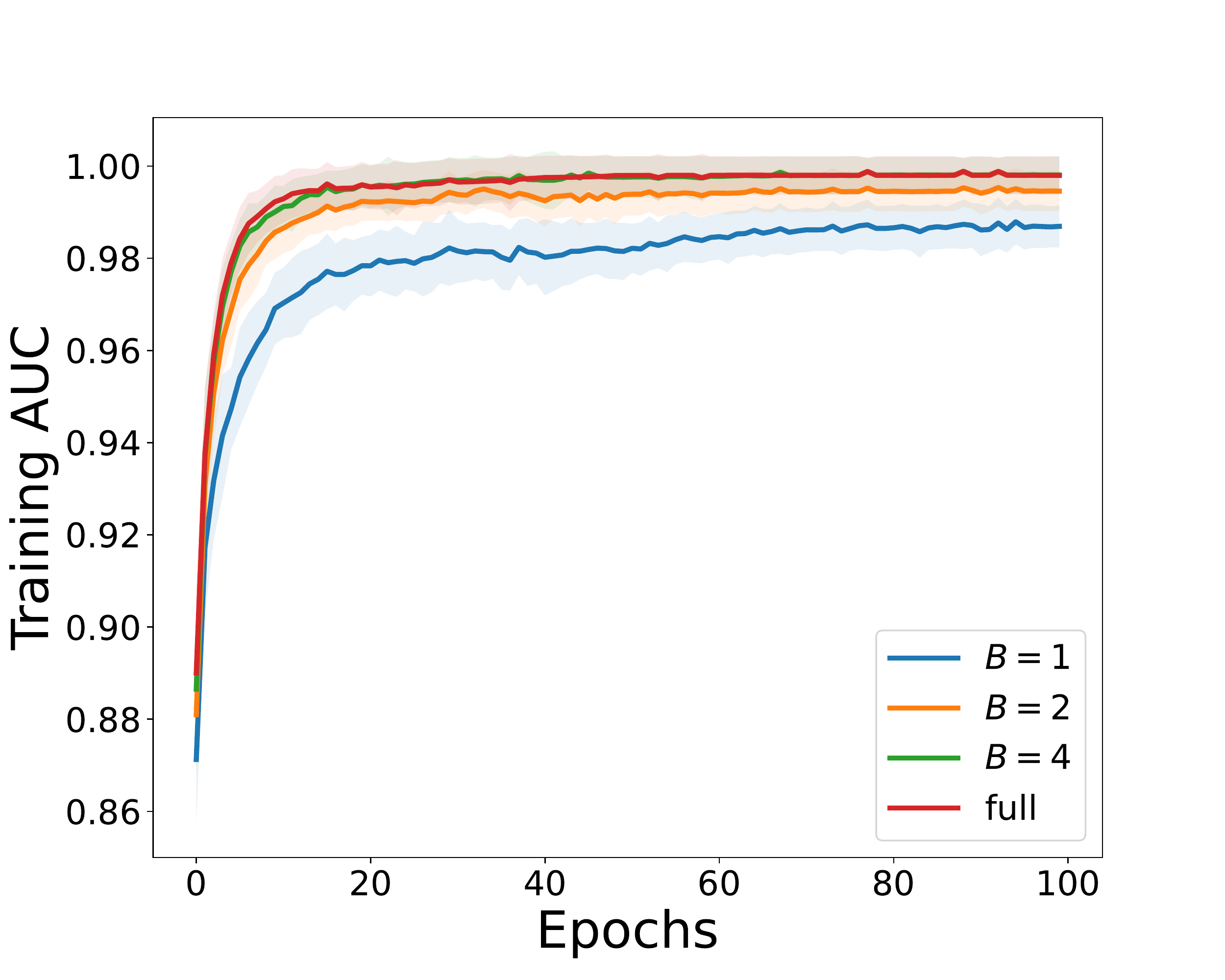}}
    \subfigure[MIDAM-att, Elephant]{\includegraphics[scale=0.16]{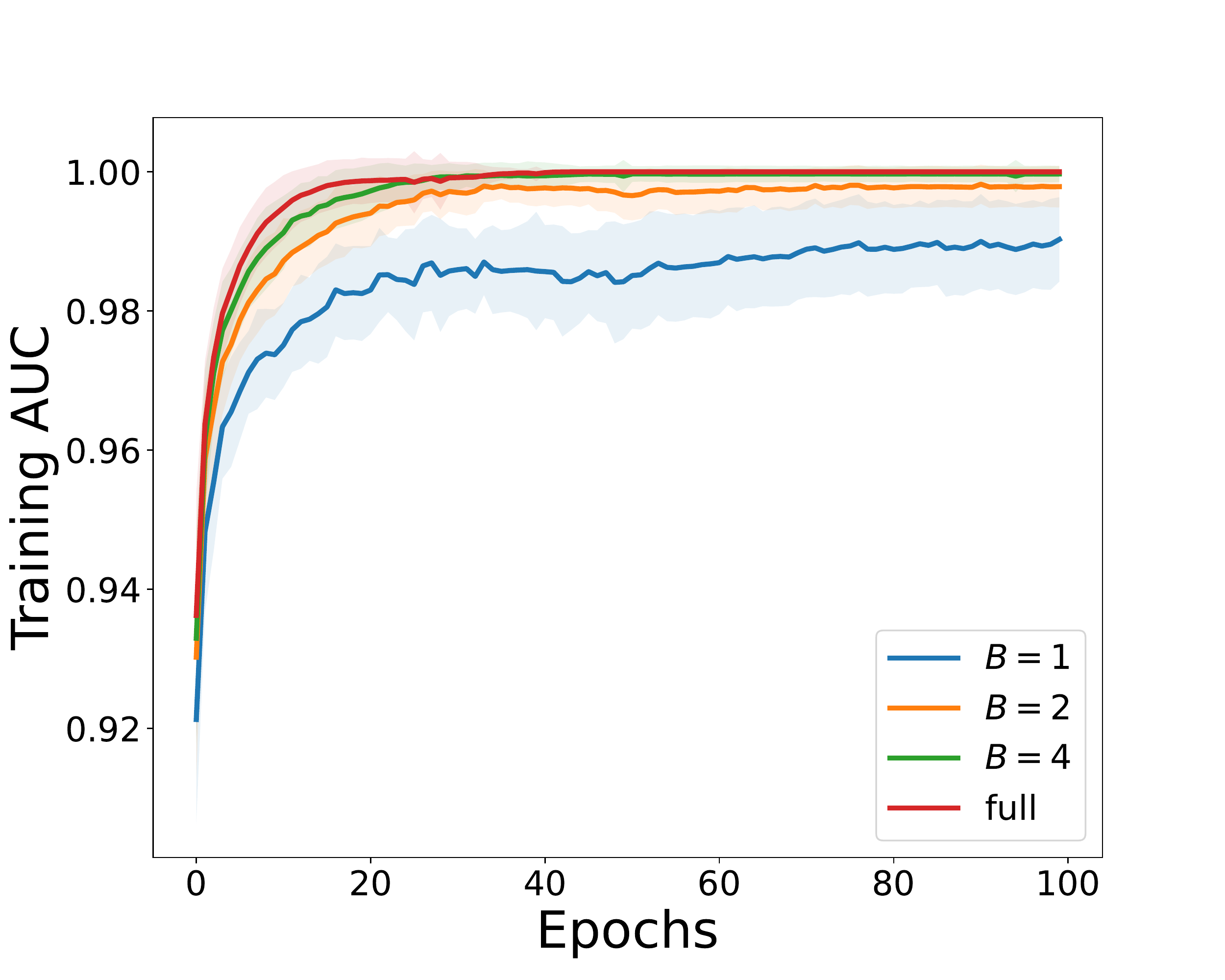}}
        \caption{Training with different instance-batch sizes} \label{fig:favor-larger-instance-batch-size}
\end{figure*}

\begin{figure*}[t]
    \centering
    \subfigure[smx, training, Breast Cancer]{\includegraphics[scale=0.16]{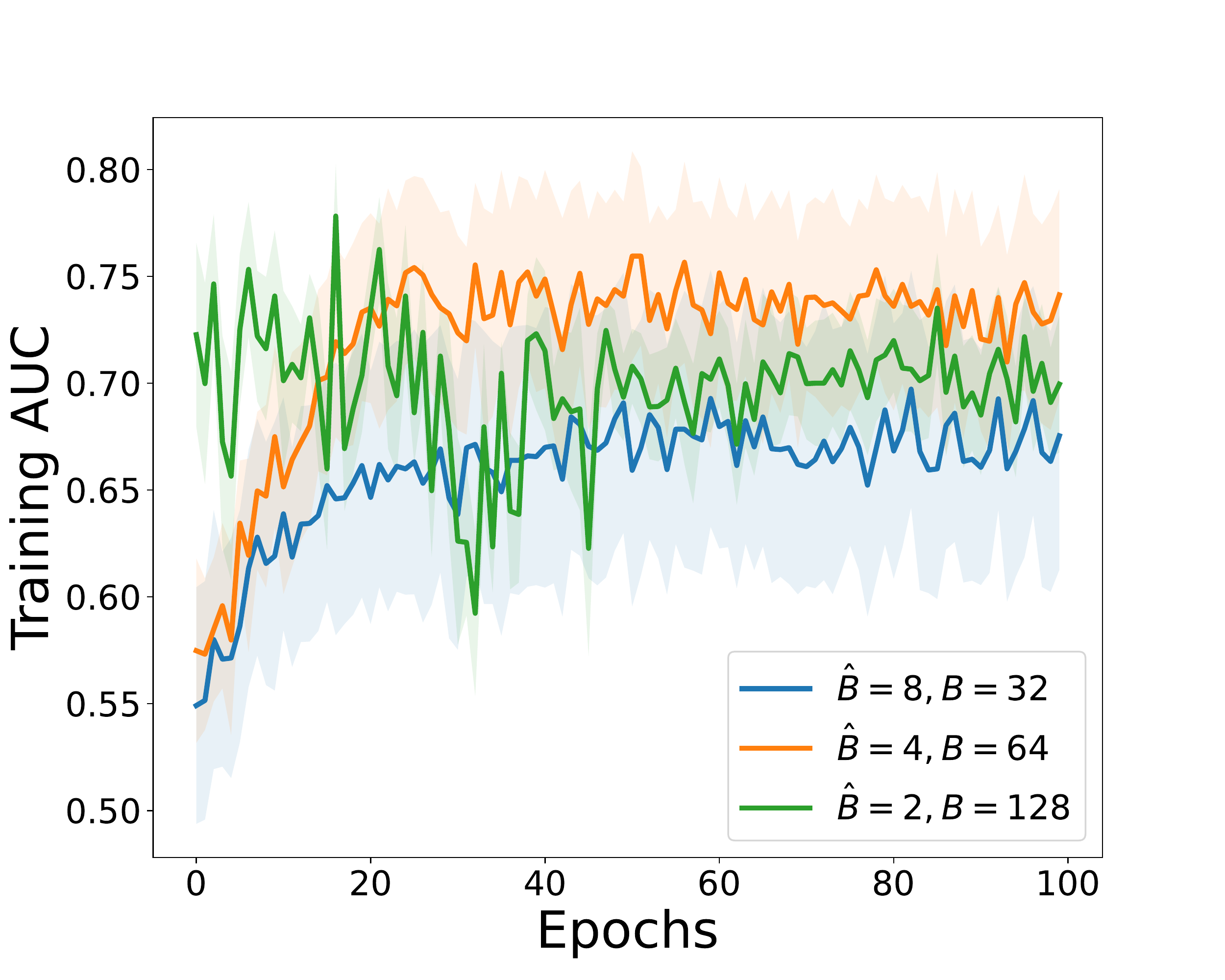}}
    \subfigure[att, training, Breast Cancer]{\includegraphics[scale=0.16]{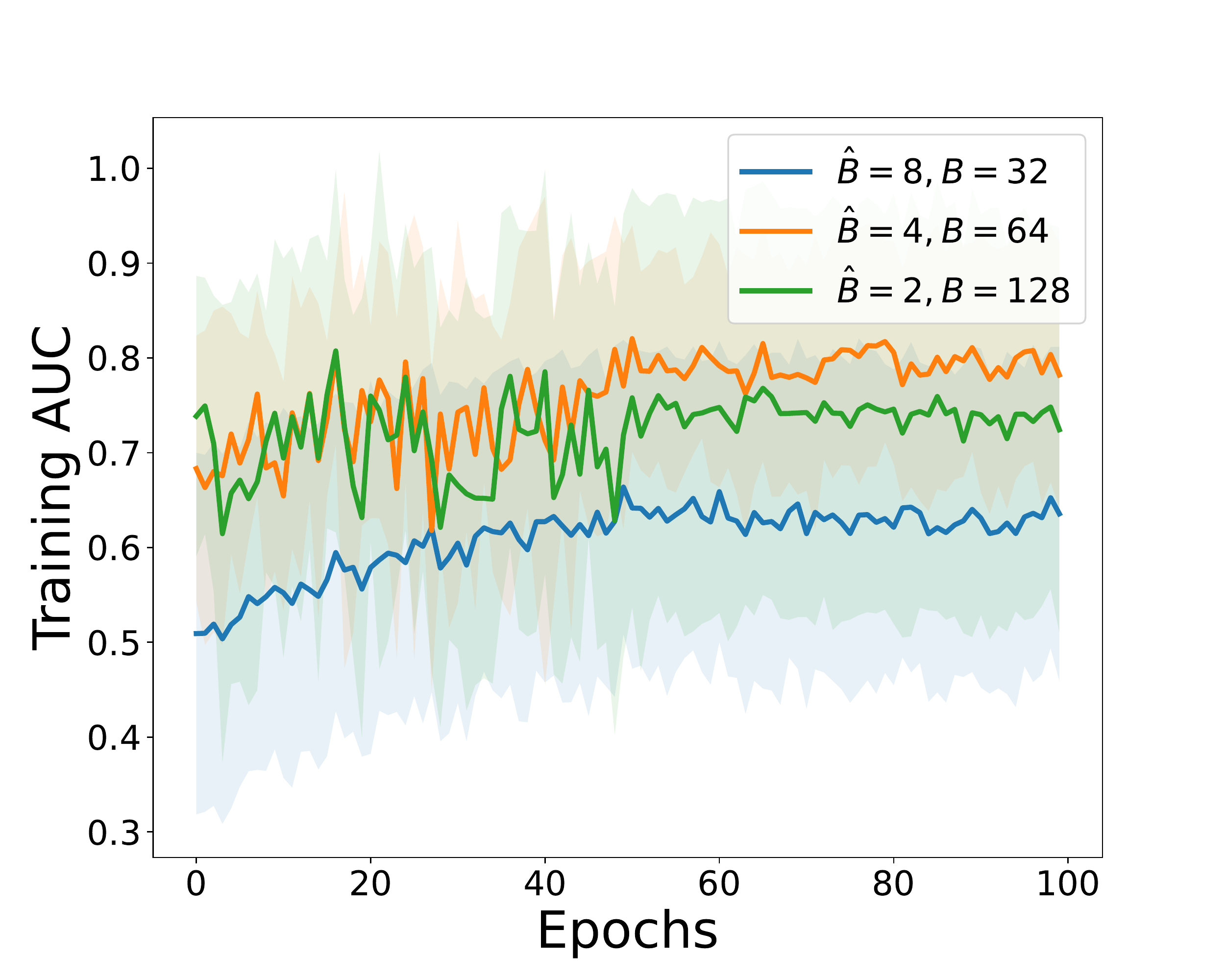}}
    \subfigure[smx, training, Colon Ade.]{\includegraphics[scale=0.16]{Figures/Colon-AUCMmax-s-Tr-AUC.pdf}}
    \subfigure[att, training, Colon Ade.]{\includegraphics[scale=0.16]{Figures/Colon-AUCMatt-s-Tr-AUC.pdf}}
 \subfigure[smx, testing, Breast Cancer]{\includegraphics[scale=0.16]{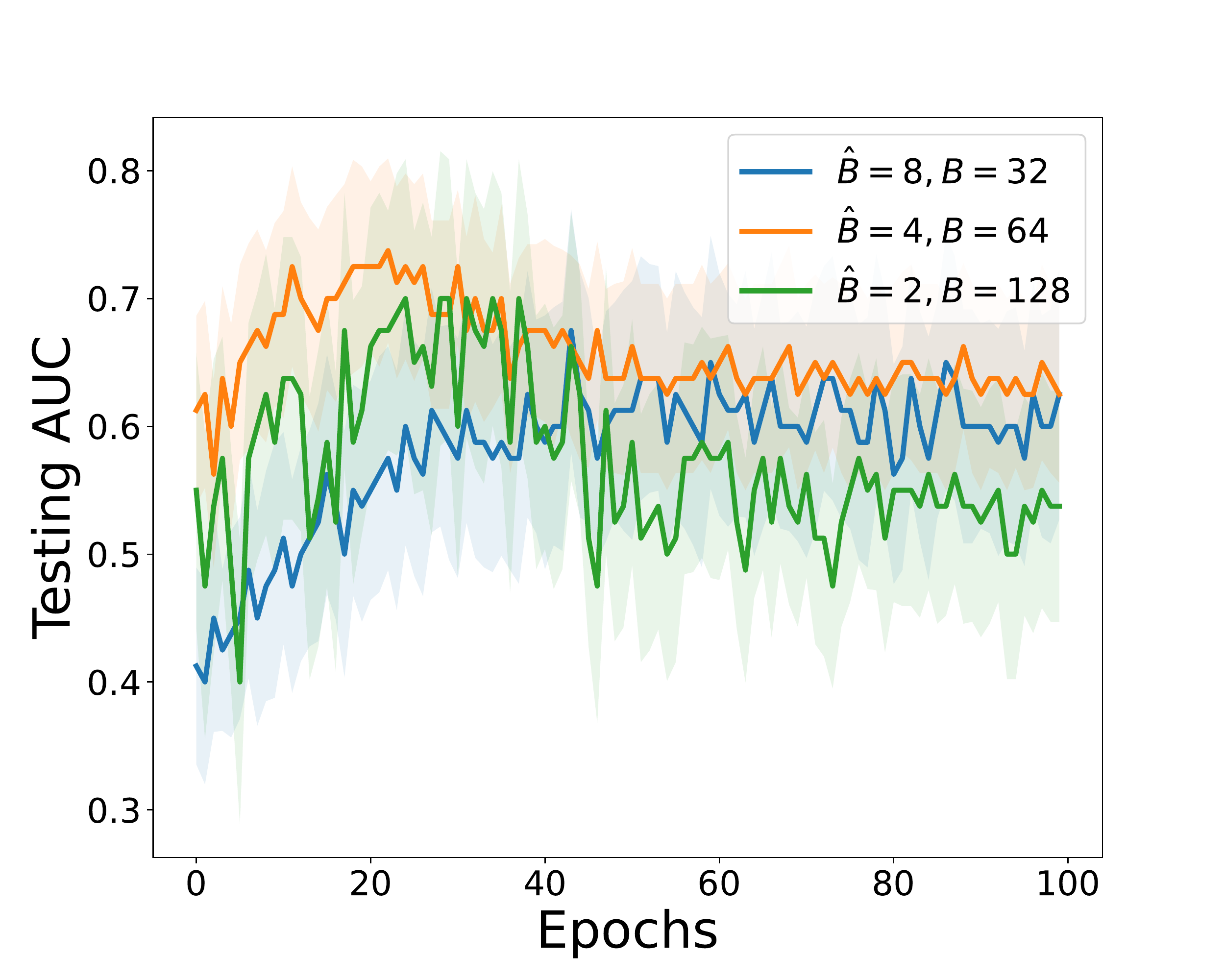}}
    \subfigure[att, testing, Breast Cancer]{\includegraphics[scale=0.16]{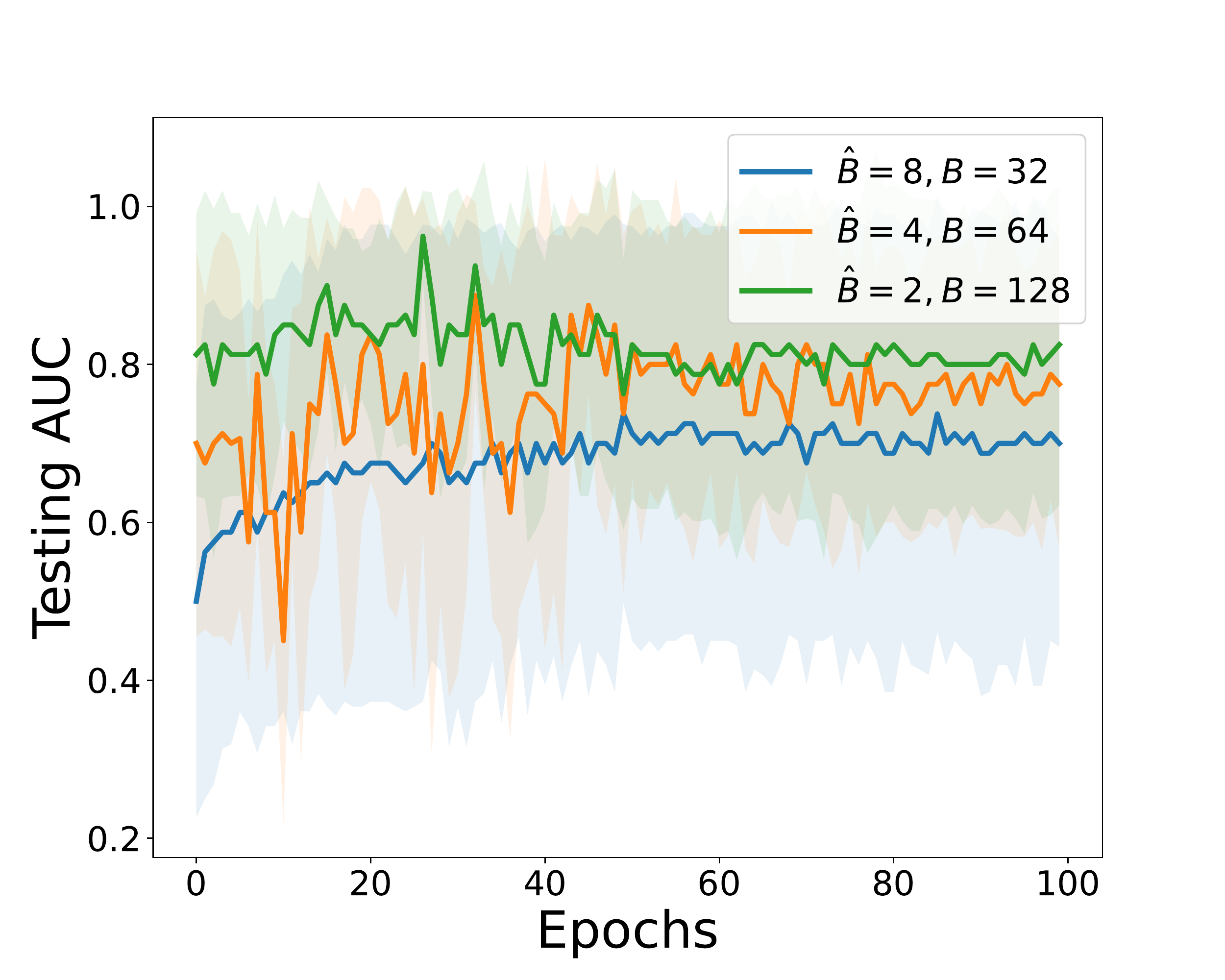}}
    \subfigure[smx, testing, Colon Ade.]{\includegraphics[scale=0.16]{Figures/Colon-AUCMmax-s-Test-AUC.pdf}}
    \subfigure[att, testing, Colon Ade.]{\includegraphics[scale=0.16]{Figures/Colon-AUCMatt-s-Test-AUC.pdf}}
        \caption{Ablation study for fixing the total budget per-iteration by varying bag-batch size $S_+=S_-=\hat B$ and instance-batch size $B$ for the proposed MIDAM approaches}\label{fig:ablation-budget}
\end{figure*}

\begin{figure*}[h]
    \centering
    \subfigure[Positive image]{\includegraphics[scale=0.35]{Figures/postive-img.pdf}}
    \subfigure[Prediction scores]{\includegraphics[scale=0.35]{Figures/postive-scores.pdf}}
    \subfigure[Attention weights]{\includegraphics[scale=0.35]{Figures/postive-wgts.pdf}}

\vskip -0.15in
    \subfigure[Negative image]{\includegraphics[scale=0.35]{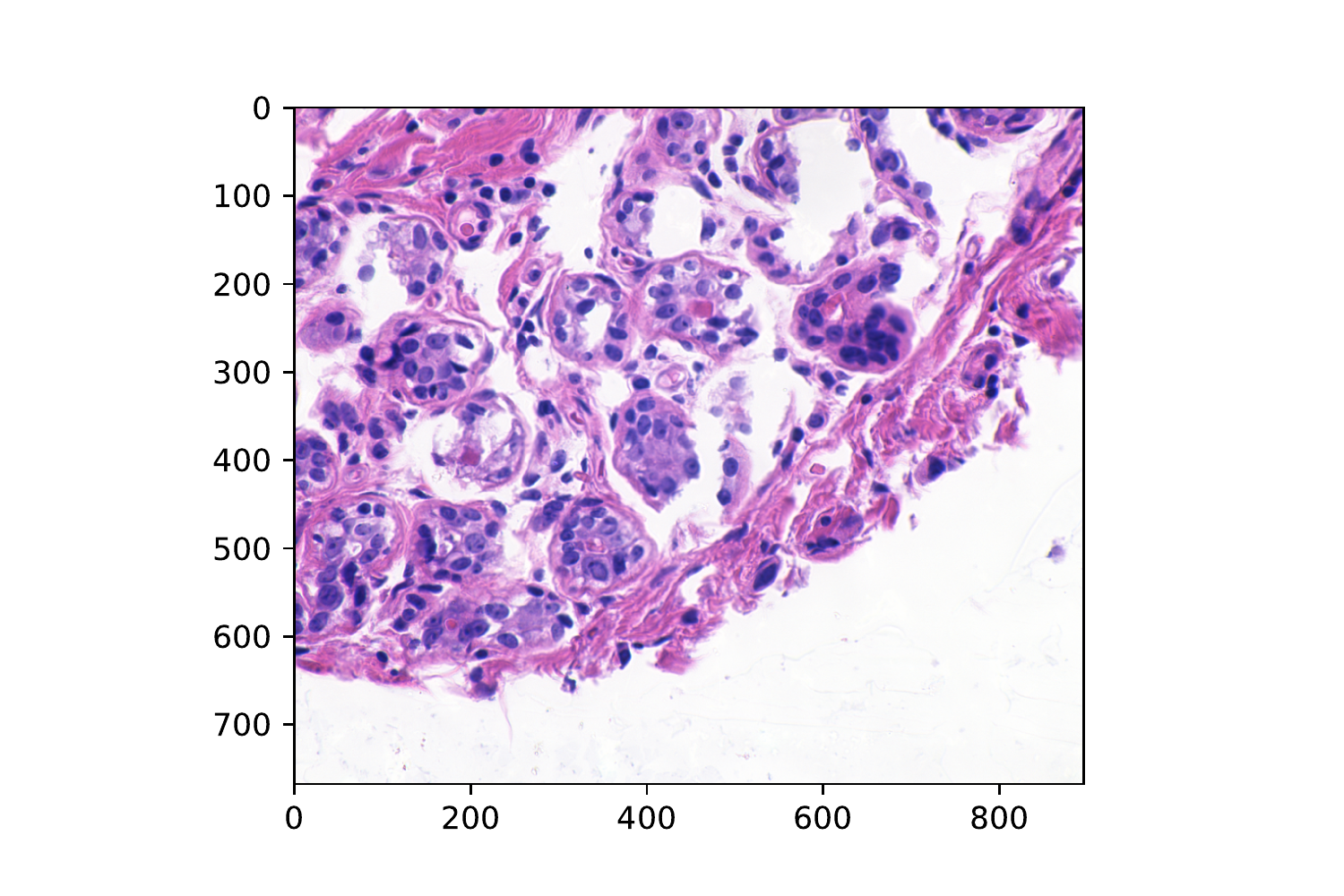}}
    \subfigure[Prediction scores]{\includegraphics[scale=0.35]{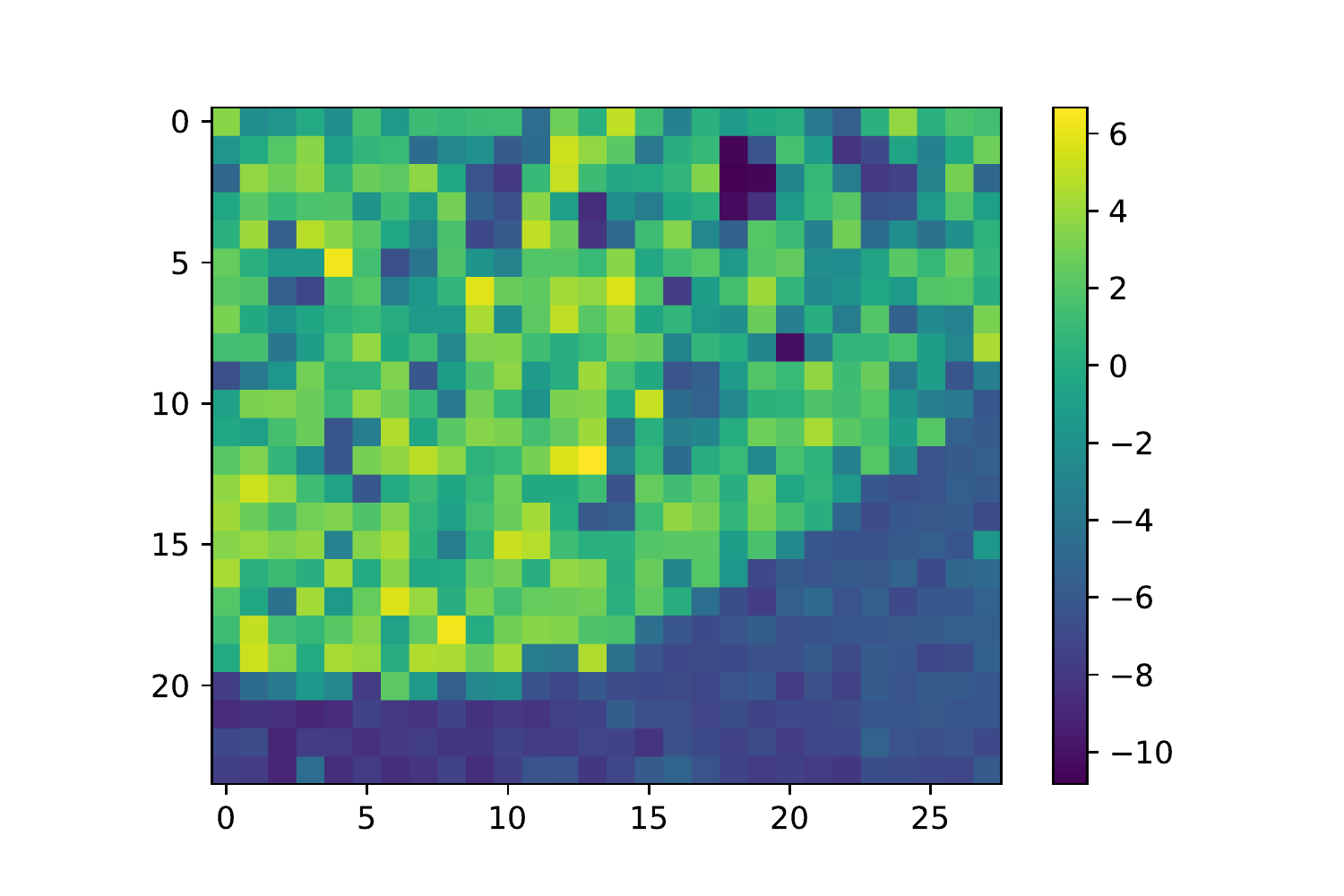}}
    \subfigure[Attention weights]{\includegraphics[scale=0.35]{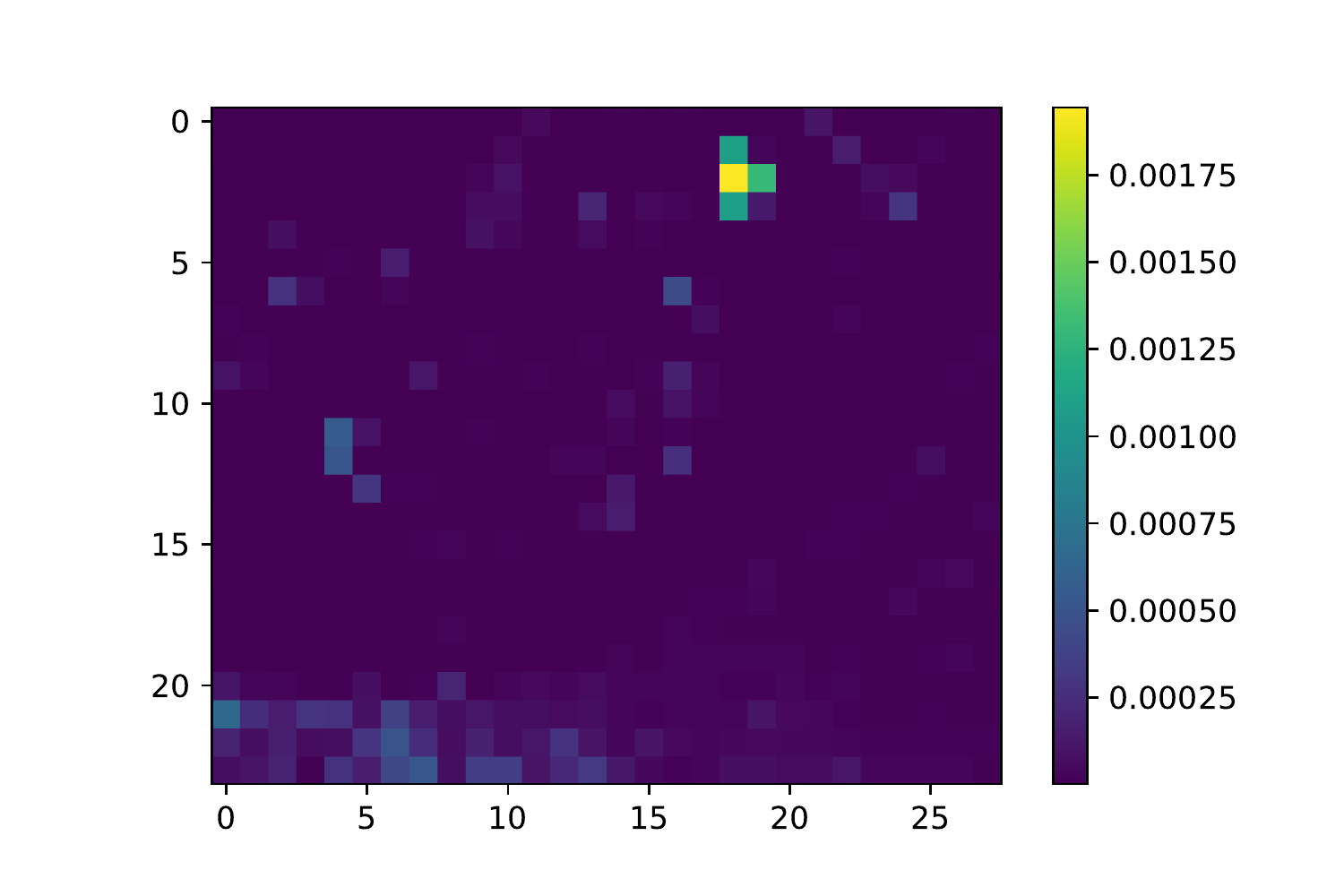}}
        \caption{Demonstrations for positive and negative examples for Breast Cancer dataset. Left: original image. Middle: prediction scores for each patch. Right: attention weights for each patch.} \label{fig:attention-elaborate}
\vspace{-0.2in}
\end{figure*}

\end{document}